\documentclass[times, 11pt,onecolumn]{article}
\usepackage[margin=1in]{geometry}
\usepackage[title]{appendix}

\usepackage[utf8]{inputenc} 
\usepackage[T1]{fontenc}    
\usepackage[textsize=tiny]{todonotes}

\usepackage[colorlinks=true,citecolor=blue,linkcolor=blue]{hyperref}
\newcommand{\circled}[1]{\small{\raisebox{.6pt}{\textcircled{\raisebox{-.8pt}{#1}}}}}

\usepackage{amssymb}
\usepackage{amsmath,mathrsfs,dsfont}
\usepackage{nicefrac}
\usepackage{algorithm}
\usepackage{algorithmicx}
\usepackage{algpseudocode}

\usepackage{booktabs}

\usepackage{color}
\usepackage{enumitem}

\PassOptionsToPackage{square,sort,comma,numbers}{natbib}
\usepackage{graphicx,tikz}
\usepackage[mathscr]{euscript}
\usepackage{amsthm}

\usepackage{bm}
\usepackage{bbm}
\usepackage{color}

\usepackage{color}
\usepackage{epstopdf}
\usepackage{subcaption}
\usepackage[capitalize,noabbrev]{cleveref}
\usepackage{thmtools}
\usepackage{thm-restate}

\usepackage{bbding}
\usepackage{mathtools}
\usepackage{wrapfig}
\usepackage{colortbl}
\usepackage{slashbox}
\allowdisplaybreaks
\usepackage{framed}
\usepackage{xcolor}

\newcommand{\cfrakR}{\mathfrak{R}} 

\newcommand{\Span}{\mathop\mathrm{Span}}
\newcommand{\tK}{\tilde K}

\newcommand{\bx}{\tilde \bx}
\newcommand{\by}{{\tilde \by}}

\newcommand{\bbx}{\overset{\rightharpoonup}{\bx}}

\newcommand{\bbw}{\overset{\rightharpoonup}{\bw}}
\newcommand{\bbe}{\overset{\rightharpoonup}{\be}}
\newcommand{\bbq}{\overset{\rightharpoonup}{\bq}}

\newcommand{\cFnn}{\cF_{\mathop\mathrm{NN}}}

\newcommand{\hell}{\hat \ell}

\newcommand{\hK}{{\hat K}}
\newcommand{\hbK}{{\hat \bK}}
\newcommand{\Kr}{{K^{(r_0)}}}

\newenvironment{proofof}[1]{\begin{proof}\textbf{of {#1}}}{\end{proof}}

\graphicspath{{illustrations/}}
\newcounter{optproblem}

\newtheoremstyle{mytheoremstyle} 
    {\topsep}                    
    {\topsep}                    
    {\normalfont}                
    {}                           
    {\bfseries}                   
    {.}                          
    {.5em}                       
    {}  

\theoremstyle{mytheoremstyle}
\newtheorem{theorem}{Theorem}[section]
\newtheorem{proposition}[theorem]{Proposition}

\newtheorem*{theorem*}{Theorem}
\newtheorem*{lemma*}{Lemma}
\newtheorem{lemma}[theorem]{Lemma}

\theoremstyle{definition}
\newtheorem{definition}[theorem]{Definition}
\newtheorem{remark}[theorem]{Remark}
\newtheorem*{remark*}{Remark}

\DeclareMathAlphabet{\pazocal}{OMS}{zplm}{m}{n}
\DeclareMathAlphabet{\mathpzc}{OMS}{pzc}{m}{it}

\setlist[itemize]{leftmargin=*}

\renewcommand{\hat}{\widehat}

\newcommand{\bfm}[1]{\ensuremath{\mathbf{#1}}}
\newcommand{\bfsym}[1]{\ensuremath{\boldsymbol{#1}}}

\def\ba{\boldsymbol a}   \def\bA{\bfm A}  
   \def\bB{\bfm B}  
     
   \def\bD{\bfm D}  
\def\be{\bfm e}   \def\bE{\bfm E}

   \def\bI{\bfm I}  
     
   \def\bK{\bfm K}  
     
   \def\bM{\bfm M}

     \def\PP{\mathbb{P}}
\def\bq{\bfm q}   \def\bQ{\bfm Q}  
   \def\bR{\bfm R}  \def\RR{\mathbb{R}}
   \def\bS{\bfm S}  
     
\def\bu{\bfm u}   \def\bU{\bfm U}  
\def\bv{\bfm v}   \def\bV{\bfm V}  
\def\bw{\bfm w}     
\def\bx{\bfm x}   \def\bX{\bfm X}  
\def\by{\bfm y}   \def\bY{\bfm Y}  
   \def\bZ{\bfm Z}  
\def\bzero{\bfm 0} 

 \def\cA{{\cal  A}}
 \def\cB{{\cal  B}}

 \def\cE{{\cal  E}}
 \def\cF{{\cal  F}}
 
 \def\cH{{\cal  H}}

 \def\cO{{\cal  O}}
 \def\cP{{\cal  P}}
 
 \def\cR{{\cal  R}}
 \def\cS{{\cal  S}}

 \def\cV{{\cal  V}}
 
 \def\cX{{\cal  X}}

\def\bbeta{\bfsym \beta}

\def\bDelta {\bfsym {\Delta}}

\def\btheta{\bfsym {\theta}}

\def\bsigma{\bfsym \sigma}
\def\bSigma{\bfsym \Sigma}

\def\btau{\bfsym {\tau}}

\def\hlambda{\hat{\lambda}}

\def\+#1{\mathcal{#1}}
\def\-#1{\textup{#1}}

\def\set#1{\left\{ #1 \right\}}
\def\pth#1{\left( #1 \right)}
\def\bth#1{\left[ #1 \right]}
\def\abth#1{\left | #1 \right |}

\def\defeq {\coloneqq}

\newcommand{\vect}[1]{{\textup{vec}\pth{#1}}}

\def \longmid {\,\middle\vert\,}
\DeclareMathSymbol{\relcolon}{\mathrel}{operators}{"3A}

\newcommand{\La}{\left\langle\kern-0.64ex\left\langle}
\newcommand{\Ra}{\right\rangle\kern-0.64ex\right\rangle}

\def\Norm#1#2{{\left\vert\kern-0.4ex\left\vert\kern-0.4ex\left\vert #1
    \right\vert\kern-0.4ex\right\vert\kern-0.4ex\right\vert}_{#2}}
\def\norm#1#2{{\left\|#1\right\|}_{#2}}

\def\ltwonorm#1{\norm{#1}{2}}
\def\fnorm#1{\norm{#1}{\textup{F}}}
\def\supnorm#1{\norm{#1}{\infty}}

\def\tr#1{\textup{tr}\left(#1\right)}

\newcommand{\1}{{\rm 1}\kern-0.25em{\rm I}}
\def\indict#1{{\rm 1}\kern-0.25em{\rm I}_{\set{#1}}}

\DeclarePairedDelimiter\floor{\lfloor}{\rfloor}

\def \eps  {\epsilon}
\def \eps {\varepsilon}

\def \diff {{\rm d}}
\def \iprod#1#2{\left\langle #1, #2 \right\rangle}

\def\set#1{\left\{#1\right\}}

\def\ball#1#2#3{\bfm{B}^{#1}\left(#2; #3\right)}

\def\unitsphere#1{\mathbb{S}^{#1}}

\def \E {\mathbb{E}}
\def\Expect#1#2{\E_{#1}\left[#2\right]}

\def \Pr {\textup{Pr}}
\newcommand{\Prob}[1]{\Pr\left[#1\right]}
\def \Var#1{\textup{Var}\left[#1\right]}

\def \lsim {\lesssim}
\def \gsim {\gtrsim}

\newcommand{\Unif}[1]{{\mathop\mathrm{Unif}}\left( #1 \right)}

\newcommand{\beq}{\begin{equation}}
\newcommand{\eeq}{\end{equation}}
\newcommand{\beqa}{\begin{eqnarray}}
\newcommand{\eeqa}{\end{eqnarray}}
\newcommand{\beqas}{\begin{eqnarray*}}
\newcommand{\eeqas}{\end{eqnarray*}}
\def\bal#1\eal{\begin{align}#1\end{align}}
\def\bals#1\eals{\begin{align*}#1\end{align*}}
\def\bsal#1\esal{\begin{small}\begin{align}#1\end{align}\end{small}}
\def\bsals#1\esals{\begin{small}\begin{align*}#1\end{align*}\end{small}}
\def\bsfal#1\esfal{\begin{small}\begin{flalign}#1\end{flalign}\end{small}}

\title{\bfseries Shallow Neural Networks Learn Low-Degree Spherical Polynomials with Feature Learning
by Learnable Channel Attention}
\author{\vspace{0.2in} Yingzhen Yang
\\School of Computing and Augmented Intelligence, Arizona State University\\
\texttt{yingzhen.yang@asu.edu}
}
\date{}
\begin{document}

\maketitle


\begin{abstract}
We study the problem of learning a low-degree spherical polynomial of degree $\ell_0 = \Theta(1) \ge 1$ defined on the unit sphere in $\RR^d$ by training
an over-parameterized two-layer neural network (NN) with channel attention in this paper.
Our main result is the significantly improved sample complexity for learning such low-degree polynomials.
We show that, for any regression risk $\eps \in (0,1)$,
a carefully designed two-layer NN with channel attention and finite width trained by the vanilla gradient descent (GD) requires the lowest sample complexity of
$n \asymp \Theta(d^{\ell_0}/\eps)$ with high probability, in contrast with the representative sample complexity $\Theta\pth{d^{\ell_0} \max\set{\eps^{-2},\log d}}$, where $n$ is the training data size. Moreover, such sample complexity is not improvable since the trained network renders a sharp rate of the nonparametric regression risk of the order $\Theta(d^{\ell_0}/{n})$
with high probability. On the other hand, the minimax optimal rate for the regression risk with a kernel of rank $\Theta(d^{\ell_0})$ is
$\Theta(d^{\ell_0}/{n})$, so that
the rate of the nonparametric regression risk of the network trained by GD is  minimax optimal.
The training of the two-layer NN with channel attention is a two-stage process. In stage one, a novel and provable learnable channel selection algorithm, as a learnable harmonic-degree selection process, is employed to select the ground truth channel number in the target function, $\ell_0$, among the initial $L \ge \ell_0$ channels in its activation function in the first layer with high probability. Such learnable channel selection is performed by efficient one-step GD on both layers of the NN, which achieves the goal of feature learning in learning low-degree polynomials. In stage two, the second layer of the network is trained by standard GD using the activation function with selected channels. To the best of our knowledge, this is the first time a minimax optimal risk bound is obtained by training an over-parameterized but finite-width neural network with feature learning capability to learn low-degree spherical polynomials.
\end{abstract}

\newpage

\section{Introduction}
With deep learning achieving remarkable breakthroughs across a wide range of machine learning tasks~\cite{YannLecunNature05-DeepLearning}, understanding the generalization capability of neural networks has become a central topic in both statistical learning theory and theoretical deep learning. A large body of work has established that gradient descent (GD) and stochastic gradient descent (SGD) can provably minimize training loss in deep neural networks (DNNs)~\cite{du2018gradient-gd-dnns,AllenZhuLS19-convergence-dnns,DuLL0Z19-GD-dnns,AroraDHLW19-fine-grained-two-layer,ZouG19,SuY19-convergence-spectral}. Beyond optimization, many studies investigate the generalization behavior of DNNs trained via gradient-based methods, deriving algorithmic generalization bounds. A key insight from this line of work is that with sufficient over-parameterization, meaning a large number of neurons, training dynamics can be effectively described using a kernel method, particularly the Neural Tangent Kernel (NTK)~\cite{JacotHG18-NTK} determined by the network’s architecture. Other results, such as~\cite{YangH21-feature-learning-infinite-network-width}, demonstrate that infinite-width neural networks can still perform feature learning. The NTK framework reveals that for highly over-parameterized models, the network weights stay close to initialization, enabling a linearized approximation via first-order Taylor expansion that facilitates generalization analysis~\cite{CaoG19a-sgd-wide-dnns,AroraDHLW19-fine-grained-two-layer,Ghorbani2021-linearized-two-layer-nn}.

The generalization properties of neural networks can also be studied through the lens of learning low-degree polynomials. This direction is motivated by analyses of spectral bias in neural networks~\cite{rahaman19a-spectral-bias,CaoFWZG21-spectral-bias,ChorariaD0MC22-spectral-bias-pnns}, which show that neural networks tend to prioritize learning target functions lying within subspaces spanned by eigenfunctions associated with NTK eigenvalues. For example, on uniformly distributed data over the unit sphere $\unitsphere{d-1}$ in $\RR^d$, degree-$\ell$ polynomials can be expressed via spherical harmonics up to degree $\ell$, as formalized in Section~\ref{sec:harmonic-analysis-detail} and Theorem~\ref{theorem:spherical-polynomial-representation-spherical-harmonics}. While~\cite{YangH21-feature-learning-infinite-network-width} shows infinite-width networks can perform feature learning, several works attempt to overcome the linear NTK regime to learn low-degree polynomials on spheres in $\RR^d$. The QuadNTK method introduced in~\cite{BaiL20-quadratic-NTK} applies a second-order Taylor expansion to improve over NTK’s linearization, achieving more effective learning of sparse ``one-directional'' polynomials. Extending this idea,~\cite{Nichani0L22-escape-ntk} shows that combining NTK and QuadNTK can capture dense polynomials with an additional sparse high-degree term. Further contributions include~\cite{DamianLS22-nn-representation-learning}, which uses two-stage optimization for learning low-degree polynomials, and~\cite{TakakuraS24-mean-field-two-layer}, which explores feature learning in the mean-field regime.

Despite these advances, existing work on training over-parameterized neural networks to learn low-degree polynomials, such as~\cite{Ghorbani2021-linearized-two-layer-nn,BaiL20-quadratic-NTK,Nichani0L22-escape-ntk,DamianLS22-nn-representation-learning,TakakuraS24-mean-field-two-layer}, often lacks sharp characterizations of regression risk. For instance,~\cite{Nichani0L22-escape-ntk} establishes that the regression risk $\eps$ holds when sample size $n \gsim d^{\ell_0} \max\set{\eps^{-2},\log d}$. Separately,~\cite{Ghorbani2021-linearized-two-layer-nn} shows that for $\tilde \Theta(d^{\ell_0}) \le n \le \Theta(d^{\ell_0+1-\delta})$ with $\tilde \Theta(d^{\ell_0})/d^{\ell_0} \to \infty$ as $d \to \infty$, NTK-based regression risk converges to zero under restrictive conditions, but no convergence rate or sharpness is established. Moreover, in practical settings where $d$ is finite, which is commonly considered in sharp rate analyses for nonparametric regression~\cite{HuWLC21-regularization-minimax-uniform-spherical,SuhKH22-overparameterized-gd-minimax,Li2024-edr-general-domain,
yang2024gradientdescentfindsoverparameterized,Yang2025-generalization-two-layer-regression}, the results from~\cite{Ghorbani2021-linearized-two-layer-nn} fail to guarantee even the vanishing regression risk.

Understanding the sharpness of regression risk in learning low-degree polynomials remains a significant open problem in statistical learning theory and theoretical deep learning. Furthermore, it is an open problem how to explore the feature learning effect of neural networks in learning such polynomials with sharp rates.  In this paper, we consider a target function $f^*$ that belongs to the Reproducing Kernel Hilbert Space (RKHS) associated with a positive definite (PD) kernel induced by an over-parameterized two-layer NN, where $f^*$ is a degree-$\ell_0$ polynomial defined on the unit sphere $\unitsphere{d-1}$ in $\RR^d$ with $\ell_0  = \Theta(1) \ge 1$. Our main result, Theorem~\ref{theorem:LRC-population-NN-fixed-point}, shows that training such a neural network using the vanilla GD achieves the minimax optimal nonparametric regression risk of
the order $ \Theta(d^{\ell_0}/n)$ with high probability. Comparatively, the minimax optimal rate for kernel regression risk with a positive definite kernel of rank $r_0 = \Theta(d^{\ell_0})$ is known to be $\Theta(r_0/n) = \Theta(d^{\ell_0}/n)$, as established in~\cite[Theorem 2(a)]{RaskuttiWY12-minimax-sparse-additive}, indicating that our result is in fact minimax optimal. Our training algorithm includes two stages. In the first stage, a novel and provable learnable channel selection algorithm is employed to select the channels in the activation function in the first layer of the network by one-step GD, where each channel covers a particular degree of spherical harmonics. It is proved that the number of selected channels is the ground truth channel number, $\ell_0$, in the target function. In the second stage, the second-layer weights are trained by GD with the fixed activation function with selected channels in the first layer. Our analysis demonstrates the potential of a new combination of feature learning and NTK-based analysis, where the feature learning effect of the network is implemented by learnable channel attention, which is followed by training the over-parameterized network by GD in the NTK regime.
The discussion of existing empirical and theoretical works about channel attention is deferred to Section~\ref{sec:related-works-channel-attention} of the appendix. To the best of our knowledge, our work is among the first to reveal the theoretical benefit of channel attention with a novel and provable learnable channel selection algorithm for learning
low-degree spherical polynomials with a minimax optimal rate.

\noindent \textbf{Feature Learning Capability of Our Method.} The feature learning capability of our training algorithm is in the training stage one, where the novel Algorithm~\ref{alg:learnable-channel-attention} is used to decide the channel number of the activation function of the NN, which is guaranteed to be  $\ell_0$ with high probability (w.h.p) by Theorem~\ref{theorem:channel-selection}. In this way, w.h.p. stage two performs kernel regression with the oracle kernel, achieving the minimax optimal rate. The estimation of $\ell_0$ is an important goal of feature learning for learning the target function of degree-$\ell_0$ with sharp regression risk. To see this, as a well-known fact widely discussed~\cite{DamianLS22-nn-representation-learning,Ghorbani2021-linearized-two-layer-nn},  the target function lies in a subspace with all spherical harmonics of degree $\le \ell_0$ as its orthogonal basis. Therefore, only the optimal sample complexity of $\Theta(d^{\ell_0}/\eps)$ is required to learn such a target function with any risk $\eps > 0$ by our training stage two w.h.p. Furthermore, an inaccurate and conservative estimate $\ell' > \ell_0$ leads to worse sample complexity $\Theta(d^{\ell'}/\eps)$ compared to our optimal sample complexity. The literature studying the feature learning effect, such as~\cite{LeeOSW24-low-dim-polynomials,DamianLS22-nn-representation-learning}, learns the features of the subspace that the target function lies in so as to achieve sharp regression risk. Our estimate of $\ell_0$ achieves the feature learning effect under the similar principle. During the two-stage training, the kernel evolves as the activation function changes from stage one to two. Thanks to the feature learning capability of our method, our result is stronger than the literature in terms of learning general low-degree spherical polynomials. For example, existing works~\cite{WeiLLM19-regularization,Glasgow24-SGD-features-xor,LeeOSW24-low-dim-polynomials,AbbeAM22-SGD-merged-staircase} do not consider the regression problem where the target function is a degree-$\ell_0$ spherical polynomial with our sharp and minimax optimal regression risk rate of $\Theta(d^{\ell_0}/n)$. In particular, \cite{WeiLLM19-regularization} does not consider the regression problem with the target function being a polynomial. The results of~\cite{Glasgow24-SGD-features-xor} are limited to a very specific case where the target function is a quadratic XOR function. In~\cite{LeeOSW24-low-dim-polynomials}, the target function is a single-index function $f^*(\bx) = \sigma^* (\langle \bx, \btheta \rangle)$ where the function $\sigma^*$ has information exponent $p$, so that $f^*$ is limited to be a polynomial of a particular direction, parameterized by $\btheta$, of the variable $\bx$, instead of being a more general non-single-index spherical polynomial considered in this paper. \cite{AbbeAM22-SGD-merged-staircase} studies the case that the target function $ f^* $ is a low-dimensional latent function of dimension $P$ in the ambient space of dimension $d$ with $P \le d$, and shows necessary and nearly sufficient condition that $ f^* $ is strongly SGD-learnable in the mean-field regime.

We organize this paper as follows. We first introduce in Section~\ref{sec:setup} the problem setup. The training algorithm of the network is described in Section~\ref{sec:training}. Our main result is summarized in Section~\ref{sec:main-results} with the novel training algorithm by GD and the sharp risk bound for learning low-degree spherical polynomials. The roadmap of proofs, the summary of the approaches and the key technical results in the proofs, and the novel proof strategies of this work are presented in Section~\ref{sec:proof-roadmap}.

\smallskip
\vspace{-.0in}\noindent \textbf{Notations.} We use bold letters for matrices and vectors, and regular lower letters for scalars throughout this paper. $\bA^{(i)}$ is the $i$-th column of a matrix $\bA$. A bold letter with subscripts indicates the corresponding rows or elements of a matrix or a vector. We put an arrow on top of
a letter with subscript if it denotes a vector, e.g.,
$\bbx_i$ denotes the $i$-th training
feature. $\norm{\cdot}{F}$ and
$\norm{\cdot}{p}$ denote the Frobenius norm and the vector $\ell^{p}$-norm or the matrix $p$-norm. $[m\relcolon n]$ denotes all the integers between $m$ and $n$ inclusively, and $[1\relcolon n]$ is also written as $[n]$. $\Var{\cdot}$ denotes the variance of a random variable. $\bI_n$ is a $n \times n$ identity matrix.  $\indict{E}$ is an indicator function which takes the value of $1$ if event $E$ happens, or $0$ otherwise. The complement of a set $A$ is denoted by $A^c$, and $\abth{A}$ is the cardinality of the set $A$. $\vect{\cdot}$ denotes the vectorization of a matrix or a set of vectors, and $\tr{\cdot}$ is the trace of a matrix.
We denote the unit sphere in $d$-dimensional Euclidean space by $\unitsphere{d-1} \defeq \{\bx \colon  \bx \in \RR^d, \ltwonorm{\bx} =1\}$. Let $\cX$ denote the input space, and
$L^p(\cX, \mu) $ with $p \ge 1$ denote the space of $p$-th power integrable functions on $\cX$ with probability measure $\mu$, and the inner product $\iprod{\cdot}{\cdot}_{L^p(\mu)}$ and $\norm{\cdot}{{L^p(\mu)}}^2$ are defined as $\iprod{f}{g}_{L^p(\mu)} \coloneqq \int_{\cX}f(x)g(x) \diff \mu(x)$ and $\norm{f}{L^p(\mu)}^p \coloneqq \int_{\cX}\abth{f}^p(x) \diff \mu (x) <\infty$. $\ball{}{\bx}{r}$ is the Euclidean closed ball centered at $\bx$ with radius $r$. Given a function $g \colon \cX \to \RR$, its $L^{\infty}$-norm is denoted by $\norm{g}{\infty} \defeq \sup_{\bx \in \cX} \abth{g(\bx)}$, and
$L^{\infty}$ is the function class whose elements have bounded $L^{\infty}$-norm. $\iprod{\cdot}{\cdot}_{\cH}$ and $\norm{\cdot}{\cH}$ denote the inner product and the norm in the Hilbert space $\cH$. $a = \cO(b)$ or $a \lsim b$ indicates that there exists a constant $c>0$ such that $a \le cb$. $\tilde \cO$ indicates there are specific requirements in the constants of the $\cO$ notation. $a = o(b)$ and $a = w(b)$ indicate that $\lim \abth{a/b}  = 0$ and $\lim \abth{a/b}  = \infty$, respectively. $a \asymp b$  or $a = \Theta(b)$ denotes that
there exists constants $c_1,c_2>0$ such that $c_1b \le a \le c_2b$. $\Unif{\unitsphere{d-1}}$ denotes the uniform distribution on $\unitsphere{d-1}$.
The constants defined throughout this paper may change from line to line.
We use $\Expect{P}{\cdot}$ to denote the expectation with respect to the distribution $P$.
$\PP_{\cS}$ denotes the orthogonal projection onto the space $\cS$, and $\Span(\bA)$ denotes the linear space spanned by the columns of the matrix $\bA$. $\overline{A}$ denotes the closure of a set $A$. Throughout this paper we let the input space be $\cX = \unitsphere{d-1}$.

\section{Problem Setup}
\label{sec:setup}
We introduce the problem setups for nonparametric regression with the target function as a low-degree spherical polynomial in this section.

\subsection{Two-Layer Neural Network with Channel Attention}
\label{sec:setup-two-layer-nn}

We are given the training data $\set{(\bbx_i,  y_i)}_{i=1}^n$ where each data point is a tuple of feature vector $\bbx_i \in \cX$ and its response $ y_i \in \RR$. Throughout this paper we assume
that no two training features coincide, that is, $\bbx_i \neq \bbx_j$ for all $i,j \in [n]$ and $i \neq j$.
We denote the training feature vectors by $\bS = \set{\bbx_i}_{i=1}^n$, and denote by $P_n$ the empirical distribution over $\bS$.
The response $y_i$ is given by $ y_i= f^*(\bbx_i) + w_i$ for
$i \in [n]$,
where $\set{w_i}_{i=1}^n$ are i.i.d. sub-Gaussian random variables as the noise with mean $0$ and variance proxy $\sigma_0^2$, that is, $\Expect{}{\exp(\lambda w_i)} \le \exp(\lambda^2 \sigma_0^2/2)$ for any $\lambda \in \RR$.
$f^*$ is the target function to be detailed later. We define $\by \defeq \bth{y_1,\ldots,y_n}^{\top}$, $\bw \defeq \bth{w_1,\ldots,w_n}^{\top}$, and use $f^*(\bS) \defeq \bth{f^*(\bbx_1),\ldots,f^*(\bbx_n)}^{\top}$ to denote the clean target labels.
The feature vectors in $\bS$ are drawn i.i.d. according to the data distribution $P = \Unif{\unitsphere{d-1}}$ with $\mu$ being the probability measure for $P$.
We consider a two-layer linear neural network (NN) with channel attention in this paper whose
mapping function is
\bal\label{eq:two-layer-nn}
&f(\btau, \ba,\bx) =  \frac{1}{\sqrt{m}} \sum_{r=1}^{m}
a_r \sigma_{\btau}(\bx,\bbq_r),
\eal%
where $\bx \in \cX$ is the input, $\bQ = \set{\bbq_r}_{r=1}^m$ are the random weights  drawn i.i.d. according to $P = \Unif{\cX}$.
$\sigma_{\btau}$ is the activation function which is a PD kernel defined as
\bal\label{eq:sigma-activation-def}
\sigma_{\btau}(\bx,\bx') \defeq \sum\limits_{\ell=0}^{L} \sum\limits_{j=1}^{N(d,\ell)}
\tau_{\ell} \mu_{\sigma,\ell} Y_{\ell,j}(\bx)Y_{\ell,j}(\bx'), \,
\mu_{\sigma,\ell} = N^{-1}(d,\ell)
\textup{ for } 0 \le \ell \le L, \,\,\, \forall \bx,\bx' \in \cX.
\eal
Here $\set{Y_{\ell j}}_{j \in [N(d,\ell)]}$ are the spherical harmonics of degree $\ell$ which form an orthogonal basis of $\cH_{\ell}$ of dimension $N(d,\ell)$, and $\cH_{\ell}$ denotes the space of degree-$\ell$ homogeneous harmonic polynomials on $\cX$. The background about harmonic analysis on $\unitsphere{d-1}$ is deferred to Section~\ref{sec:harmonic-analysis-detail} of the appendix.
Each $\mu_{\sigma,\ell} Y_{\ell,j}(\bx)Y_{\ell,j}(\bx')$ with $\ell \in [0\relcolon L]$ constitutes a channel in the output of the activation function, and $\btau =  \set{\tau_{\ell}}_{\ell=0}^L$ are the channel attention weights with $L+1$ channels.
It is noted that in the two-layer NN (\ref{eq:two-layer-nn}), the first layer comprises the spherical harmonics as the activation functions with channel attention weights, and $\ba = \bth{a_1, \ldots, a_m} \in \RR^m$ denotes the weights of the second layer.
It  follows from the background in harmonic analysis on spheres in Section~\ref{sec:harmonic-analysis-detail}
that for every given $\bx,\bx' \in \cX$, $\sigma_{\btau}(\bx,\bx')$ can be efficiently computed with $\Theta(L)$ time complexity through
$\sigma_{\btau}(\bx,\bx') =  \sum\limits_{\ell=0}^{L} \tau_{\ell} P^{(d)}_{\ell}(\iprod{\bx}{\bx'})$, where each channel, $P^{(d)}_{\ell}$, is the $\ell$-th  Gegenbauer polynomial which can be computed efficiently in $\Theta(1)$ time for each $\ell \in [0\relcolon L]$ by dynamic programming, as shown in Lemma~\ref{lemma:efficient-computation-activation} in
Section~\ref{sec:harmonic-analysis-detail} of the appendix.
We let $L \ge \ell_0$. Intuitively, each $P^{(d)}_{\ell}$ covers the information about the spherical harmonics of degree $\ell$, so that all the information in the target function is captured with $L \ge \ell_0$. With a constant $\ell_0 \in \Theta(1)$, it is always feasible to set $L \ge \ell_0$ with suitably large $L$, and the computation of $\sigma_{\btau}(\bx,\bx')$ takes $\Theta(L)=\Theta(1)$ time when $L = \Theta(1)$.

We will first run a learnable channel selection algorithm described in Algorithm~\ref{alg:learnable-channel-attention}, which is essentially a learnable harmonic-degree selection algorithm to be detailed in Section~\ref{sec:training},
to keep only the first $\hell$ channels with the updated attention weights $\set{\tau_{\ell}= \mu^{-\frac 12}_{\sigma,\ell}}_{\ell=0}^{\hell}$, and $\hell \le L$. The activation function after applying such learnable channel selection becomes
\bal\label{eq:sigma-channel-selected}
\sigma_{\btau}(\bx,\bx') =  \sum\limits_{\ell=0}^{\hell} \tau_{\ell} P^{(d)}_{\ell}(\iprod{\bx}{\bx'})
=\sum\limits_{\ell=0}^{\hell}\sum\limits_{j=1}^{N(d,\ell)}
 \mu^{\frac 12}_{\sigma,\ell}  Y_{\ell,j}(\bx)Y_{\ell,j}(\bx').
\eal
The feature learning effect of the two-layer NN with channel attention (\ref{eq:two-layer-nn}) is that, the number of selected channels, $\hell$, is the ground truth channel number, $\ell_0$, in the target function w.h.p., to be detailed in Section~\ref{sec:training}. With the updated activation function
(\ref{eq:sigma-channel-selected}) after learnable channel selection, we will train the second-layer weights $\ba$ by GD with fixed activation function $\sigma_{\btau}$ in the first layer.
Herein we define the following empirical kernel incurred during the training of the two-layer NN (\ref{eq:two-layer-nn}) with selected channels by GD,
\bal\label{eq:Km-def}
\hK(\bx,\bx') = \frac 1m \sum\limits_{r=1}^m \sigma_{\btau}(\bx,\bbq_r)\sigma_{\btau}(\bbq_r,\bx'),
\eal
and its population version
\bal\label{eq:K-def}
K(\bx,\bx') = \Expect{\bw \sim \Unif{\cX}}{\sigma_{\btau}(\bx,\bw)\sigma_{\btau}(\bw,\bx')} =
\sum\limits_{\ell=0}^{\hell} \sum\limits_{j=1}^{N(d,\ell)} \mu_{\sigma,\ell} Y_{\ell,j}(\bx)Y_{\ell,j}(\bx').
\eal
$K$ is in fact the NTK of the network (\ref{eq:two-layer-nn}) with respect to its second-layer weights $\ba$.
We denote by $\hbK \in \RR^{n \times n}$ with $\hbK_{ij} = \hK(\bbx_i,\bbx_j)$
for $i,j \in [n]$ the gram matrix of $\hK$ over the training features $\bS$, and let $\hbK_n = \hbK/n$.
Similarly, the gram matrix of $K$ is $\bK \in \RR^{n \times n}$ with $\bK_{ij} = K(\bbx_i,\bbx_j)$ for $i,j \in [n]$, and $\bK_n = \bK/n$. Let the eigendecomposition  of $\bK_n$ be $\bK_n = \bU \bSigma {\bU}^{\top}$ where $\bU$ is an $n \times n$ orthogonal matrix, and $\bSigma$ is a diagonal matrix with its diagonal elements $\set{\hlambda_i}_{i=1}^n$ being the eigenvalues of $\bK_n$ and sorted in a non-increasing order.
It follows from Lemma~\ref{lemma:spectrum-K} that $\sup_{\bx,\bx' \in \cX} K(\bx,\bx') = \hell+1$, so that it can be verified that $\hlambda_1 \in (0, \hell+1]$.

\subsection{Kernel and Kernel Regression for Nonparametric Regression}
\label{sec:setup-kernels-target-function}
Let $\cH_{K}$ be the Reproducing Kernel Hilbert Space (RKHS) associated with
$K$. Because $K$ is of finite rank, the integral operator $T_K \colon L^2(\cX,\mu) \to L^2(\cX,\mu), \pth{T_K f}(\bx) \defeq \int_{\cX} K(\bx,\bx') f(\bx') \diff \mu(\bx')$ is a positive, self-adjoint, and compact operator on $L^2(\cX,\mu)$. By the spectral theorem and
Lemma~\ref{lemma:spectrum-K}, the eigenfunctions of $T_K$ are $\set{Y_{\ell j}}_{\ell \in [0:\hell], j \in [N(d,\ell)]}$, the spherical harmonics of degree up to $\hell$. $\mu_{\ell} = {\mu_{\sigma,\ell}} = N(d,\ell)^{-1}$ is the eigenvalue corresponding to the eigenspace $\cH_{\ell}$, that is, $T_K Y_{\ell,j} = \mu_{\ell} Y_{\ell,j}$ for every $\ell \in [0:\hell]$ and $j \in [N(d,\ell)]$. Let $\set{\mu_{\ell}}_{\ell \ge 0}$ be the distinct eigenvalues
associated with $T_K$, and let
$m_{\ell}$ be the be the sum of multiplicities of
the eigenvalues $\set{\mu_{\ell'}}_{\ell'=0}^{\ell}$.
That is, $m_{\ell} - m_{\ell-1}$ is the multiplicity
of $\mu_{\ell}$ with $m_{-1} = 0$. We define $r_0 \defeq m_{\ell_0} = \sum_{\ell=0}^{\ell_0} N(d,\ell)$ as the multiplicity of all the top $\ell_0+1$ distinct eigenvalues. For a positive constant $\gamma_0$, we define $\cH_{K}(\gamma_0) \defeq \set{f \in \cH_{K} \colon \norm{f}{\cH} \le \gamma_0}$ as the closed  ball in $\cH_K$ centered at $0$ with radius $\gamma_0$. We note that $\cH_{K}(\gamma_0)$ is also specified by $\cH_{K}(\gamma_0) =
\big\{f \in L^2(\cX,\mu)\colon f = \sum_{\ell=0}^{\hell} \sum_{j=1}^{N(d,\ell)} \alpha_{\ell,j} Y_{\ell,j}, \sum_{\ell=0}^{\ell_0} \sum_{j=1}^{N(d,\ell)} \alpha_{\ell,j}^2/\mu_{\ell}\le \gamma_0^2 \big\}$.
$\cH_K(\gamma_0)$ is in fact formed by the union of the space of homogeneous harmonic polynomials up to degree $\hell$ with RKHS-norm $\gamma_0$, and $\cH_K$ is a subspace of dimension $m_{\hell}$ in $L^2(\cX,\mu)$.
We define a PD kernel $\Kr(\bx,\bx') \defeq \sum\limits_{\ell=0}^{\ell_0}\sum\limits_{j=1}^{N(d,\ell)}  \mu_{\ell} Y_{\ell,j}(\bx)Y_{\ell,j}(\bx')$ for all $\bx,\bx' \in \cX$, then $\Kr$ is a low-rank kernel of rank $r_0$. It is also shown in Lemma~\ref{lemma:r0-estimate} in Section~\ref{sec:harmonic-analysis-detail} of the appendix that $r_0 = \Theta(d^{\ell_0})$. 

\vspace{0.1in}\noindent \textbf{The task of nonparametric regression.}
We consider the target function
\bal\label{eq:target-func}
f^*(\bx) = \sum\limits_{\ell=0}^{\ell_0} \sum\limits_{j=1}^{N(d,\ell)} \beta_{\ell,j} Y_{\ell,j}(\bx), \,\, \textup{s.t. } \sum_{\ell=0}^{\ell_0} \sum_{j=1}^{N(d,\ell)} \beta_{\ell,j}^2/\mu_{\ell}\le \gamma_0^2, \quad
\forall \bx \in \cX,
\eal
where $\ell_0 = \Theta(1)\ge 1$, and $f^*$ lies in the space of homogeneous harmonic polynomials up to degree $\ell_0$. It can be verified that $f^* \in \cH_{\Kr}(\gamma_0)$, and $\cH_{\Kr}(\gamma_0) \subseteq \cH_{K}(\gamma_0)$ if $\hell \ge \ell_0$. The task of the analysis for nonparametric regression is to find an estimator $\hat f$ from the training data $\set{(\bbx_i,  y_i)}_{i=1}^n$ so that
the risk $\Expect{P}{\pth{\hat f - f^*}^2}$ vanishes at a fast rate. In this work, we aim to establish a sharp rate of the risk where the over-parameterized neural network (\ref{eq:two-layer-nn}) trained by GD serves as the estimator $\hat f$.

\noindent \textbf{Minimax Lower Risk Bound for Learning a Low-Degree Spherical Polynomial.}
The established result in
\cite[Theorem 2(a)]{RaskuttiWY12-minimax-sparse-additive} gives the minimax optimal lower bound for kernel regression with the kernel $K$, that is,
 $\inf_{\hat f_n} \sup_{f^* \in \cH_{\Kr}(\gamma_0)} \Expect{\bx}{\pth{\hat f_n(\bx) -f^*(\bx)}^2} \gsim d^{\ell_0}/{n}$, where the infimum is taken over all measurable functions of the training sample $\set{\bbx_i,y_i}_{i=1}^n$.
This result suggests that the minimax optimal lower bound for the regression risk with $K$ is $\Theta(r_0/{n}) =\Theta(d^{\ell_0}/{n})$, which is provably achieved by the
two-layer NN (\ref{eq:two-layer-nn}) trained by GD, to be shown by our main result in the next section.

\section{Training the Two-Layer Neural Network by Gradient Descent}
\label{sec:training}
In the training process of our two-layer NN (\ref{eq:two-layer-nn}), both the channel attention weights $\btau$
and the second-layer weights $\ba$ are optimized, and the first-layer weights $\bQ=\set{\bbq_r}_{r=1}^m$ are randomly sampled and then fixed during the training.
The following quadratic loss function is minimized during the training process:
\bal \label{eq:obj-dnns}
L(\btau,\ba) \defeq \frac{1}{2n} \sum_{i=1}^{n} \pth{f(\ba,\bbx_i) - y_i}^2.
\eal%
The training process of the two-layer NN (\ref{eq:two-layer-nn}) consists of two stages. In the first stage, one step of GD is applied to learn the channel attention weights $\btau$. With the channel attention weights learned, the activation function is set to (\ref{eq:sigma-channel-selected}), that is, $\sigma_{\btau}(\bx,\bx') =  \sum\limits_{\ell=0}^{\hell} \mu^{-\frac 12}_{\sigma,\ell}  P^{(d)}_{\ell}(\iprod{\bx}{\bx'})$.
We then train the second-layer weights $\ba$ by minimizing the objective (\ref{eq:obj-dnns}) through GD in the second training stage.
We introduce the following notations for the training process. Let $\set{Y_j}_{j=0}^{m_L-1} = \set{Y_{\ell j}}_{0 \le \ell \le \hell, j \in [N(d,\ell)]}$ as the enumeration of all the spherical harmonics of up to degree $L$. We define
$\bY(\bS,m_L) \in \RR^{n \times m_L}$ where $\bth{\bY(\bS,m_L)}_{ij} = Y_{j-1}(\bbx_i)$ for every $i \in [n]$ and $j \in [m_L]$, $ \bY(\bS,r_0) =  \bY(\bS,m_{\ell_0}) \in \RR^{n \times r_0}$ is defined similarly, and
$\bY(\bS,\ell) \in \RR^{n \times N(d,\ell)}$ where
$\bth{\bY(\bS,\ell)}_{ij} = Y_{\ell,j}(\bbx_i)$ for all $i \in [n]$ and
$j \in [N(d,\ell)]$.
 Similarly,
$\bY(\bQ,m_L) \in \RR^{m \times m_L}$ with $\bth{\bY(\bQ,m_L)}_{rj} = Y_{j-1}(\bbq_r)$ every $r \in [m]$ and $j \in [m_L]$, and $\bth{\bY(\bQ,\ell)}_{rj} = Y_{\ell,j}(\bbq_r)$ for all $r \in [m]$ and $j \in [N(d,\ell)]$.

\smallskip
\noindent \textbf{Training Stage One: Learning the Channel Attention Weights $\tau$.}
We have the initialization $\ba(0) = \bzero$ and $\tau_{\ell}(0) = 0$
for all $\ell \in [0\relcolon L]$, where $\bzero$ denotes a vector whose elements are all $0$. In this training stage, we first perform the one-step GD for $\ba$ to obtain
\bal\label{eq:channel-attention-ba-update}
\ba(1)
=\ba(0) - \eta_1 \left.\nabla_{\ba} L(\btau,\ba)\right|_{\ba =
\bzero, \btau_{\ell}  = \mu^{-1}_{\sigma,\ell}, \forall \ell \in [0\relcolon L]}
= \frac{1}{n {\sqrt m}} \bY(\bQ,m_L) \bY^{\top}(\bS,m_L) \by,
\eal
where the learning rate $\eta_1 = 1$. $\btau(1)$ is then obtained by one-step of GD with $\ba = \ba(1)$ by
\bal\label{eq:channel-attention-btau-update}
\tau_{\ell}(1) &= \tau_{\ell}(0) - \eta_2
\left.\frac{\diff \partial L(\btau,\ba)}{\partial \tau_{\ell}}\right|_{(\btau,\ba) = (\bzero,\ba(1))}
= \frac{1}{n {\sqrt m}} \by^{\top} \bY(\bS,\ell)  \bY^{\top}(\bQ,\ell)
\ba(1)
\eal
for all $\ell \in [0\relcolon L]$, where $\eta_2 = N(d,\ell)$. We note that the initialization of $\btau(0) = \bzero$ is used in the one-step GD update for $\btau(1)$  in
(\ref{eq:channel-attention-btau-update}), and a different initialization $\btau(0)$ is used in (\ref{eq:channel-attention-ba-update}). Theorem~\ref{theorem:channel-selection} below shows that w.h.p., when
$\min\set{n,m} \ge \Theta(m_L)\log({4m_L}/{\delta})$,  after performing the one-step GD update for the channel attention weights by (\ref{eq:channel-attention-btau-update}), the channel attention weights of all the informative channels, defined as the channels with indices in $[0\relcolon \ell_0]$, are not smaller than $2\eps_0$ for a positive threshold  $\eps_0 \in (0,\beta_0^2/3]$. The absolute channel attention weights for the redundant channels, defined as the channels with indices in $[\ell_0+1\relcolon L]$, are smaller than $\eps_0$.
As a result, Theorem~\ref{theorem:channel-selection} gives the strong theoretical guarantee for a novel and principled learnable channel selection algorithm, described in Algorithm~\ref{alg:learnable-channel-attention},
which assigns updated  attention weights $ \mu^{-\frac 12}_{\sigma,\ell}$ to every informative channel with index $\ell$, and assigns updated  attention weights $0$ to all redundant channels. We use $\hell$ to denote the number of channels with nonzero channel attention weights after running Algorithm~\ref{alg:learnable-channel-attention}, and Theorem~\ref{theorem:channel-selection} guarantees that $\hell = \ell_0$ in (\ref{eq:sigma-channel-selected}), the activation function after running the learnable channel selection by Algorithm~\ref{alg:learnable-channel-attention}. We note that Theorem~\ref{theorem:channel-selection} needs the minimum absolute value condition on the target function that
$\min_{\ell \in [0\relcolon \ell_0],j \in [N(d,\ell)]]}
{\abth{\bbeta_{\ell,j}}} \ge \beta_0 \sqrt{\mu_{\sigma,\ell}}$ for some positive constant $\beta_0$. Due to the presence of noise in the response vector $\by$, similar minimum absolute value conditions on the target signal are in fact necessary and broadly used in standard compressive sensing literature such as~\cite{Aeron2010-information-theoretic-bound-cs} for signal recovery.
\begin{theorem}\label{theorem:channel-selection}
Assume that the minimum absolute value condition on the target function holds, that is, $\min_{\ell \in [0\relcolon \ell_0],j \in [N(d,\ell)]]}
{\abth{\bbeta_{\ell,j}}} \ge \beta_0 \sqrt{\mu_{\sigma,\ell}}$ holds for some positive constant $\beta_0$. $\eps_0$ is a positive threshold such that $\eps_0 \in (0,\beta_0^2/3]$. Let
$\set{\tau_{\ell}(1)}_{\ell=0}^L = \btau(1)$ be computed by the one-step GD
(\ref{eq:channel-attention-btau-update}).
Suppose that
\bal
m &> \max\set{\frac{256\gamma_0^4}{\eps_0^2},4}m_L
\log\pth{\frac{4m_L}{\delta}}, \label{eq:channel-selection-m-cond}
\\
n &> \max\set{\max\set{\frac{400\gamma_0^4}{\eps_0^2},4}m_L
\log\pth{\frac{4m_L}{\delta}}, \frac{32m_L(\sigma_0^2+1)}{\eps_0}, \frac{8192 \gamma_0^2 m_L(\sigma_0^2+1)}{\eps_0^2}},
\label{eq:channel-selection-n-cond}
\eal
then for every $\delta \in (0,1)$, with probability at least $1-\exp\pth{-\Theta(m_L)}-\delta$,
we have
\bal
\begin{cases}
\tau_{\ell}(1)  \ge 2\eps_0, & \ell \in [0\relcolon \ell_0],\\
\abth{\tau_{\ell}(1)} \le \eps_0, &\ell_0 < \ell \le L.
\end{cases}
\eal
\end{theorem}

\noindent \textbf{Training Stage Two: Learning the Second-Layer Weights $\ba$.}
We use GD to train the two-layer NN (\ref{eq:two-layer-nn}) with the channels attention weights updated in its activation function (\ref{eq:sigma-channel-selected}) in the first training stage.
In the $(t+1)$-th step of GD with $t \ge 0$, the second-layer weights $\ba$ are updated by one-step GD through
\bal\label{eq:GD-two-layer-nn-a}
\ba(t+1)  = \ba(t) - \frac{\eta}{n} \bZ(t) (\hat \by(t) -  \by),
\eal
where $\by_i = y_i$,  $\hat \by(t) \in \RR^n$ with $\bth{\hat \by(t)}_i = f(\ba(t),\bbx_i)$. We also denote $f(\ba(t),\cdot)$ as
$f_t(\cdot)$ as the neural network function with weighting vectors $\ba(t)$ obtained right after the $t$-th step of GD.
We define $\bZ(t) \in \RR^{r \times n}$ which is computed by
$\bth{\bZ(t)}_{ri} = 1/{\sqrt m} \cdot \sigma_{\btau}(\bbx_i,\bbq_r)$ for every $r \in [m]$ and $i \in [n]$ where $\sigma_{\btau}$ is specified by (\ref{eq:sigma-channel-selected}).
We employ the initialization $\ba(0) = \bzero$  so that $\hat \by(0) = \bzero$, that is, the initial output of the two-layer NN (\ref{eq:two-layer-nn}) is zero.
The two-layer NN is trained by GD with $T$ steps for $T \ge 1$. In the second training stage the channel attention weights $\btau$ are not updated, so we abbreviate the two-layer NN (\ref{eq:two-layer-nn}) mapping function $f(\btau, \ba,\bx)$ as $f(\ba,\bx)$.

\begin{figure}[!htbp]
\centering
\begin{minipage}{0.56\linewidth}
\begin{algorithm}[H]
\renewcommand{\algorithmicrequire}{\textbf{Input:}}
\renewcommand\algorithmicensure {\textbf{Output:}}
\caption{Learnable Channel Selection}
\label{alg:learnable-channel-attention}
\begin{algorithmic}[1]
\State $\btau \leftarrow$ Channel-Attention($\bS,\by,\eps_0$)
\State input: $\bS,\by$
\State Compute the channel attention weights $\btau(1)=\set{\tau_{\ell}(1)}_{\ell=0}^L$ by the one-step GD (\ref{eq:channel-attention-btau-update}).
\State For each $\ell \in [0 \relcolon L]$, set
$\tau_{\ell}  = \indict{\tau_{\ell}(1) \ge 2\eps_0} \mu^{-\frac 12}_{\sigma,\ell}$.
\State \Return the channel attention weights
$\btau=\set{\tau_{\ell}}_{\ell=0}^L$
\end{algorithmic}
\end{algorithm}
\end{minipage}
\hfill
\begin{minipage}{0.43\linewidth}
\begin{algorithm}[H]
\renewcommand{\algorithmicrequire}{\textbf{Input:}}
\renewcommand\algorithmicensure {\textbf{Output:}}
\caption{Training the Two-Layer NN by GD}
\label{alg:GD}
\begin{algorithmic}[1]
\State $\ba(T) \leftarrow$ Training-by-GD($T,\bQ,\ba$)
\State input: $T,\bQ,\eta, \ba(0) = \bzero$
\State \textbf{\bf for } $t=1,\ldots,T$ \,\,\textbf{\bf do }
\State \quad Perform the $t$-th step of GD by (\ref{eq:GD-two-layer-nn-a})
\State \textbf{\bf end for }
\State \Return $\ba(T)$
\end{algorithmic}
\end{algorithm}
\end{minipage}
\end{figure}

\section{Main Result}
\label{sec:main-results}

We present our main result about the sharp risk bound in Theorem~\ref{theorem:LRC-population-NN-fixed-point}, with its proof deferred to Section~\ref{sec:proofs-main-results} of the appendix.
\begin{theorem}\label{theorem:LRC-population-NN-fixed-point}
Suppose the minimum absolute value condition Theorem~\ref{theorem:channel-selection} holds, and
$\hell \le L$ nonzero attention weights are returned by the learnable channel selection algorithm described in Algorithm~\ref{alg:learnable-channel-attention} with the threshold $\eps_0 \in (0,\beta_0^2/3]$,  $c_t \in (0,1]$ is an arbitrary positive constant. Suppose the network width $m$ satisfies
\bal\label{eq:m-cond-LRC-population-NN-fixed-point}
m \gsim  \frac{n^4 \log (2n/\delta)}{d^{2\ell_0}},
\eal
and the neural network
$f(\ba(t),\cdot)$ is
trained by GD using Algorithm~\ref{alg:GD} with the constant learning rate
$\eta = \Theta(1)\in (0,1/(\ell_0+1))$ and $T \asymp n/d^{\ell_0}$.
Then for every $t \in [c_t T \colon T]$ and every $\delta \in (0,1/2)$, with probability at least
$1-7\exp\pth{- \Theta(r_0)}-\exp\pth{-\Theta(n)}-\exp\pth{-\Theta(m_L)}-2\delta$ over the random noise $\bw$, the random training features $\bS$, and the random initialization $\bQ$, $f(\ba(t),\cdot) = f_t$ satisfies
\bal\label{eq:LRC-population-NN-fixed-point}
&\Expect{P}{(f_t-f^*)^2} \lsim \Theta\pth{\frac{d^{\ell_0}}{n}}.
\eal
Here $r_0 = m_{\ell_0} = \Theta(d^{\ell_0})$.
\end{theorem}

Theorem~\ref{theorem:LRC-population-NN-fixed-point} shows that the neural network (\ref{eq:two-layer-nn}) trained by GD enjoys a sharp rate of the regression risk for learning a degree-$\ell_0$ spherical polynomial, $\Theta(d^{\ell_0}/{n})$, which is minimax optimal as explained in Section~\ref{sec:setup-kernels-target-function}.
As an immediate result, (\ref{eq:LRC-population-NN-fixed-point}) shows that the two-layer NN (\ref{eq:two-layer-nn}) trained GD enjoys a sample complexity of
$n \asymp \Theta(d^{\ell_0}/\eps)$ for any regression risk $\eps \in (0, 1)$, much lower than the sample complexity $\Theta\pth{d^{\ell_0} \max\set{\eps^{-2},\log d}}$ in the representative work~\cite{Nichani0L22-escape-ntk}. We herein compare our result with the competing results in learning low-degree spherical polynomials in Table~\ref{table:main-results-comparison} from the perspective of the sharpness of the regression risk and the algorithmic guarantees, that is, whether a finite-width neural network is trained to obtain the corresponding bound for the regression risk.

\begin{table*}[!htbp]
        \centering
        \caption{Comparison between our result and the existing works on learning low-degree polynomials on the spheres of $\RR^d$
by training over-parameterized neural networks with or without algorithmic guarantees. Almost all the results here are under a common and popular setup that $f^* \in \cH_{\tK}$ where $\tK$ is the NTK of a specific studied neural studied in each work,
and the responses $\set{y_i}_{i=1}^n$ are corrupted by i.i.d. Gaussian noise with zero mean, with \cite{Nichani0L22-escape-ntk} being the only exception where the responses are noise-free. It is remarked that the sample complexity can be straightforwardly obtained from the regression risk. The regression risk of \cite[Theorem 1]{DamianLS22-nn-representation-learning} is for the risk less than $1/\sqrt{\log d}$, with the meaning of $r$ explained in Section~\ref{sec:main-results}, and $\tilde \Theta$ hides a logarithmic factor of $\log (mnd)$.
}
        \resizebox{1\linewidth}{!}{
        \begin{tabular}{|l|c|c|c|c|}
                \hline
                \textbf{Existing Works and Our Result}
                & \textbf{Finite-Width NN is Trained} & \textbf{Sharpness of the Regression Risk}
\\ \hline
\cite[Theorem 4]{Ghorbani2021-linearized-two-layer-nn}
 & No  & \begin{tabular}{@{}c@{}}Only matching the lower bound for pointwise kernel learning, \\not minimax optimal \end{tabular}  \\ \hline
\begin{tabular}{@{}c@{}}
\cite[Theorem 7]{BaiL20-quadratic-NTK}
\end{tabular}
&Yes &Not minimax optimal  \\ \hline
\cite[Theorem 1]{Nichani0L22-escape-ntk}
&\begin{tabular}{@{}c@{}}Yes\end{tabular}
  &$\Theta(\sqrt{d^{\ell_0}/n})$, not minimax optimal   \\ \hline
\cite[Theorem 1]{DamianLS22-nn-representation-learning}
&\begin{tabular}{@{}c@{}}Yes\end{tabular}
  &\begin{tabular}{@{}c@{}} $L^1$-norm regression risk $\tilde \Theta(\sqrt{dr^{\ell_0}/n}+\sqrt{r^p/m})$,
  \\ not minimax optimal
\end{tabular}
\\ \hline
Our Result (Theorem~\ref{theorem:LRC-population-NN-fixed-point})
 &\cellcolor{blue!15}Yes &\cellcolor{blue!15}
{Minimax optimal, $\Theta\pth{\frac{d^{\ell_0}}{n}}$ }
\\ \hline
       \end{tabular}
}
\label{table:main-results-comparison}
\end{table*}
It is shown in \cite[Theorem~1]{Nichani0L22-escape-ntk} that a regression risk $\eps > 0$ can be achieved with sample complexity
$n \gsim d^{\ell_0} \max\set{\eps^{-2}, \log d}$, implying a convergence rate of order
$\Theta(\sqrt{d^{\ell_0}/n})$ when the regression risk is below $1/\sqrt{\log d}$.
This rate is not minimax optimal and is considerably less sharp than our bound.
The two-stage feature learning method of \cite{DamianLS22-nn-representation-learning} requires the restrictive assumption
that the target function depends only on $r \ll d$ input directions. Under this assumption, vanilla GD ensures that the learned
network function lies in a subspace of rank $r$ within the RKHS. Without it (i.e., $r=d$), the $L^1$-risk bound in
\cite[Theorem~1]{DamianLS22-nn-representation-learning} is at least $\tilde\Theta(\sqrt{d^{\ell_0+1}/n})$.
In contrast, since $L^p$-norm risks are non-decreasing in $p$, our $L^2$-risk bound in
Theorem~\ref{theorem:LRC-population-NN-fixed-point} immediately yields a sharper $L^1$-risk bound of $\Theta(\sqrt{d^{\ell_0}/n})$.
Furthermore, \cite{Ghorbani2021-linearized-two-layer-nn} shows that for
$\tilde \Theta(d^{\ell_0}) \le n \le \Theta(d^{\ell_0+1-\delta})$ with
$\tilde \Theta(d^{\ell_0})/d^{\ell_0} \to \infty$ as $d \to \infty$,
the NTK-based regression risk converges to zero.
However, their result requires restrictive conditions on the activation function and assumes infinite network width ($m \to \infty$).
In sharp contrast, our result establishes that the minimax-optimal regression risk can be achieved by training
finite-width neural networks with the feature learning capability by channel attention.

Beyond such feature learning approaches that aim to escape the linear NTK regime
(Table~\ref{table:main-results-comparison}),
the statistical learning literature has long established sharp convergence rates for nonparametric kernel regression
\cite{Stone1985,Yang1999-minimax-rates-convergence,
RaskuttiWY14-early-stopping-kernel-regression,
Yuan2016-minimax-additive-models}.
In particular, training over-parameterized shallow \cite[Theorem~5.2]{HuWLC21-regularization-minimax-uniform-spherical}
or deep \cite[Theorem~3.11]{SuhKH22-overparameterized-gd-minimax} neural networks with spherical-uniform training features
on the unit sphere achieves the minimax-optimal rate $\cO(n^{-d/(2d-1)})$ for the regression risk,
when the target function lies in $\cH_{\tK}(\gamma_0)$ where $\tK$ is the NTK of the respective network.

As discussed in Section~\ref{sec:setup-kernels-target-function},
since the target function $f^*$ is a degree-$\ell_0$ spherical polynomial,
it lies in the union of eigenspaces up to degree $\ell_0$.
Therefore, learning requires identifying the subspace $\cup_{\ell=0}^{\ell_0} \cH_\ell$ of dimension $r_0 = m_{\ell_0}$,
rather than the full $L^2(\cX,\mu)$.
Crucially, with a carefully designed learnable channel selection algorithm described in Algorithm~\ref{alg:learnable-channel-attention},
the goal of feature learning is achieved by setting the number of channels in the activation function of the first layer to $\ell=\ell_0$ w.h.p.
In this way, the NTK of the two-layer NN~(\ref{eq:two-layer-nn}) in the second training stage becomes a low-rank kernel $K=\Kr$~(\ref{eq:K-def}) of rank $r_0$,
whose eigenspaces corresponding to nonzero eigenvalues span all and only spherical harmonics of degree up to $\ell_0$.
Consequently, vanilla GD on such a two-layer NN with sufficient width $m$
can fit the target $f^*$ using the $r_0$ eigenfunctions of $K$,
thereby attaining the minimax-optimal regression rate.
The roadmap for the proof of this main result is provided in Section~\ref{sec:proof-roadmap},
following the necessary background on kernel complexity.

\section{Roadmap of Proofs}
\label{sec:proof-roadmap}

The summary of the  approaches and key technical results in the proofs are presented as follows. We first introduce kernel complexity in Section~\ref{sec:kernel-complexity}, a key concept in our results and their proofs.
Section~\ref{sec:detailed-roadmap-key-results} details the roadmap, key technical results in the proofs, our novel proof strategies and insights from our theoretical results.

\subsection{Kernel Complexity}
\label{sec:kernel-complexity}
The local kernel complexity has been studied by
\cite{bartlett2005,koltchinskii2006,Mendelson02-geometric-kernel-machine}. Let $\set{\lambda_i}_{i=1}^{m_{\hell}}$ be the enumeration of the distinct eigenvalues of the integral operator $T_K$, $\set{\mu_{\ell}}_{\ell=0}^{\hell}$, where each eigenvalue repeat as many times as its multiplicity in the sequence
 $\set{\lambda_i}_{i=1}^{m_{\hell}}$. We let $\lambda_i = 0$ for all $i > m_{\hell}$. For the PD kernel $K$, we define the empirical kernel complexity $\hat R_K$
and the population kernel complexity $R_K$ as
\bal\label{eq:kernel-LRC-empirical}
\hat R_{K}(\eps) \defeq \sqrt{\frac 1n
\sum\limits_{i=1}^n \min\set{\hat \lambda_i,\eps^2}}, \quad R_{K}(\eps) \defeq \sqrt{\frac 1n \sum\limits_{i=1}^{\infty} \min\set{\lambda_i,\eps^2}}.
\eal
It can be verified that both
 $\sigma_0 R_K(\eps)$ and $\sigma_0 \hat R_K(\eps)$ are
 sub-root functions \cite{bartlett2005} in terms of $\eps^2$.
 The formal definition of sub-root functions is deferred to
 Definition~\ref{def:sub-root-function} in  the appendix.
For a given noise ratio $\sigma_0$, the critical empirical radius
$\hat \eps_n > 0$ is the smallest positive solution to the inequality
$\hat R_K(\eps) \le {\eps^2}/{\sigma_0}$, where
$\hat \eps_n^2$ is the also the fixed point of $\sigma_0 \hat R_K(\eps)$ as a function
of $\eps^2$: $\sigma_0 \hat R_K(\hat \eps_n) = \hat \eps_n^2$.
Similarly, the critical population rate $\eps_n$ is
defined to be the smallest positive solution to the inequality
$R_K(\eps) \le {\eps^2}/{\sigma_0}$, where
$\eps_n^2$ is the fixed point of $\sigma_0 \hat R_K(
\eps)$ as a function of
$\eps^2$: $\sigma_0 R_K(\eps_n) = \eps_n^2$. In this paper we consider the case that $n\eps_n^2 \to \infty$ as $n \to \infty$, which is also used in standard analysis of nonparametric regression with minimax rates by kernel regression
~\cite{RaskuttiWY14-early-stopping-kernel-regression}.
We also define $\eta_t \defeq \eta t$ for all $t \ge 0$.

\subsection{Detailed Roadmap and Key Results}
\label{sec:detailed-roadmap-key-results}
We present the roadmap of our theoretical results which lead to the main result, Theorem~\ref{theorem:LRC-population-NN-fixed-point}, in this
section.
Before presenting the key technical results, we note the by performing learnable channel selection algorithm described in Algorithm~\ref{alg:learnable-channel-attention}, Theorem~\ref{theorem:channel-selection} guarantees that $\hell = \ell_0$ w.h.p. Therefore, the condition on $\ell$ is satisfied in all the results of this section and Theorem~\ref{theorem:LRC-population-NN-fixed-point}. Moreover, all the technical results in this section are for the second training stage, that is, training the second-layer weights $\ba$ by the standard GD.
Our main result,
Theorem~\ref{theorem:LRC-population-NN-fixed-point}, is built upon the following three significant technical results of independent interest.

First, we can have the following principled decomposition of the neural network function at any step of GD into a function in the RKHS associated with the NTK (\ref{eq:K-def}), which is $\cH_{K}(B_h)$, and an error function with a small $L^{\infty}$-norm.

\begin{theorem}\label{theorem:bounded-NN-class-informal}
Suppose $\hell = \Theta(1) \ge \ell_0$, the network width $m$ is sufficiently large and finite,
and the neural network $f_t = f(\ba(t),\cdot) $ is trained by GD with constant learning rate $\eta = \Theta(1) \in (0,1/\hell)$. Then for every $t \in [T]$, w.h.p.,
$f_t$ has the following decomposition on $\cX$: $f_t = h_t + e_t$,
where $h_t \in \cH_{K}(B_h)$ with $B_h$ defined in (\ref{eq:B_h}) of the appendix,
$e_t  \in L^{\infty}$ with sufficient small $\supnorm{e_t }$.
\end{theorem}
The proof of Theorem~\ref{theorem:bounded-NN-class-informal} relies on the uniform convergence of the empirical kernel $\hK$ to the corresponding population kernel $K$, established by the following theorem, which is proved by the concentration inequality for independent random variables taking values in the RKHS associated with the PD activation function $\sigma$, $\cH_{\sigma}$.

\begin{theorem}\label{theorem:Km-close-to-K-supnorm-informal}
Suppose $\hell = \Theta(1)$. For any fixed $\bx' \in \cX$ and every $\delta \in (0,1)$, with probability at least
$1-\delta$ over the random initialization $\bQ = \set{\bbq_r}_{r=1}^m$, we have
\bals
\sup_{\bx \in \cX}\abth{\hK(\bx,\bx') - K(\bx,\bx')} \lsim d^{\hell} \sqrt{\frac{\log 2/{\delta}}{m}}.
\eals
\end{theorem}
Theorem~\ref{theorem:Km-close-to-K-supnorm-informal} is proved as Theorem~\ref{theorem:Km-close-to-K-supnorm} in the appendix.
Theorem~\ref{theorem:bounded-NN-class-informal} shows that, w.h.p., the neural network function $f(\ba(t),\cdot)$ right after the $t$-th step of GD can be decomposed into
two functions by $f(\ba(t),\cdot) = f_t = h  + e$, where $h \in \cH_{K}(B_h)$ is a function in the RKHS associated with $K$ with a bounded $\cH_{K}$-norm.
The error function $e$ has a small $L^{\infty}$-norm, that is, $\supnorm{e} \le w$ with $w$ being a small number controlled by the network width
$m$, and larger $m$ leads to smaller $w$.

Second, local Rademacher complexity is employed
to tightly bound the risk of nonparametric regression in
Theorem~\ref{theorem:LRC-population-NN-eigenvalue-informal} below, which is based on
 the Rademacher complexity of a localized subset of the function class $\cF(B_h,w)$ in
Lemma~\ref{lemma:LRC-population-NN} in the appendix. We  use
Theorem~\ref{theorem:bounded-NN-class-informal} and Lemma~\ref{lemma:LRC-population-NN}
to derive Theorem~\ref{theorem:LRC-population-NN-eigenvalue-informal}.
\begin{theorem}\label{theorem:LRC-population-NN-eigenvalue-informal}
Suppose $\hell = \Theta(1) \ge \ell_0$, the network width $m$ is sufficiently large and finite,
and the neural network
$f_t = f(\ba(t),\cdot)$ is
trained by GD with constant learning rate $\eta > 0$.
Then for every $t \in [T]$, w.h.p.,
\bal\label{eq:LRC-population-NN-bound-eigenvalue-informal}
&\Expect{P}{(f_t-f^*)^2} - 2 \Expect{P_n}{(f_t-f^*)^2}
\lsim \frac{d^{\hell}}{n} + w.
\eal
\end{theorem}
It is remarked that the regression risk $\Expect{P}{(f_t-f^*)^2}$ is bounded by the sum of the training loss and a small term ${d^{\hell}}/{n} + w$ through Theorem~\ref{theorem:LRC-population-NN-eigenvalue-informal}. $w$ is an arbitrarily small positive number with sufficiently large network width $m$. The sharp rate ${d^{\ell_0}}/{n}$ on the regression risk bound (\ref{eq:LRC-population-NN-bound-eigenvalue-informal}) in Theorem~\ref{theorem:LRC-population-NN-eigenvalue-informal}
is due to the finite rank $m_{\hell} = \Theta(d^{\hell})$ of the kernel $K$ with $\hell = \Theta(1)$.

Third, we have the following sharp upper bound for the training loss $\Expect{P_n}{(f_t-f^*)^2}$.
\begin{theorem}\label{theorem:empirical-loss-bound-informal}
Suppose $\hell = \Theta(1) \ge \ell_0$, the neural network trained after the $t$-th step of GD, $f_t = f(\ba(t),\cdot)$, satisfies $\bu(t) = f_t(\bS) - \by = \bv(t) + \be(t)$
with $\bv(t) \in \cV_t$, $\be(t) \in \cE_{t,\tau}$. If
$\eta \in (0,1/\hell)$ and $\tau$ is suitably small,
then for every $t \in [T]$, w.h.p., we have 
\bal\label{eq:empirical-loss-bound-informal}
\Expect{P_n}{(f_t-f^*)^2} \le \Theta\pth{\frac {\gamma_0^2}{\eta t}}.
\eal
\end{theorem}
We then obtain Theorem~\ref{theorem:LRC-population-NN-fixed-point} using the upper bound (\ref{eq:LRC-population-NN-bound-eigenvalue-informal}) for the regression risk in Theorem~\ref{theorem:LRC-population-NN-eigenvalue-informal} where $w$ is set to ${d^{\hell}}/{n}$, with the empirical loss $\Expect{P_n}{(f_t-f^*)^2}$ bounded by
$\Theta({d^{\hell}}/{n})$ w.h.p. by (\ref{eq:empirical-loss-bound-informal}) in Theorem~\ref{theorem:empirical-loss-bound-informal}, and $\hell = \ell_0$ w.h.p.

Detailed proofs of all the technical results of this paper are deferred to the appendix. In particular, Theorem~\ref{theorem:bounded-NN-class}, Theorem~\ref{theorem:Km-close-to-K-supnorm},
Theorem~\ref{theorem:LRC-population-NN-eigenvalue}, and Theorem~\ref{theorem:empirical-loss-bound} in the appendix
are the formal versions of Theorem~\ref{theorem:bounded-NN-class-informal}, Theorem~\ref{theorem:Km-close-to-K-supnorm-informal}, Theorem~\ref{theorem:LRC-population-NN-eigenvalue-informal}, and Theorem~\ref{theorem:empirical-loss-bound-informal} in this section.
The proof of Theorem~\ref{theorem:LRC-population-NN-fixed-point} is presented in Section~\ref{sec:proofs-main-results} of the appendix.

\subsection{Novel Proof Strategies}
\label{sec:novel-proof-strategy}

We remark that the proof strategies of our main result,
Theorem~\ref{theorem:LRC-population-NN-fixed-point}, summarized above are significantly different from the existing works in training over-parameterized neural networks for nonparametric regression with minimax rates
\cite{HuWLC21-regularization-minimax-uniform-spherical,
SuhKH22-overparameterized-gd-minimax,Li2024-edr-general-domain} and existing works about learning low-degree polynomials
\cite{Ghorbani2021-linearized-two-layer-nn,BaiL20-quadratic-NTK,Nichani0L22-escape-ntk,DamianLS22-nn-representation-learning}.

First, a novel learnable channel selection algorithm is used to select the informative channels in the activation function of the first-layer of the network (\ref{eq:two-layer-nn}), and the selected channel number $\ell$ is the ground truth channel number $\ell_0$ in the target function w.h.p. Such channel selection ensures that the kernel $K$ is in fact the low-rank kernel $\Kr$, ensuring the sharp regression risk bound for the second training stage.

Second, GD is carefully incorporated into the analysis about the uniform convergence results for NTK (\ref{eq:K-def}) in Theorem~\ref{theorem:Km-close-to-K-supnorm-informal}, leading to the crucial decomposition of the neural network function $f_t$ in Theorem~\ref{theorem:bounded-NN-class-informal}. It is remarked that while existing works such as \cite{Li2024-edr-general-domain} also has uniform convergence results for over-parameterized neural network, our results about the uniform convergence for the NTK, rooted in the martingale based concentration inequality for Banach space-valued process~\cite{Pinelis1992},
do not depend on the H\"older continuity of the NTK.

Third, to the best of our knowledge, Theorem~\ref{theorem:LRC-population-NN-eigenvalue-informal} is the first result about the sharp upper bound
of the order $\Theta( {d^{\hell}}/{n})$ with $w={d^{\hell}}/{n}$ for the regression risk of the neural network function which has the decomposition in Theorem~\ref{theorem:bounded-NN-class-informal}. We note that the regression risk in Theorem~\ref{theorem:LRC-population-NN-eigenvalue-informal} is $\Theta({d^{\hell}}/{n})=\Theta({d^{\ell_0}}/{n})$ w.h.p., which has the expected and the desired order since the target function is in a $r_0$-dimensional subspace of the RKHS $\cH_K(\gamma_0)$ with $r_0 = \Theta(d^{\ell_0})$.
Moreover, the proof of Theorem~\ref{theorem:bounded-NN-class-informal}, Theorem~\ref{theorem:LRC-population-NN-eigenvalue-informal}, and Theorem~\ref{theorem:empirical-loss-bound-informal} employ the kernel complexity introduced in Section~\ref{sec:kernel-complexity}. In fact,
the term $\Theta({d^{\hell}}/{n})$ corresponds to the fixed point of the kernel complexity $R_K$.

\section{Conclusion}

We study nonparametric regression
by training an over-parameterized two-layer neural network with channel attention
where the target function is in the RKHS
associated with the NTK of the neural network and also a degree-$\ell_0$ spherical polynomial on the unit sphere in $\RR^d$.
We show that, through the feature learning capability of the network by a novel learnable channel selection algorithm, the neural network with channel attention trained by the vanilla Gradient Descent (GD) renders a sharp and minimax optimal regression risk bound of $\Theta(d^{\ell_0}/{n})$. Novel proof strategies are employed to achieve this result, and we compare our results to the current state-of-the-art with a detailed roadmap of our technical approaches and results.

\newpage

\begin{appendices}

The appendix of this paper is organized as follows.
We present the basic mathematical results employed in our proofs in Section~\ref{sec::math-tools}, and then introduce the detailed technical background about harmonic analysis on spheres in Section~\ref{sec:harmonic-analysis-detail}. Detailed proofs are presented in Section~\ref{sec:proofs}.

\section{Mathematical Tools}
\label{sec::math-tools}


The Rademacher complexity of a function class and its empirical version are defined below.
\begin{definition}\label{def:RC}
Let $\bsigma = \set{\sigma_i}_{i=1}^n$ be $n$ i.i.d. random variables such that $\Pr[\sigma_i = 1] = \Pr[\sigma_i = -1] = \frac{1}{2}$. The Rademacher complexity of a function class $\cF$ is defined as
\bal\label{eq:RC}
&\cfrakR(\cF) = \Expect{\set{\bbx_i}_{i=1}^n, \set{\sigma_i}_{i=1}^n}{\sup_{f \in \cF} {\frac{1}{n} \sum\limits_{i=1}^n {\sigma_i}{f(\bbx_i)}} }.
\eal%
The empirical Rademacher complexity is defined as
\bal\label{eq:empirical-RC}
&\hat \cfrakR(\cF) = \Expect{\set{\sigma_i}_{i=1}^n} { \sup_{f \in \cF} {\frac{1}{n} \sum\limits_{i=1}^n {\sigma_i}{f(\bbx_i)}} },
\eal%
For simplicity of notations, Rademacher complexity and empirical Rademacher complexity are also denoted by $\Expect{}{\sup_{f \in \cF} {\frac{1}{n} \sum\limits_{i=1}^n {\sigma_i}{f(\bbx_i)}} }$ and $\Expect{\bsigma}{\sup_{f \in \cF} {\frac{1}{n} \sum\limits_{i=1}^n {\sigma_i}{f(\bbx_i)}}}$, respectively. 
\end{definition}

For data $\set{\bbx}_{i=1}^n$ and a function class $\cF$, we define the notation $R_n \cF$ by $R_n \cF \coloneqq \sup_{f \in \cF} \frac{1}{n} \sum\limits_{i=1}^n \sigma_i f(\bbx_i)$.

\begin{theorem}[{\cite[Theorem 2.1]{bartlett2005}}]
\label{theorem:Talagrand-inequality}
Let $\cX,P$ be a probability space, $\set{\bbx_i}_{i=1}^n$ be independent random variables distributed according to $P$. Let $\cF$ be a class of functions that map $\cX$ into $[a, b]$. Assume that there is some $r > 0$
such that for every $f \in \cF$,$\Var{f(\bbx_i)} \le r$. Then, for every $x > 0$, with probability at least $1-e^{-x}$,
\bal\label{eq:Talagrand-inequality}
\resizebox{0.99\hsize}{!}{$
\sup_{f \in \cF} \big( \E_{P} [f(\bx)] - \E_{P_n } [f(\bx)] \big) \le \inf_{\alpha > 0} \bigg( 2(1+\alpha) \E_{\set{\bbx_i}_{i=1}^n,\set{\sigma_i}_{i=1}^n} [R_n \cF] + \sqrt{\frac{2rx}{n}} + (b-a) \pth{\frac{1}{3}+\frac{1}{\alpha}} \frac{x}{n} \bigg)$},
\eal%
and with probability at least $1-2e^{-x}$,
\bal\label{eq:Talagrand-inequality-empirical}
\resizebox{0.995\hsize}{!}{$
\sup_{f \in \cF} \big( \E_{P} [f(\bx)] - \E_{P_n } [f(\bx)] \big)
\le \inf_{\alpha \in (0,1)} \bigg(\frac{2(1+\alpha)}{1-\alpha} \E_{\set{\sigma_i}_{i=1}^n} [R_n \cF] + \sqrt{\frac{2rx}{n}}
 + (b-a) \pth{ \frac{1}{3}+\frac{1}{\alpha} + \frac{1+\alpha}{2\alpha(1-\alpha)} } \frac{x}{n} \bigg)$}.
\eal%
$P_n$ is the empirical distribution over $\set{\bbx_i}_{i=1}^n$ with
$\Expect{P_n}{f(\bx)} = \frac{1}{n} \sum\limits_{i=1}^n f(\bbx_i)$. Moreover, the same results hold for $\sup_{f \in \cF} \big( \E_{P_n } [f(\bx)] - \E_{P} [f(\bx)] \big)$.
\end{theorem}

In addition, we have the contraction property for Rademacher complexity, which is due to Ledoux
and Talagrand~\cite{Ledoux-Talagrand-Probability-Banach}.

\begin{theorem}\label{theorem:RC-contraction}
Let $\phi$ be a contraction,that is, $\abth{\phi(x) - \phi(y)} \le \mu \abth{x-y}$ for $\mu > 0$. Then, for every function class $\cF$,
\bal\label{eq:RC-contraction}
&\Expect{\set{\sigma_i}_{i=1}^n} {R_n \phi \circ \cF} \le \mu \Expect{\set{\sigma_i}_{i=1}^n} {R_n \cF},
\eal%
where $\phi \circ \cF$ is the function class defined by $\phi \circ \cF = \set{\phi \circ f \colon f \in \cF}$.
\end{theorem}

\begin{definition}[Sub-root function,{\cite[Definition 3.1]{bartlett2005}}]
\label{def:sub-root-function}
A function $\psi \colon [0,\infty) \to [0,\infty)$ is sub-root if it is nonnegative,
nondecreasing and if $\frac{\psi(r)}{\sqrt r}$ is nonincreasing for $r >0$.
\end{definition}

\begin{theorem}[{\cite[Theorem 3.3]{bartlett2005}}]
\label{theorem:LRC-population}
Let $\cF$ be a class of functions with ranges in $[a, b]$ and assume
that there are some functional $T \colon \cF \to \RR+$ and some constant $\bar B$ such that for every $f \in \cF$ , $\Var{f} \le T(f) \le \bar B P(f)$. Let $\psi$ be a sub-root function and let $r^*$ be the fixed point of $\psi$.
Assume that $\psi$ satisfies that, for any $r \ge r^*$,
$\psi(r) \ge \bar B \cfrakR(\set{f \in \cF \colon T (f) \le r})$. Fix $x > 0$,
then for any $K_0 > 1$, with probability at least $1-e^{-x}$,
\bals
\forall f \in \cF, \quad \Expect{P}{f} \le \frac{K_0}{K_0-1} \Expect{P_n}{f} + \frac{704K_0}{\bar B} r^*
+ \frac{x\pth{11(b-a)+26 \bar B K_0}}{n}.
\eals
Also, with probability at least $1-e^{-x}$,
\bals
\forall f \in \cF, \quad \Expect{P_n}{f} \le \frac{K_0+1}{K_0} \Expect{P}{f}  + \frac{704K_0}{\bar B} r^*
+ \frac{x\pth{11(b-a)+26 \bar B K_0}}{n}.
\eals
\end{theorem}

\begin{lemma}[{\cite[Lemma 3.4]{bartlett2005}}]
\label{lemma:RC-fstar-class-is-sub-root}
If a function class $\cF$ is star-shaped around a function $\hat f$, and $T \colon \cF \to \RR^{+}$ with
$\RR^{+}$ being the set of all nonnegative real numbers is a (possibly random) function that satisfies
$T (\alpha f) \le \alpha^2 T (f)$ for every $f \in \cF$ and any $\alpha \in [0, 1]$, then the (random) function
$\psi$ defined for $r \ge 0$ by $\psi(r) \defeq \Expect{\set{\sigma_i}_{i=1}^n}{R_n \set{f - \hat f \colon f \in \cF, T(f-\hat f) \le r}}$
is sub-root and $r \to \Expect{\set{\bbx_i}_{i=1}^n}{\psi(r)}$ is also sub-root.
\end{lemma}

\section{Detailed Technical Background about Harmonic Analysis on Spheres}
\label{sec:harmonic-analysis-detail}
In this section, we provide background materials on spherical
harmonic analysis needed for our study of the RKHS. We refer the reader to \cite{chihara2011introduction,Efthimiou-spherical-harmonics-in-p-dim,
szego1975orthogonal} for further information
on these topics. As mentioned above, expansions in spherical harmonics were used in the past in the statistics literature, such as
\cite{Bach17-breaking-curse-dim,BiettiM19}.

With $\ell \ge 0$, let $\cP^{\textup{(hom)}}_{\ell}$ denote the space of all the degree-$\ell$ homogeneous polynomials on $\cX = \unitsphere{d-1}$, and let $\cH_{\ell}$ denote the space of degree-$\ell$ homogeneous harmonic polynomials on $\cX$, or the degree-$\ell$  spherical harmonics. That is,
\bal\label{eq:spherical-harmonic-degree-ell}
\cH_{\ell}  = \set{P \colon \cX \to \RR \colon P(\bx)
=\sum_{\abth{\alpha} = \ell} c_{\alpha} \bx^{\alpha},  \Delta P = 0},
\eal
where $\alpha = \bth{\alpha_1,\ldots,\alpha_d}$, $\bx^{\alpha} = \prod_{i=1}^d \bx_i^{\alpha_i}$, $\abth{\alpha} = \sum_{i=1}^d \alpha_i$, and $\Delta$ is the Laplacian operator.  For $\ell \neq \ell'$, the elements of $\cH_{\ell}$ and $\cH_{\ell'}$ are orthogonal to each other.
All the functions in the following text of this section are assumed to be elements of
$L^2(\cX,v_{d-1})$, where $v_{d-1}$ stands for the uniform distribution
on the sphere $\cX = \unitsphere{d-1}$. We have
$\iprod{f}{g}_{L^2} \coloneqq \int_{\cX}f(x)g(x)
 \diff v_{d-1}(x)$. We denote by $\set{Y_{kj}}_{j \in [N(d,k)]}$ the spherical harmonics of degree $k$ which form an orthogonal basis of $\cH_k$, where
 $N(d,k) = \frac{2k+d-2}{k} {k+d-3 \choose d-2}$ is the dimension of $\cH_k$. They form an orthonormal basis of $L^2(\cX,v_{d-1})$. We have
 $\sum_{j=1}^{N(d,k)} Y_{kj}(\bx) Y_{kj}(\bx') = N(d,k) P^{(d)}_k(\iprod{\bx}{\bx'}) $ for all $\bx,\bx' \in \cX$, where
 $P^{(d)}_k$ is the $k$-th Legendre polynomial in dimension $d$, which is also known as Gegenbauer polynomials,
given by the Rodrigues formula:
\bal\label{eq:Gegenbauer-polynomial-def}
P^{(d)}_k(t) = (-\frac 12)^k \frac{\Gamma\pth{\frac{d-1}{2}}}{\Gamma\pth{k+\frac{d-1}{2}}} \pth{1-t^2}^{(3-d)/2} \pth{\frac{\diff }{\diff t}}^k \pth{1-t^2}^{k+(d-3)/2}.
\eal
The polynomials $\set{P^{(d)}_k}$ are orthogonal in $L^2(\cX,\diff v_{d-1})$ where the measure $\diff v_{d-1}$ is given by $\diff v_{d-1}(t) = (1-t^2)^{(d-3)/2} \diff t$, and we have
\bals
\int_{-1}^1 {P^{(d)}_k}^2(t) (1-t^2)^{(d-3)/2} \diff t = \frac{w_{d-1}}{w_{d-2}} \frac{1}{N(d,k)},
\eals
where $w_{d-1} \defeq \frac{2 \pi^{d/2}}{\Gamma(d/2)}$ denotes the surface of the unit sphere $\unitsphere{d-1}$. It follows from the
 orthogonality of spherical harmonics that
\bals
\int_{\cX} P^{(d)}_j(\iprod{\bx}{\bw}) P^{(d)}_j(\iprod{\bx'}{\bw}) \diff v_{d-1}(\bw) = \frac{\delta_{jk}}{N(d,k)} P^{(d)}_k(\iprod{\bx}{\bx'}),
\eals
where $\delta_{jk} = \indict{j=k}$. We have the following recurrence relation \cite[Equation 4.36]{Efthimiou-spherical-harmonics-in-p-dim},
\bal\label{eq:Gegenbauer-polynomial-recursive}
t P^{(d)}_k(t) = \frac{k}{2k+d-2} P^{(d)}_{k-1}(t) + \frac{k+d-2}{2k+d-2} P^{(d)}_{k+1}(t)
\eal
for all $k \ge 1$, and $t P^{(d)}_0(t) = P^{(d)}_1(t)$, and $P^{(d)}_0 \equiv 1$. It follows that $P^{(d)}_k(1) = 1$ for all $k \ge 0$, and it can be verified
that $\abth{P^{(d)}_k(t)} \le 1$ for all $k \ge 0$ and $t \in [-1,1]$.

The Funk-Hecke formula is helpful for computing Fourier coefficients in the basis of spherical
harmonics in terms of Legendre polynomials. For any $j \in [N(d,k)]$, we have
\bals
\int_{\cX} f(\iprod{\bx}{\bx'}) Y_{kj}(\bx') \diff v_{d-1}(\bx') = \frac{w_{d-2}}{w_{d-1}}  Y_{kj}(\bx) \int_{-1}^1 f(t) P^{(d)}_k(t)  (1-t^2)^{(d-3)/2} \diff t.
\eals
For a positive-definite kernel $\tK (\bx,\bx') = \kappa (\iprod{\bx}{\bx'})$ defined on $\cX$, we have its Mercer decomposition as follows.
\bals
\tK(\bx,\bx') = \sum\limits_{\ell \ge 0} \mu_{\ell} \sum\limits_{j=1}^{N(d,\ell)} Y_{\ell,j}(\bx)Y_{\ell,j}(\bx')
=  \sum\limits_{\ell \ge 0} \mu_{\ell} N(d,\ell) P^{(d)}_{\ell}(\iprod{\bx}{\bx'}),
\eals
where $\mu_{\ell}$ is the eigenvalue of the integral operator $T_{\tK}$ associated with $\tK$ corresponding to $\cH_{\ell}$.
It follows that
\bals
\mu_{\ell} = \frac{w_{d-2}}{w_{d-1}} \int_{-1}^1 \kappa(t) P^{(d)}_{\ell}(t) (1-t^2)^{(d-3)/2} \diff t.
\eals

\begin{proposition}
[{\cite[Theorem 4.2]{Krylov-harmonic-polynomials}}]
\label{prop:docomp-homogeneous-poly}
Let $p \in \cP^{\textup{(hom)}}_{\ell}$. Then there exists unique $h_{n-2i} \in \cH_{n-2i}$ for $i \in \set{0,1,\ldots,\floor{n/2}}$ such that
\bals
p(\bx) = h_n + h_{n-2} + \ldots + h_{n-2k}.
\eals
\end{proposition}
\begin{theorem}
\label{theorem:spherical-polynomial-representation-spherical-harmonics}
Every polynomial $p$ defined on $\unitsphere{d-1}$ of degree $k$ for $k \ge 0$ can be represented as a linear combination of homogeneous harmonic polynomials up to degree $k$, that is,
\bals
p = \sum\limits_{i=0}^k c_i p_i,
\eals
where $p_i \in \cH_i$ for $i \in \set{0,1,\ldots,k}$.
\end{theorem}
\begin{proof}
Every polynomial $p$ defined on $\unitsphere{d-1}$ of degree $k$ can be represented as the sum of homogeneous polynomials on $\unitsphere{d-1}$ by grouping the terms of $p$ of the same degree together. It follows from Proposition~\ref{prop:docomp-homogeneous-poly} that every homogeneous polynomial is a linear combination of homogeneous harmonic polynomials up to degree $k$. As a result, the conclusion holds.
\end{proof}

\begin{lemma}
\label{lemma:r0-estimate}
For $\ell_0 = \Theta(1)$ and $d > \Theta(1)$, we have
\bal\label{eq:r0-estimate}
r_0 = \Theta(d^{\ell_0}).
\eal
\end{lemma}
\begin{proof}
It follows from the direct calculation that $N(d,\ell) \asymp d^{\ell}$ under the given conditions, so that $r_0 = \sum_{\ell=0}^{\ell_0} N(d,\ell) \asymp d^{\ell_0}$.
\end{proof}

\begin{lemma}
[Efficient Computation of the Activation Function $\sigma$ Defined in
(\ref{eq:sigma-activation-def})]
\label{lemma:efficient-computation-activation}
For every given $\bx,\bx' \in \cX$ and the channel attention weights $\btau$, $\sigma_{\btau}(\bx,\bx')$ can be computed in $\Theta(1)$ time.
\end{lemma}
\begin{proof}
We note that $\sigma_{\btau}(\bx,\bx')$ is computed by
$\sigma_{\btau}(\bx,\bx') =  \sum\limits_{\ell=0}^{L} \tau_{\ell}
P^{(d)}_{\ell}(t)$ with $t = \iprod{\bx}{\bx'}$.
Using the recursive formula (\ref{eq:Gegenbauer-polynomial-recursive}) and standard dynamic programming,
$\set{P^{(d)}_{\ell}(t)}_{\ell=0}^L$
can be computed in $\Theta(L)$ time. To see this, we note that $P^{(d)}_{0}(t) = 1$,
and the computation of
$P^{(d)}_{\ell'}(t)$ for every $\ell' \in [1\relcolon L]$ takes
$\Theta(1)$ time by (\ref{eq:Gegenbauer-polynomial-recursive}) using the stored values of $\set{P^{(d)}_{\ell}(t)}_{\ell=0}^{\ell'-1}$. Summing all the $\tau_{\ell} P^{(d)}_{\ell}(t)$ takes $\Theta(L)$, so the computation of $\sigma_{\btau}(\bx,\bx')$ takes $\Theta(L)$ time in total.
\end{proof}

\section{Detailed Proofs}
\label{sec:proofs}

We present detailed proofs for the theoretical results that lead to our main result, Theorem~\ref{theorem:LRC-population-NN-fixed-point}, in this section.
The proof of Theorem~\ref{theorem:LRC-population-NN-fixed-point} is presented in Section~\ref{sec:proofs-main-results}, followed by the basic definitions and the detailed proofs of our other technical results. We first introduce the definition of stopping time which serves as the upper bound for the  number of steps $T$ in Algorithm~\ref{alg:GD}.

\noindent \textbf{Definition of Stopping Time.} Recall that $\eta_t = \eta t$ for all $t > 0$, we then define the
 stopping time $\hat T$  as
\bal\label{eq:stopping-time-hatT}
\hat T
 \defeq \min\set{t \colon \hat R_{K}(\sqrt{1/\eta_t}) > (\sigma_0 \eta_t)^{-1}}-1.
\eal
The stopping time in fact is the upper bound for the number of steps $T$ for Algorithm~\ref{alg:GD}, that is, $T \le \hat T$. In the proof of
Theorem~\ref{theorem:LRC-population-NN-fixed-point}, we will show that
$\hat T \asymp  n/d^{\hell}$, so that it is always feasible to choose $T \le \hat T$ such that $T \asymp n/d^{\hell}$. Throughout this appendix we let $T \le \hat T$.

\subsection{Proof of Theorem~\ref{theorem:LRC-population-NN-fixed-point}}
\label{sec:proofs-main-results}

\begin{proofof}{Theorem~\ref{theorem:LRC-population-NN-fixed-point}}
We use Theorem~\ref{theorem:LRC-population-NN-eigenvalue} and
Theorem~\ref{theorem:empirical-loss-bound} in this appendix to prove this
theorem. Theorem~\ref{theorem:LRC-population-NN-eigenvalue} and Theorem~\ref{theorem:empirical-loss-bound}
are the formal versions of Theorem~\ref{theorem:LRC-population-NN-eigenvalue-informal} and Theorem~\ref{theorem:empirical-loss-bound-informal}, respectively.

First of all, it follows by
Theorem~\ref{theorem:empirical-loss-bound} that with probability at least
$1-\exp\pth{- \Theta(n\hat\eps_n^2)}$  over $\bw$,
\bals
\Expect{P_n}{(f_t-f^*)^2} \le \Theta\pth{\frac {1}{\eta t}} .
\eals
Plugging such bound for $\Expect{P_n}{(f_t-f^*)^2}$ in
(\ref{eq:LRC-population-NN-bound-eigenvalue})
of Theorem~\ref{theorem:LRC-population-NN-eigenvalue}
leads to
\bal\label{eq:LRC-population-NN-fixed-point-seg1}
&\Expect{P}{(f_t-f^*)^2} \lsim
\Theta\pth{\frac {1}{\eta t}}+\frac{d^{\hell}}{n} + w.
\eal
By the definition of $\hat T$ and $\hat \eps_n^2$, we have
\bals
\hat \eps_n^2 \le \frac {1}{\eta \hat T} \le \frac {2}{\eta (\hat T+1)}
\le 2 \hat \eps_n^2,
\eals
so that $\hat T \asymp \hat \eps_n^{-2} $. Furthermore, it follows from \cite[Corollary 4]{RaskuttiWY14-early-stopping-kernel-regression} that $\eps_n^2 \asymp r_0/n$. In addition, Lemma~\ref{lemma:hat-eps-eps-relation}
suggests that with probability $1-4\exp(-\Theta(n\eps_n^2)) =
1-4\exp(-\Theta(r_0))$, $\hat \eps_n^2  \asymp \eps_n^2$. As a result,
$\hat T \asymp  n/d^{\hell}$, and we choose $T \le \hat T$ such that $T \asymp n/d^{\hell}$.

Due to the setting that $T \asymp n/d^{\hell}$ and $\eta = \Theta(1)$, we have
\bal\label{eq:LRC-population-NN-fixed-point-seg2}
\frac{1}{\eta t} \asymp \frac{1}{\eta T}
\asymp \frac{d^{\hell}}{n}.
\eal
Let $w = {d^{\hell}}/{n}$, then $w \in (0,1)$ with $n > d^{\hell}$. (\ref{eq:LRC-population-NN-fixed-point}) then follows from
(\ref{eq:LRC-population-NN-fixed-point-seg1}) with $w={d^{\hell}}/{n}$,
(\ref{eq:LRC-population-NN-fixed-point-seg2}) and the union bound.
We note that $c_{\bu}$ is bounded by  a positive constant, so
that the condition on $m$ in (\ref{eq:m-cond-bounded-NN-class})
in Theorem~\ref{theorem:bounded-NN-class}, together with $w ={d^{\hell}}/{n}$ and
(\ref{eq:LRC-population-NN-fixed-point-seg2}) leads to
the condition on $m$ in (\ref{eq:m-cond-LRC-population-NN-fixed-point}).

This theorem is then proved by noting that Theorem~\ref{theorem:channel-selection} guarantees that $\hell = \ell_0$ holds with probability at least $1-\exp\pth{-\Theta(m_L)}-\delta$, where $\hell$ is the number of channels selected by the learnable channel selection algorithm described in Algorithm~\ref{alg:learnable-channel-attention}.

\end{proofof}

\subsection{Basic Definitions}
\label{sec:proofs-definitions}
We introduce the following definitions for our analysis.
We define
\bal\label{eq:ut}
\bu(t) \defeq \hat \by(t) - \by
\eal
as the difference between the network output
$\hat \by(t)$ and the training response vector $\by$ right after the $t$-th step of GD.
Let $\tau \le 1$ be a positive number. For $t \ge 0$ and $T \ge 1$ we define the following quantities:
$c_{\bu} \defeq \Theta(\gamma_0) + \sigma_0 + \tau + 1$,
\bal\label{eq:cV_S}
\cV_t \defeq \set{\bv \in \RR^n \colon \bv = -\pth{\bI_n- \eta \bK_n }^t f^*(\bS)},
\eal
\bal\label{eq:cE_S}
\cE_{t,\tau} \defeq &\set{\be \colon \be = \bbe_1 + \bbe_2 \in \RR^n, \bbe_1 = -\pth{\bI_n-\eta\bK_n}^t \bw,
\ltwonorm{\bbe_2} \le {\sqrt n} \tau }.
\eal
In particular, Theorem~\ref{theorem:empirical-loss-convergence} in the appendix shows that w.h.p. over the random noise $\bw$
and the random initialization $\bQ$, $\bu(t)$ can be composed into two vectors,
$\bu(t) = \bv(t) + \be(t)$ such that $\bv(t) \in \cV_t$
and $\be(t) \in \cE_{t,\tau}$. We then define the set of the neural network weights during the training by GD as follows:
\bal\label{eq:weights-nn-on-good-training}
&\cA(\bS,\ba,T) \defeq \left\{\ba \colon \exists t \in [T] {\textup{ s.t. }}\ba = - \sum_{t'=0}^{t-1} \frac{\eta}{n}  \bZ(t') \bu(t'), \right. \nonumber \\
& \left. \bu(t') \in \RR^{n}, \bu(t') = \bv(t') + \be(t'),
\bv(t') \in \cV_{t'}, \be(t') \in \cE_{t',\tau}, {\textup { for all } } t' \in [0,t-1] \vphantom{\frac12}  \right\}.
\eal%
The set of the functions represented by the neural network with weights in $\cA(\bS,\ba,T)$ is then defined as
\bal\label{eq:random-function-class}
\cFnn(\bS,\ba,T) \defeq \set{f_t = f(\ba(t),\cdot) \colon \exists \, t \in [T], \ba(t) \in \cA(\bS,\ba,T)}.
\eal%
We also define the function class $\cF(B,w)$ for any $B,w > 0$ as
\bal\label{eq:func-class-B-w}
\cF(B,w) \defeq \set{f \colon f = h+e, h \in \cH_K(B) , \supnorm{e} \le w}.
\eal
We will show by
Theorem~\ref{theorem:bounded-NN-class-informal} in the next subsection
that w.h.p. over $\bw$,
$\cFnn(\bS,\ba,T)$ is a subset of $\cF(B,w)$, where a smaller
$w$ requires a larger network width $m$, and $B_h > \gamma_0$ is an absolute positive constant defined by
\bal
B_h &\defeq \gamma_0 + \sqrt{2} + 1. \label{eq:B_h}
\eal

\subsection{Proofs for Results in Section~\ref{sec:detailed-roadmap-key-results}}
\label{sec:proofs-key-results}

We present our key technical results regarding optimization and generalization of the two-layer NN (\ref{eq:two-layer-nn}) trained by GD in this section.
The following theorem, Theorem~\ref{theorem:bounded-NN-class}, is the formal version of Theorem~\ref{theorem:bounded-NN-class-informal} in Section~\ref{sec:detailed-roadmap-key-results},
and it states that w.h.p. over $\bw$,
$\cFnn(\bS,\ba,T) \subseteq \cF(B_h,w)$.
\begin{theorem}\label{theorem:bounded-NN-class}
Suppose $\hell = \Theta(1) \ge \ell_0$. Suppose $w  \in (0,1)$, the network width $m$ satisfies
\bal\label{eq:m-cond-bounded-NN-class}
m \gsim
\max\set{T^2 d^{2\hell}\log (2n/\delta)/{w^2}, T^4 d^{2\hell}\log (2n/\delta)},
\eal
and the neural network
$f_t = f(\ba(t),\cdot) $ is trained by GD with the constant learning rate $\eta \in (0,1/(\hell+1))$ and $\eta = \Theta(1)$. Then for every $t \in [T]$ and every $\delta \in (0,1)$, with probability at least
$1-\exp\pth{- \Theta(n\hat\eps_n^2)}-\exp\pth{-\Theta(n)}-\delta$ over the random initialization $\bQ$ and the random noise $\bw$,
$f_t\in \cFnn(\bS,\ba,T)$, and $f_t$ has the following decomposition on $\cX$:
\bal\label{eq:nn-function-class-decomposition}
f_t = h_t + e_t,
\eal
where $h_t \in \cH_{K}(B_h)$ with $B_h$ defined in (\ref{eq:B_h}),
$e_t  \in L^{\infty}$ with $\supnorm{e_t } \le w$.
\end{theorem}

Based on Theorem~\ref{theorem:bounded-NN-class} and the local Rademacher complexity based analysis
\cite{bartlett2005},
Theorem~\ref{theorem:LRC-population-NN-eigenvalue} presents
a sharp upper bound for the nonparametric regression risk,
$\Expect{P}{(f_t-f^*)^2}$, where $f_t$ is the function represented by
the two-layer NN (\ref{eq:two-layer-nn}) right after the
$t$-th step of GD.
Theorem~\ref{theorem:LRC-population-NN-eigenvalue} is the formal version of Theorem~\ref{theorem:LRC-population-NN-eigenvalue-informal} in
Section~\ref{sec:detailed-roadmap-key-results}.
\begin{theorem}\label{theorem:LRC-population-NN-eigenvalue}
Suppose $\hell = \Theta(1) \ge \ell_0$, $w  \in (0,1)$,
$m$ satisfies (\ref{eq:m-cond-bounded-NN-class}),
and the neural network
$f_t = f(\ba(t),\cdot)$ is
trained by GD with the constant learning rate
$\eta \in (0,1/(\hell+1))$ and $\eta = \Theta(1)$.
Then for every $t \in [T]$ and every $\delta \in (0,1)$, with probability at least
$1-\exp\pth{- \Theta(n\hat\eps_n^2)}-\exp(-{m_{\hell}})-\exp\pth{-\Theta(n)}-\delta$
over the random noise $\bw$, the random training features $\bS$, and the random initialization $\bQ$,
\bal\label{eq:LRC-population-NN-bound-eigenvalue}
&\Expect{P}{(f_t-f^*)^2} - 2 \Expect{P_n}{(f_t-f^*)^2}
\lsim \frac{d^{\hell}}{n} + w.
\eal
\end{theorem}

Theorem~\ref{theorem:empirical-loss-bound} below shows
that the empirical loss $\Expect{P_n}{(f_t-f^*)^2}$ is bounded by
$\Theta(1/(\eta t))$ w.h.p. over $\bw$.
Theorem~\ref{theorem:empirical-loss-bound} is the formal version of
Theorem~\ref{theorem:empirical-loss-bound-informal} in Section~\ref{sec:detailed-roadmap-key-results}.
Such upper bound for the empirical loss by Theorem~\ref{theorem:empirical-loss-bound}
will be plugged in the risk bound in
Theorem~\ref{theorem:LRC-population-NN-eigenvalue} to prove
Theorem~\ref{theorem:LRC-population-NN-fixed-point}. The proofs of
Theorem~\ref{theorem:LRC-population-NN-fixed-point} and its
corollary are presented in the next subsection.

\begin{theorem}\label{theorem:empirical-loss-bound}
Suppose  $\hell = \Theta(1) \ge \ell_0$, the neural network trained after the $t$-th step of GD, $f_t = f(\ba(t),\cdot)$, satisfies $\bu(t) = f_t(\bS) - \by = \bv(t) + \be(t)$
with $\bv(t) \in \cV_t$, $\be(t) \in \cE_{t,\tau}$. If
\bal\label{eq:eta-tau-cond-empirical-loss-convergence}
\eta \in (0,1/(\hell+1)), \quad \tau \le \frac{1}{\eta T},
\eal
then for every $t \in [T]$, with probability at least
$1-\exp\pth{- \Theta(n\hat\eps_n^2)}$ over the random noise $\bw$, we have
\bal\label{eq:empirical-loss-bound}
\Expect{P_n}{(f_t-f^*)^2} \le \Theta\pth{\frac {1}{\eta t}}.
\eal
\end{theorem}

\newtheorem{innercustomthm}{{\bf{Theorem}}}
\newenvironment{customthm}[1]
  {\renewcommand\theinnercustomthm{#1}\innercustomthm}
  {\endinnercustomthm}


\subsubsection{Proof of
Theorem~\ref{theorem:bounded-NN-class}}

We prove Theorem~\ref{theorem:bounded-NN-class} in this subsection. The proof requires the following theorem, Theorem~\ref{theorem:empirical-loss-convergence}, about our main result about the optimization of the network (\ref{eq:two-layer-nn}). Theorem~\ref{theorem:empirical-loss-convergence} states that w.h.p. over the random noise $\bw$ and the random initialization $\bQ$, the weights of the network $\ba(t)$ obtained right after the $t$-th step of GD
belongs to $\cA(\bS,\ba,T)$. The proof of Theorem~\ref{theorem:empirical-loss-convergence} is
based on
Lemma~\ref{lemma:yt-y-bound} and
Lemma~\ref{lemma:empirical-loss-convergence-contraction} deferred to
Section~\ref{sec:lemmas-main-results} of this appendix.

\begin{theorem}\label{theorem:empirical-loss-convergence}
Suppose  $\hell = \Theta(1) \ge \ell_0$,
\bal\label{eq:m-cond-empirical-loss-convergence}
m \gsim  T^2 d^{2\hell}\log (2n/\delta)/{\tau^2},
\eal
and the neural network $f(\ba(t),\cdot)$ trained by GD with the constant learning rate $\eta = \Theta(1) \in (0,1/(\hell+1))$. Then
with probability at least
$1-\exp\pth{-\Theta(n)}-\delta$ over the random noise $\bw$ and the random initialization $\bQ$,
$\ba(t) \in \cA(\bS,\ba,T)$ for every $t \in [T]$.
Moreover, for every $t \in [0,T]$, $\bu(t) = \bv(t) + \be(t)$ where
$\bu(t) = \hat \by(t) -  \by$, $\bv(t) \in \cV_t$,
$\be(t) \in \cE_{t,\tau}$,  $\ltwonorm{\bu(t)} \le c_{\bu} \sqrt n$.
\end{theorem}

\begin{proofof}{Theorem~\ref{theorem:empirical-loss-convergence}}
First, when $m \gsim  T^2 d^{2\hell}\log (2n/\delta)/{\tau^2}$ with a proper constant, it can be verified that $\bE_{m,n,\eta} \le {\tau {\sqrt n}}/{T}$ where $\bE_{m,n,\eta}$ is specified by
(\ref{eq:empirical-loss-Et-bound}) of
Lemma~\ref{lemma:empirical-loss-convergence-contraction}.
We then use mathematical induction to prove this theorem. We will first prove that $\bu(t) = \bv(t) + \be(t)$ where $\bv(t) \in \cV_t$,
$\be(t) \in \cE_{t,\tau}$, and $\ltwonorm{\bu(t)} \le c_{\bu} \sqrt n$ for all $t \in [0,T]$.

When $t = 0$, we have
\bal\label{eq:empirical-loss-convergence-seg1}
\bu(0) = - \by &= \bv(0) + \be(0),
\eal
where $\bv(0) \defeq -f^*(\bS) = -\pth{\bI-\eta \bK_n }^0 f^*(\bS)$,
$\be(0) = -\bw$ with
$\be(0) = -\pth{\bI-\eta \bK_n } ^0 \bw$. Therefore,
$\bv(0) \in \cV_{0}$ and $\be(0) \in \cE_{0,\tau}$.
Also, it follows from the proof of Lemma~\ref{lemma:yt-y-bound}
that $\ltwonorm{\bu(0)} \le  c_{\bu}{\sqrt n}$ with probability at least
$1 -  \exp\pth{-\Theta(n)}$
over the random noise $\bw$.

Suppose that for all $t_1 \in[0,t]$ with $t \in [0,T-1]$, $\bu(t_1) = \bv(t_1) + \be(t_1)$ where $\bv(t_1) \in \cV_{t_1}$, $\be(t_1) \in \cE_{t_1,\tau}$, and $\ltwonorm{\bu(t_1)} \le  c_{\bu}{\sqrt n}$. Then it follows from Lemma~\ref{lemma:empirical-loss-convergence-contraction} that the recursion
$\bu(t'+1)  = \pth{\bI- \eta \bK_n  }\bu(t')+\bE(t'+1)$ holds for all $t' \in [0,t]$.
 As a result, we have
\bal\label{eq:empirical-loss-convergence-seg5}
\bu(t+1)  &= \pth{\bI- \eta \bK_n  }\bu(t)+\bE(t+1) \nonumber \\
&= -\pth{\bI-\eta \bK_n }^{t+1} f^*(\bS)
 -\pth{\bI-\eta \bK_n }^{t+1} \bw+\sum_{t'=1}^{t+1}
 \pth{\bI-\eta \bK_n}^{t+1-t'} \bE(t')
\nonumber \\
&=\bv(t+1) + \be(t+1),
\eal
where $\bv(t+1)$ and $\be(t+1)$ are defined as
\bal\label{eq:empirical-loss-convergence-vt-et-def}
\bv(t+1) \defeq-\pth{\bI-\eta \bK_n}^{t+1} f^*(\bS)\in \cV_{t+1},
\eal
\bal\label{eq:empirical-loss-convergence-et-pre}
&\be(t+1) \defeq \underbrace{-\pth{\bI-\eta \bK_n}^{t+1} \bw}_{\bbe_1(t+1)}
+ \underbrace{ \sum_{t'=1}^{t+1}
 \pth{\bI-\eta \bK_n}^{t+1-t'} \bE(t') }_{\bbe_2(t+1)}.
\eal
We now prove the upper bound for $\bbe_2(t+1)$.
With $\eta \in (0,1/(\hell+1))$, we have $\ltwonorm{\bI - \eta \bK_n} \in (0,1)$.
It follows that
\bal\label{eq:empirical-loss-convergence-et-bound}
&\ltwonorm{\bbe_2(t+1)} \le \sum_{t'=1}^{t+1} \ltwonorm{\bI-\eta \bK_n}^{t+1-t'}\ltwonorm{\bE(t')}
\le  \tau {\sqrt n},
\eal
where the last inequality follows from the fact
that $\ltwonorm{\bE(t)} \le \bE_{m,n,\eta} \le {\tau {\sqrt n}}/{T}$ for all $t \in [T]$. It follows that $\be(t+1) \in \cE_{t+1,\tau}$.
Also, since $\hell \ge \ell_0$, it follows from Lemma~\ref{lemma:yt-y-bound}
that
\bals
\ltwonorm{\bu(t+1)} &\le \ltwonorm{\bv(t+1)} + \ltwonorm{\bbe_1(t+1)}
+\ltwonorm{\bbe_2(t+1)} \nonumber \\
&\le \pth{\frac{\gamma_0}{ \sqrt{2e\eta } } + \sigma_0+\tau+1} {\sqrt n} \le \ c_{\bu}{\sqrt n}.
\eals
The above inequality completes the induction step, which also completes the proof.
\end{proofof}

\begin{proofof}{Theorem~\ref{theorem:bounded-NN-class}}
In this proof we abbreviate $f_t$ as $f$.
It follows from Theorem~\ref{theorem:empirical-loss-convergence}
and its proof that conditioned on an event $\Omega$ with probability at
least $1-\exp\pth{-\Theta(n)}-\delta$,
$f \in \cFnn(\bS,\ba,T)$. Moreover, $f = f(\ba,\cdot)$ with $\ba = \set{a_r}_{r=1}^m \in \cA(\bS,\ba,T)$, where $\bu(t') \in \RR^n, \bu(t') = \bv(t') + \be(t')$ with $\bv(t') \in \cV_{t'}$ and $\be(t') \in \cE_{t',\tau}$ for all $t' \in [0,t-1]$.
$\ba$ is expressed as
\bal\label{eq:bounded-Linfty-function-class-wr}
\ba = \ba(t) &= - \sum_{t'=0}^{t-1} \frac{\eta}{n} \bZ(t')\bu(t')
\eal%
for some $t \in [T]$.
Using (\ref{eq:bounded-Linfty-function-class-wr}), $g(\bx)$
is expressed as
\bal\label{eq:bounded-Linfty-function-class-seg1}
f(\bx) = f(\ba,\bx) &=  {-} \sum_{t'=0}^{t-1} \frac{1}{\sqrt m}
\sum\limits_{r=1}^m  \sigma_{\btau}(\bx,\bbq_r) \frac{\eta}{n}
\bth{\bZ(t')}_{r}
 \bu(t')   \nonumber \\
&= -\sum_{t'=0}^{t-1}
\underbrace{\frac{\eta}{n}
\sum\limits_{j=1}^n \hK(\bx,\bbx_j) \bth{ \bu(t')}_j }_{\defeq G_{t'}(\bx)},
\eal
For each $G_{t'}$ in the RHS of
 (\ref{eq:bounded-Linfty-function-class-seg1}), we have
\bal\label{eq:bounded-Linfty-function-class-Gt}
&G_{t'}(\bx) = \frac{\eta}{n}
\sum\limits_{j=1}^n \hK(\bx,\bbx_j) \bth{ \bu(t')}_j
\stackrel{\circled{1}}{=}
\frac{\eta}{n}
\sum\limits_{j=1}^n K(\bx,\bbx_j)\bth{ \bu(t')}_j +
\underbrace{\frac{\eta}{n} \sum\limits_{j=1}^n q_j \bth{ \bu(t')}_j}_{\defeq E(\bx)}.
\eal
where $q_j \defeq\hK(\bx,\bbx_j) - K(\bx,\bbx_j)$
for all $j \in [n]$ in $\circled{1}$.
We now analyze each term on the RHS of (\ref{eq:bounded-Linfty-function-class-Gt}).
Let $h(\cdot,t') \colon \cX \to \RR$ be defined by
$h(\bx,t') \defeq \frac{\eta}{n} \sum\limits_{j=1}^n K(\bx,\bbx_j) \bth{ \bu(t')}_j$,
then $h(\cdot,t') \in \cH_K$ for each $t' \in [0,t-1]$. We  define
\bal\label{eq:bounded-Linfty-function-class-h}
h_t(\cdot) \defeq -\sum_{t'=0}^{t-1} h(\cdot,t') \in \cH_K,
\eal
We note that w.h.p., $\bu(t') \le c_{\bu} \sqrt n$. Since $\hell = \Theta(1)$, it follows from (\ref{eq:Km-close-to-K-supnorm}) in Theorem~\ref{theorem:Km-close-to-K-supnorm}
that $\abth{q_j} \lsim d^{\hell} \sqrt{\frac{\log (2n/{\delta})}{m}}$ for all $j \in [n]$. As a result,
we have
\bal\label{eq:bounded-Linfty-function-class-E1-bound}
\supnorm{E} = \supnorm{\frac{\eta}{n} \sum\limits_{j=1}^n q_j \bu_j(t')} &\lsim \frac{\eta}{n} c_{\bu} \sqrt n  \cdot \sqrt{n}  d^{\hell} \sqrt{\frac{\log (2n/{\delta})}{m}}
\lsim \eta c_{\bu}  d^{\hell} \sqrt{\frac{\log (2n/{\delta})}{m}} .
\eal
Combining (\ref{eq:bounded-Linfty-function-class-Gt}) and
(\ref{eq:bounded-Linfty-function-class-E1-bound}),
any $t' \in [0,t-1]$,
\bal\label{eq:bounded-Linfty-function-class-Gt-ht-bound}
\sup_{\bx \in \cX} \abth{G_{t'}(\bx)-h(\bx,t')}
&\le \supnorm{E}
\lsim \eta c_{\bu}  d^{\hell} \sqrt{\frac{\log (2n/{\delta})}{m}} .
\eal
Define $e_t \defeq f(\ba,\cdot) - h_t$.
It then follows from (\ref{eq:bounded-Linfty-function-class-seg1}) and
(\ref{eq:bounded-Linfty-function-class-Gt-ht-bound})
that
\bal\label{eq:bounded-Linfty-function-class-f-h-bound}
\supnorm{e_t} &\le \sup_{\bx \in \cX} \abth{f(\ba,\bx) -h_t(\bx)} \nonumber \le \sum\limits_{t'=0}^{t-1}
\sup_{\bx \in \cX} \abth{G_{t'}(\bx)-h(\bx,t')} \nonumber \\
& \lsim
\eta c_{\bu}  Td^{\hell} \sqrt{\frac{\log (2n/{\delta})}{m}}
\defeq \Delta_{m,n,\eta,T}.
\eal
It follows that, for any $w \in (0,1)$, when
$m \gsim T^2 d^{2\hell}\log (2n/\delta)/{w^2}$,
we have $\Delta_{m,n,\eta,T} \le w$.

It follows from Lemma~\ref{lemma:bounded-Linfty-vt-sum-et} that
with probability at least $1 - \exp\pth{- \Theta(n\hat\eps_n^2)}$ over the random noise $\bw$,
$\norm{h_t}{\cH_K} \le B_h$,
where $B_h$ is defined in (\ref{eq:B_h}),
and $\tau$ is required to satisfy $\tau \lsim 1/(\eta T) $.

Theorem~\ref{theorem:empirical-loss-convergence} requires that
$m \gsim  T^2 d^{2\hell}\log (2n/\delta)/{\tau^2}$. As a result,
we also need to have
\bals
m \gsim \eta^2 T^4 d^{2\hell}\log (2n/\delta),
\eals
which leads to the condition (\ref{eq:m-cond-bounded-NN-class}) on $m$ with $\eta = \Theta(1)$.
\end{proofof}

\subsubsection{Proof of Theorem~\ref{theorem:LRC-population-NN-eigenvalue}}
We need the following lemma, Lemma~\ref{lemma:LRC-population-NN}, which gives a sharp upper bound for the Rademacher complexity of a localized function class as a subset of the function class $\cF(B,w)$,
and then prove Theorem~\ref{theorem:LRC-population-NN-eigenvalue} using Lemma~\ref{lemma:LRC-population-NN}.

\begin{lemma}\label{lemma:LRC-population-NN}
For every $B,w > 0$ every $r > 0$,
\bal\label{eq:LRC-population-NN}
&\cfrakR
\pth{\set{f \in \cF(B,w) \colon \Expect{P}{f^2} \le r}}
\le \varphi_{B,w}(r),
\eal%
where
\bal\label{eq:varphi-LRC-population-NN}
\varphi_{B,w}(r) &\defeq
\min_{Q \colon Q \ge 0} \pth{({\sqrt r} + w) \sqrt{\frac{Q}{n}} +
B
\pth{\frac{\sum\limits_{q = Q+1}^{\infty}\lambda_q}{n}}^{1/2}} + w.
\eal

\end{lemma}

\begin{proofof}{Theorem~\ref{theorem:LRC-population-NN-eigenvalue}}
It follows from Theorem~\ref{theorem:bounded-NN-class} that for every $t \in [T]$, conditioned on an event $\Omega$ with probability at least $1-\exp\pth{- \Theta(n\hat\eps_n^2)}-\exp\pth{-\Theta(n)}-\delta$ over $\bQ$ and $\bw$, we have $\ba(t) \in \cA(\bS,\ba,T)$,
and $f(\ba(t),\cdot) = f_t \in \cFnn(\bS,\ba,T)$.
Moreover, conditioned on the event $\Omega$, $f_t = h_t + e_t$ where $h_t \in \cH_{K}(B_h)$ and
$e_t  \in L^{\infty}$ with $\supnorm{e_t } \le w$.

We then derive the sharp upper bound for $\Expect{P}{(f_t-f^*)^2} $ by
applying Theorem~\ref{theorem:LRC-population} to the function class
$\cF = \set{F=\pth{f- f^*}^2 \colon f \in
\cF(B_h,w)  }$.
Since $B_0 \defeq (B_h+\gamma_0)\sqrt{\hell+1} + 1
\ge (B_h+\gamma_0)\sqrt{\hell+1}  + w$, then
$\supnorm{F} \le B^2_0$ with $F \in \cF$, so that
$\Expect{P}{F^2} \le B^2_0\Expect{P}{F}$.
Let $T(F) = B^2_0\Expect{P}{F}$ for $F \in \cF$. Then
$\Var{F} \le \Expect{P}{F^2} \le T(F) = B^2_0\Expect{P}{F}$.

We have
\bal\label{eq:LRC-population-NN-seg1}
&\cfrakR \pth{\set{F \in \cF \colon T(F) \le r}} = \cfrakR
\pth{ \set{(f-f^*)^2 \colon  f \in \cF(B_h,w),\Expect{P}{(f-f^*)^2}
\le \frac r{B^2_0}}} \nonumber \\
&\stackrel{\circled{1}}{\le} 2B_0 \cfrakR \pth{\set{f-f^* \colon
  f\in \cF(B_h,w), \Expect{P}{(f-f^*)^2} \le \frac{r}{B_0^2}}}\nonumber \\
&\stackrel{\circled{2}}{\le} 4B_0 \cfrakR \pth{\set{f\colon
  f \in \cF(B_h,w), \Expect{P}{f^2} \le \frac{r}{4B_0^2}}}.
\eal
where $\circled{1}$ is due to the contraction property of
Rademacher complexity in Theorem~\ref{theorem:RC-contraction}. Since $f^* \in \cF(B_h,w)$,
$f \in \cF(B_h,w)$, we have $\frac{f-f^*}{2} \in \cF(B_h,w)$ due to the fact that
$\cF(B_h,w)$ is  symmetric and convex, and it follows that $\circled{2}$
holds.

It follows from (\ref{eq:LRC-population-NN-seg1})
and Lemma~\ref{lemma:LRC-population-NN} that
\bal\label{eq:LRC-population-NN-seg2}
B^2_0 \cfrakR \pth{\set{F \in \cF \colon T(F) \le r}}
&\le 4 B_0^3 \cfrakR \pth{ \set{f \colon f \in \cF(B_h,w) , \Expect{P}{f^2} \le
\frac{r}{4B_0^2}}} \nonumber \\
&\le 4 B_0^3 \varphi_{B_h,w}\pth{\frac{r}{4B_0^2}} \defeq \psi(r).
\eal
$\psi$ defined as the RHS of (\ref{eq:LRC-population-NN-seg2}) is a sub-root function since it is nonnegative, nondecreasing and
$\frac{\psi(r)}{\sqrt r}$ is nonincreasing. Let $r^*$ be the fixed point of $\psi$, and $0 \le r \le r^*$. It follows from {\cite[Lemma 3.2]{bartlett2005}} that
$0 \le r \le \psi(r) =  4 B_0^3 \varphi\pth{\frac{r}{4B_0^2}}$.
Therefore, by the definition of $\varphi$ in (\ref{eq:varphi-LRC-population-NN}),
for every $0 \le Q \le n$, we have
\bal\label{eq:LRC-population-NN-seg3}
\frac{r}{4 B_0^3} \le \pth{ \frac{\sqrt r}{2B_0} + w} \sqrt{\frac{Q}{n}} +
B_h
\pth{\frac{\sum\limits_{q = Q+1}^{\infty}\lambda_q}{n}}^{1/2}+w.
\eal
Solving the quadratic inequality (\ref{eq:LRC-population-NN-seg3}) for $r$, we have
\bal\label{eq:LRC-population-NN-seg4}
r \le \frac{8B_0^4 Q}{n} + 8B_0^3
\pth{ w \pth{\sqrt{\frac{Q}{n}}+1}
+ B_h \pth{\frac{\sum\limits_{q = Q+1}^{\infty}\lambda_q}{n}}^{1/2}
}.
\eal
(\ref{eq:LRC-population-NN-seg4}) holds for every $0 \le Q \le n$,
so we have
\bal\label{eq:LRC-population-NN-seg5}
r \le 8 B_0^3 \min_{0 \le Q \le n} \pth{\frac{ B_0Q}{n} +w \pth{\sqrt{\frac{Q}{n}}+1}
+ B_h
\pth{\frac{\sum\limits_{q = Q+1}^{\infty}\lambda_q}{n}}^{1/2}}.
\eal
It then follows from (\ref{eq:LRC-population-NN-seg2}) and
Theorem~\ref{theorem:LRC-population} that with probability at least
$1-\exp(-x)$
over the random training features $\bS$,
\bal\label{eq:LRC-population-NN-risk-E1-bound}
&\Expect{P}{(f_t-f^*)^2} - \frac{K_0}{K_0-1} \Expect{P_n}{(f_t-f^*)^2}-\frac{x\pth{11B_0^2+26B_0^2 K_0}}{n} \le \frac{704K_0}{B_0^2} r^*,
\eal
or
\bal\label{eq:LRC-population-NN-risk-E1-bound-simple}
&\Expect{P}{(f_t-f^*)^2} - 2 \Expect{P_n}{(f_t-f^*)^2} \lsim r^* +\ \frac {x}n,
\eal
with $K_0 = 2$ in (\ref{eq:LRC-population-NN-risk-E1-bound}).
It follows from (\ref{eq:LRC-population-NN-seg5})
and (\ref{eq:LRC-population-NN-risk-E1-bound-simple}) with $Q=m_{\hell}$ that
\bal\label{eq:LRC-population-NN-seg6}
&\Expect{P}{(f_t-f^*)^2} - 2 \Expect{P_n}{(f_t-f^*)^2} \lsim \frac{{m_{\hell}}}{n} +w \pth{\sqrt{\frac{Q}{n}}+1}
+ B_h
\pth{\frac{\sum\limits_{q = {m_{\hell}}+1}^{\infty}\lambda_q}{n}}^{1/2} + \frac {x}n.
\eal
We note that $\lambda_q = 0$ for all $q > m_{\hell}$ in (\ref{eq:LRC-population-NN-seg6}),
and the above argument requires Theorem~\ref{theorem:bounded-NN-class} which holds with probability
at least $1-\exp\pth{- \Theta(n\hat\eps_n^2)}-\exp\pth{-\Theta(n)}-\delta$ over the random noise $\bw$.
Setting $x = m_{\hell}$  in (\ref{eq:LRC-population-NN-seg6}) and noting that $m_{\hell} = \Theta(d^{\hell})$ due to $\hell = \Theta(1)$  prove (\ref{eq:LRC-population-NN-bound-eigenvalue}).

\end{proofof}

\begin{proofof}{Theorem~\ref{theorem:empirical-loss-bound}}
We have
\bal\label{eq:empirical-loss-bound-seg1}
f_t(\bS) = f^*(\bS) + \bw + \bv(t) + \be(t),
\eal
where $\bv(t) \in \cV_{t}$, $\be(t) \in \cE_{t,\tau}$,
$\bbe(t) = \bbe_1(t) + \bbe_2(t)$ with
$\bbe_1(t) = -\pth{\bI_n-\eta\bK_n}^{t} \bw$
and $\ltwonorm{\bbe_2(t)} \le {\sqrt n} \tau$.
We have $\eta \lambda_1 \in (0,1)$ if $\eta \in (0,1/(\hell+1))$.
It follows from (\ref{eq:empirical-loss-bound-seg1}) that
\bsal\label{eq:empirical-loss-bound-seg2}
&\Expect{P_n}{(f_t-f^*)^2}
=\frac 1n \ltwonorm{f_t(\bS) - f^*(\bS)}^2  =\frac 1n \ltwonorm{\bv(t)+\bw+\be(t)}^2 \nonumber \\
&=\frac 1n \ltwonorm{-\pth{\bI- \eta \bK_n }^t f^*(\bS)
+\pth{\bI_n -\pth{\bI_n-\eta\bK_n}^t }\bw +\bbe_2(t)}^2 \nonumber \\
&\stackrel{\circled{1}}{\le} \frac 3n \sum\limits_{i=1}^n
\pth{1 - \eta \hlambda_i }^{2t}
\bth{{\bU}^{\top} f^*(\bS)}_i^2 + \frac 3n \sum\limits_{i=1}^{n} \pth{1-
\pth{1-\eta \hlambda_i  }^t}^2
\bth{{\bU}^{\top} \bw}_i^2 + \frac 3n \ltwonorm{\bbe_2(t)}^2 \nonumber \\
&
\nonumber \\
&\stackrel{\circled{2}}{\le} \frac{3\mu_0^2}{ 2e\eta t } + \frac 3n \sum\limits_{i=1}^{n} \pth{1-
\pth{1-\eta \lambda_i  }^t}^2
\bth{{\bU}^{\top} \bw}_i^2
+ 3 \tau^2 \nonumber \\
&\le \Theta\pth{\frac{1}{\eta t}} +
3\cdot \underbrace{\frac 1n \sum\limits_{i=1}^{n} \pth{1-\pth{1-\eta \lambda_i  }^t}^2
\bth{{\bU}^{\top} \bw}_i^2}_{ \defeq E_{\eps}}
= \Theta\pth{\frac{1}{\eta t}} + 3 E_{\eps}.
\esal

Here $\circled{1}$ follows from the Cauchy-Schwarz inequality,
$\circled{2}$ follows from (\ref{eq:yt-y-bound-seg1}) in the proof of
Lemma~\ref{lemma:yt-y-bound}.
We then derive the upper bound for $E_{\eps}$ on the RHS of
(\ref{eq:empirical-loss-bound-seg2}). We define the diagonal matrix
$\bR \in \RR^{n \times n}$ with $\bR_{ii} =
\pth{1-\pth{1-\eta \lambda_i  }^t}^2$.
Then we have
\bals
E_{\eps} = 1/n \cdot \tr{\bU \bR \bU^{\top} \bw \bw^{\top} }
\eals
It follows from \cite{quadratic-tail-bound-Wright1973}
that
\bal\label{eq:empirical-loss-bound-E-1}
&\Prob{1/n \cdot \tr{\bU \bR \bU^{\top} \bw \bw^{\top} } -
\Expect{}{1/n \cdot \tr{\bU \bR \bU^{\top} \bw \bw^{\top} }} \ge u}
\nonumber \\
&\le \exp\pth{-c \min\set{nu/\ltwonorm{\bR},n^2u^2/\fnorm{\bR}^2}}
\eal
holds for all $u > 0$, and $c$ is a  positive constant. With
$\eta_t = \eta t$ for all $t \ge 0$, we have
\bal\label{eq:empirical-loss-bound-E-2}
\Expect{}{1/n \cdot \tr{\bU \bR \bU^{\top} \bw \bw^{\top} }}
&\le \frac {\sigma_0^2}n \sum\limits_{i=1}^n
\pth{1-\pth{1-\eta \hlambda_i }^t}^2
 \stackrel{\circled{1}}{\le}
\frac {\sigma_0^2}n \sum\limits_{i=1}^n
\min\set{1,\eta_t^2 \hlambda_i^2}
\nonumber \\
&\le
\frac {{\sigma_0^2}\eta_t}n \sum\limits_{i=1}^n
\min\set{\frac{1}{\eta_t},\eta_t \hlambda_i^2}
\stackrel{\circled{2}}{\le}
\frac {{\sigma_0^2}\eta_t}n \sum\limits_{i=1}^n
\min\set{\frac{1}{\eta_t}, \hlambda_i} \nonumber \\
&= {{\sigma_0^2}\eta_t} \hat R_K^2(\sqrt{{1}/{\eta_t}}) \le
\frac{1}{\eta_t}.
\eal

Here $\circled{1}$ follows from the fact that
$(1-\eta \hlambda_i )^t \ge \max\set{0,1-t\eta \hlambda_i}$,
and $\circled{2}$ follows from
$\min\set{a,b} \le \sqrt{ab}$ for any nonnegative numbers $a,b$.
Because $t \le T \le \hat T$, we have
$R_K(\sqrt{{1}/{\eta_t}}) \le 1/(\sigma_0 \eta_t)$, so the last inequality holds.

Moreover, we have the upper bounds for $\ltwonorm{\bR}$ and $\fnorm{\bR}$
as follows. First, we have
\bal\label{eq:empirical-loss-bound-E-3}
\ltwonorm{\bR} &\le \max_{i \in [n] }
\pth{1-\pth{1-\eta \hlambda_i }^t}^2 \le \min\set{1,\eta_t^2 \hlambda_i^2}
\le 1.
\eal

We also have
\bal\label{eq:empirical-loss-bound-E-4}
\frac 1n \fnorm{\bR}^2 &=  \frac 1n
\sum\limits_{i=1}^n
\pth{1-\pth{1-\eta \hlambda_i }^t}^4
\le \frac {\eta_t}n \sum\limits_{i=1}^n
\min\set{\frac{1}{\eta_t},\eta_t^{3} \hlambda_i^4} \nonumber \\
&\stackrel{\circled{3}}{\le} \frac {\eta_t}n \sum\limits_{i=1}^n
\min\set{\hlambda_i,\frac{1}{\eta_t}}
=\eta_t\hat R_K^2(\sqrt{{1}/{\eta_t}})\le
\frac {1}{\sigma_0^2 \eta_t}.
\eal
If ${1}/{\eta_t} \le \eta_t^{3} (\hlambda_i)^4$, then
$\min\set{{1}/{\eta_t},\eta_t^{3} (\hlambda_i)^4} = {1}/{\eta_t}$. Otherwise,
we have $\eta_t^{4} \hlambda_i^4 < 1$, so that
$\eta_t \hlambda_i < 1$ and it follows that
$\min\set{{1}/{\eta_t},\eta_t^{3} (\hlambda_i)^4}
\le \eta_t^{3} \hlambda_i^4 \le \hlambda_i$. As a result,
$\circled{3}$ holds.

Combining (\ref{eq:empirical-loss-bound-E-1})-
(\ref{eq:empirical-loss-bound-E-4}), we have
\bals
\Prob{1/n \cdot \tr{\bU \bR \bU^{\top} \bw \bw^{\top} } -
\Expect{}{1/n \cdot \tr{\bU \bR \bU^{\top} \bw \bw^{\top} }} \ge u}
&\le \exp\pth{-c n\min\set{u, u^2\sigma_0^2 \eta_t}}.
\eals
Let $u = 1/\eta_t$ in the above inequality, we have
\bals
\exp\pth{-c n\min\set{u, u^2\sigma_0^2 \eta_t}}
= \exp\pth{-c' n/\eta_t} \le
\exp\pth{-c'n\hat \eps_n^2},
\eals
where $c'  = c\min\set{1,\sigma_0^2}$, and the last inequality is due
to the fact that $1/\eta_t \ge \hat\eps_n^2$ since
$t \le T \le \hat T$.
It follows that with probability at least $1-
\exp\pth{- \Theta(n\hat\eps_n^2)}$,
\bal\label{eq:empirical-loss-bound-E-5}
E_{\eps}\le u+\frac{1}{\eta_t} = \frac{2}{\eta_t}.
\eal
It then follows from (\ref{eq:empirical-loss-bound-seg2}),
(\ref{eq:empirical-loss-bound-E-1})-(\ref{eq:empirical-loss-bound-E-5})
that
\bals
\Expect{P_n}{(f_t-f^*)^2} \le
\Theta\pth{\frac{1}{\eta t}}
\eals
holds with probability at least $1-\exp\pth{-c'n\hat \eps_n^2}$.

\end{proofof}

\subsection{Proof of the Lemmas Required for the Proofs in Section~\ref{sec:proofs-key-results}}
\label{sec:lemmas-main-results}

\begin{lemma}\label{lemma:yt-y-bound}
Suppose $\hell \ge \ell_0$.
Let $t \in [0\relcolon T]$, $\bv = -\pth{\bI-\eta \bK_n }^{t} f^*(\bS)$,
 $\be = -\pth{\bI-\eta \bK_n }^{t} \bw$, and $\eta \in (0,1/(\hell+1))$. Then with probability at least
$1-\exp\pth{-\Theta(n)}$ over the random noise $\bw$,
\bal\label{eq:yt-y-bound}
\ltwonorm{\bv} + \ltwonorm{\be} \le \pth{\Theta(\gamma_0)+\sigma_0+1} \cdot
{\sqrt n}.
\eal
\end{lemma}
\begin{proof}
When $t \in [T]$, we have
\bal\label{eq:yt-y-bound-seg1}
\ltwonorm{\bv}^2 &=\sum\limits_{i=1}^{n}
\pth{1-\eta \hlambda_i }^{2t}
\bth{{\bU}^{\top} f^*(\bS)}_i^2
= \sum\limits_{i=1}^{n}
\pth{1-\eta \hlambda_i }^{2t}
\bth{{\bU}^{\top} f^*(\bS)}_i^2 \nonumber \\
&\le  \sum\limits_{i=1}^{n}
\pth{1-\eta \hlambda_i }^{2t}
\bth{{\bU}^{\top} f^*(\bS)}_i^2  \stackrel{\circled{1}}{\le}
\sum\limits_{i=1}^{n}
\frac{1}{2e\eta \hlambda_i  t}
\bth{{\bU}^{\top} f^*(\bS)}_i^2
\nonumber \\
&\stackrel{\circled{2}}{\le}
\frac{n\gamma_0^2}{ 2e\eta t } \le \Theta(\gamma_0^2 ) \cdot n.
\eal
Here $\circled{1}$ follows from Lemma~\ref{lemma:auxiliary-lemma-1}. $\circled{2}$ follows
from  Lemma~\ref{lemma:bounded-Ut-f-in-RKHS}. This is because with $\hell \ge \ell_0$, $f^* \in \cH_{\Kr}(\gamma_0) \subseteq \cH_{K}(\gamma_0)$. Moreover, it follows from the concentration inequality about quadratic forms of sub-Gaussian random variables in~\cite{quadratic-tail-bound-Wright1973} that
\bals
\Pr\bth{\ltwonorm{\bw}^2 -
\Expect{}{\ltwonorm{\bw}^2} > n}
\le \exp\pth{-\Theta(n)},
\eals
so that $\ltwonorm{\be} \le \ltwonorm{\bw} \le
\sqrt{\Expect{}{\ltwonorm{\bw}^2}}  + {\sqrt n}= \sqrt{n} (\sigma_0+1)$ with probability at least $1-\exp\pth{-\Theta(n)}$.
As a result, (\ref{eq:yt-y-bound}) follows from this inequality and (\ref{eq:yt-y-bound-seg1}) for $t \ge 1$.
When $t = 0$, $\ltwonorm{\bv} \le \Theta(\gamma_0) {\sqrt n}$, so that (\ref{eq:yt-y-bound}) still holds.

\end{proof}

\begin{lemma}
\label{lemma:empirical-loss-convergence-contraction}
Suppose $\hell = \Theta(1)$.
Let $0<\eta<1$, $0 \le t \le T-1$ for $T \ge 1$, and suppose that $\ltwonorm{\hat \by(t') - \by} \le
 c_{\bu}{\sqrt{n}}  $ holds for all $0 \le t' \le t$. Then for every $\delta \in (0,1)$, with probability at least
$1-\delta$ over the random initialization $\bQ$,
\bal\label{eq:empirical-loss-convergence-contraction}
\hat \by(t+1) - \by  &= \pth{\bI- \eta \bK_n }\pth{\hat \by(t) - \by} +\bE(t+1),
\eal
where $\ltwonorm {\bE(t+1)} \le \bE_{m,n,\eta}$,
and $\bE_{m,n,\eta}$ satisfies
\bal\label{eq:empirical-loss-Et-bound}
\bE_{m,n,\eta} \lsim  \eta c_{\bu}  d^{\hell} \sqrt{\frac{\log (2n/{\delta})}{m}}  {\sqrt{n}}.
\eal
\end{lemma}

\begin{proof}
Because $\ltwonorm{\hat \by(t') - \by} \le {\sqrt{n}} c_{\bu}$ holds for all $t' \in [0,t]$.
We have
\bal\label{eq:empirical-loss-convergence-contraction-seg1}
&\hat \by(t+1) - \hat \by(t) = \frac{1}{\sqrt m} \sum_{r=1}^m \pth{a_r(t+1)-a_r(t)}
\sigma_{\btau}(\bbx_i,\bbq_r) \nonumber \\
&=-\frac{\eta}{n}
\hbK
 (\hat \by(t) -  \by)  \nonumber \\
&=- \frac{\eta}{n} \bK  \pth{\hat \by(t) - \by}
+ \underbrace{\frac{\eta}{n} \pth{\bK-\hbK}  \pth{\hat \by(t) - \by}}_{\defeq \bE(t+1)}.
\eal%
Since $\hell = \Theta(1)$, it follows from (\ref{eq:Km-close-to-K-spectralnorm}) of
Theorem~\ref{theorem:Km-close-to-K-supnorm} that
with probability at least
$1-\delta$ over $\bQ$,
$\ltwonorm{\hbK_n - \bK_n} \lsim d^{\hell} \sqrt{\frac{\log (2n/{\delta})}{m}}$.
As a result,
$\ltwonorm{\bE(t+1)}$ can be bounded by
\bal\label{eq:empirical-loss-convergence-contraction-E-bound}
\ltwonorm{\bE(t+1)} \lsim \eta c_{\bu} \cdot d^{\hell} \sqrt{\frac{\log (2n/{\delta})}{m}} \cdot {\sqrt{n}}.
\eal
(\ref{eq:empirical-loss-convergence-contraction-seg1})
can be rewritten as
\bals
\hat \by(t+1) - \by
&=\pth{\bI-\frac{\eta}{n} \bK}\pth{\hat \by(t) - \by} + \bE(t+1),
\eals
which proves (\ref{eq:empirical-loss-convergence-contraction})
with the upper bound for $\ltwonorm {\bE(t+1)} $ in
(\ref{eq:empirical-loss-convergence-contraction-E-bound}).

\end{proof}

\begin{lemma}\label{lemma:bounded-Linfty-vt-sum-et}
Suppose $\hell = \Theta(1) \ge \ell_0$.
Let $h_t(\cdot) = \sum_{t'=0}^{t-1} h(\cdot,t')$ for $t \in [T]$, $T \le
\hat T$ where
\bals
h(\cdot,t') &= v(\cdot,t') + \hat e(\cdot,t'), \\
v(\cdot,t')  &= \frac{\eta}{n} \sum_{j=1}^n
K(\cdot, \bbx_j)  \bth{ \bv(t')}_j  , \\
\hat e(\cdot,t') &= \frac{\eta}{n}
\sum\limits_{j=1}^n  K(\cdot, \bbx_j)\bth{ \be(t')}_j,
\eals
where $\bv(t') \in \cV_{t'}$,
$\be(t') \in \cE_{t',\tau}$ for all $0 \le t' \le t-1$.
Suppose that $\tau \lsim 1/(\eta T)$,
then with probability at least $1 - \exp\pth{- \Theta(n\hat\eps_n^2)}$
over the random noise $\bw$,
\bal\label{eq:bounded-h}
\norm{h_t}{\cH_K} \le B_h = \gamma_0 +\sqrt{2} + 1,
\eal
and $B_h$ is also defined in (\ref{eq:B_h}).
\end{lemma}
\begin{proof}
We have $\bv(t) = -\pth{\bI- \eta \bK_n  }^t f^*(\bS)$,
$\be(t) = \bbe_1(t) + \bbe_2(t)$ with
$\bbe_1(t) = -\pth{\bI-\eta\bK_n }^t \bw$,
$\ltwonorm{\bbe_2(t)} \le {\sqrt n} \tau$.
We define
\bal\label{eq:bounded-Linfty-vt-sum-et-hat-e1-hat-e2}
\hat e_1(\cdot,t') \defeq- \frac{\eta}{n}
\sum\limits_{j=1}^n  K(\bbx_j,\bx) \bth{ \bbe_1(t')}_j,
\quad
\hat e_2(\cdot,t') \defeq- \frac{\eta}{n}
\sum\limits_{j=1}^n  K(\bbx_j,\bx) \bth{\bbe_2(t')}_j.
\eal
Let $\bSigma$ be the diagonal matrix
containing eigenvalues of $\bK_n$, which are
$\hlambda_1 \ge \hlambda_2 \ldots \ge \hlambda_r \ge \hlambda_{r+1} = \ldots \hlambda_n = 0$ where $r \le n$ is the rank of the gram matrix $\bK_n$.
Then we have
\bal\label{eq:bounded-Linfty-vt-sum-seg1}
\sum_{t'=0}^{t-1} v(\bx,t') &=\frac{\eta}{n} \sum\limits_{j=1}^n  \sum_{t'=0}^{t-1}
\bth{ \pth{\bI- \eta \bK_n }^{t'} f^*(\bS)}_j K(\bbx_j,\bx) \nonumber \\
&=\frac{\eta}{n} \sum\limits_{j=1}^n \sum_{t'=0}^{t-1}
\bth{ \bU \pth{\bI-\eta\bSigma }^{t'} {\bU}^{\top} f^*(\bS)}_j K(\bbx_j,\bx).
\eal
It follows from (\ref{eq:bounded-Linfty-vt-sum-seg1}) that
\bal\label{eq:bounded-Linfty-vt-sum-seg2}
\norm{\sum_{t'=0}^{t-1} v(\cdot,t')}{\cH_K}^2
&= \frac{\eta^2}{n^2} f^*(\bS)^{\top}
\bU \sum_{t'=0}^{t-1} \pth{\bI-\eta \bSigma}^{t'} {\bU}^{\top}
 \bK  \bU \sum_{t'=0}^{t-1} \pth{\bI-\eta \bSigma}^{t'}
{\bU}^{\top} f^*(\bS) \nonumber \\
&= \frac 1n \ltwonorm{\eta\pth{\bK_n}^{1/2}  \bU \sum_{t'=0}^{t-1} \pth{\bI-\eta \bSigma}^{t'} {\bU}^{\top} f^*(\bS)}^2 \nonumber \\
&\le \frac 1n \sum\limits_{i=1}^{r} \frac{\pth{1-
\pth{1-\eta \hlambda_i }^t}^2}
{\hlambda_i}\bth{{\bU}^{\top} f^*(\bS)}_i^2
\le \frac 1n \sum\limits_{i=1}^{r} \frac{\bth{{\bU}^{\top} f^*(\bS)}_i^2}
{\hlambda_i}
\le \gamma_0^2,
\eal
where the last inequality
 follows from Lemma~\ref{lemma:bounded-Ut-f-in-RKHS}.

Similarly, we have
\bal\label{eq:bounded-Linfty-hat-et-sum-1}
&\norm{\sum_{t'=0}^{t-1} \hat e_1(\cdot,t')}{\cH_K}^2
\le \frac 1n \sum\limits_{i=1}^{r} \frac{\pth{1-
\pth{1-\eta \hlambda_i }^t}^2}
{\hlambda_i}\bth{{\bU}^{\top} \bw}_i^2,
\eal
It then follows from the argument in the proof of \cite[Lemma 9]{RaskuttiWY14-early-stopping-kernel-regression}
that the RHS of (\ref{eq:bounded-Linfty-hat-et-sum-1}) is bounded w.h.p. We define a diagonal matrix $\bR \in \RR^{n \times n}$
with $\bR_{ii} = \big(1-(1-\eta \hlambda_i )^t\big)^2/\hlambda_i$ for $i \in [n]$. Then the RHS of (\ref{eq:bounded-Linfty-hat-et-sum-1}) is
$1/n \cdot \tr{\bU \bR \bU^{\top} \bw \bw^{\top} }$. It follows from \cite{quadratic-tail-bound-Wright1973}
that
\bal\label{eq:bounded-Linfty-hat-et-sum-2}
&\Prob{1/n \cdot \tr{\bU \bR \bU^{\top} \bw \bw^{\top} } -
\Expect{}{1/n \cdot \tr{\bU \bR \bU^{\top} \bw \bw^{\top} }} \ge u}
\nonumber \\
&\le \exp\pth{-c \min\set{nu/\ltwonorm{\bR},n^2u^2/\fnorm{\bR}^2}}
\eal
for all $u > 0$, and $c$ is a  positive constant. Let $\eta_t = \eta t$ for all $t \ge 0$, we have
\bal\label{eq:bounded-Linfty-hat-et-sum-3}
\Expect{}{1/n \cdot \tr{\bU \bR \bU^{\top} \bw \bw^{\top} }}
&\le \frac {\sigma_0^2}n \sum\limits_{i=1}^r
\frac{\pth{1-\pth{1-\eta \hlambda_i }^t}^2}{\hlambda_i}
\stackrel{\circled{1}}{\le}
\frac {\sigma_0^2}n \sum\limits_{i=1}^r
\min\set{\frac{1}{\hlambda_i},\eta_t^2 \hlambda_i}
\nonumber \\
&\le
\frac {{\sigma_0^2}\eta_t}n \sum\limits_{i=1}^r
\min\set{\frac{1}{\eta_t\hlambda_i},\eta_t \hlambda_i}
\stackrel{\circled{2}}{\le}
\frac {{\sigma_0^2}\eta_t}n \sum\limits_{i=1}^r
\min\set{1,\eta_t \hlambda_i} \nonumber \\
&= \frac {{\sigma_0^2}\eta_t^2}n \sum\limits_{i=1}^r
\min\set{\eta_t^{-1},\hlambda_i}
= {{\sigma_0^2}\eta_t^2} \hat R_K^2(\sqrt{{1}/{\eta_t}}) \le
1.
\eal
Here $\circled{1}$ follows from the fact that
$(1-\eta \hlambda_i )^t \ge \max\set{0,1-t\eta \hlambda_i}$,
and $\circled{2}$ follows from
$\min\set{a,b} \le \sqrt{ab}$ for any nonnegative numbers $a,b$.
Because $t \le T \le \hat T$, we have
$\hat R_K(\sqrt{{1}/{\eta_t}}) \le 1/(\sigma_0 \eta_t)$, so the last inequality holds.

Moreover, we have the upper bounds for $\ltwonorm{\bR}$ and $\fnorm{\bR}$
as follows. First, we have
\bal\label{eq:bounded-Linfty-hat-et-sum-4}
\ltwonorm{\bR} &\le \max_{i \in [r] }\frac{\pth{1-\pth{1-\eta \hlambda_i }^t}^2}{\hlambda_i} \le \max_{i \in [r] } \min\set{\frac{1}{\hlambda_i},\eta_t^2 \hlambda_i}
\le \eta_t.
\eal

We also have
\bal\label{eq:bounded-Linfty-hat-et-sum-5}
\frac 1n \fnorm{\bR}^2 &=  \frac 1n
\sum\limits_{i=1}^r
\frac{\pth{1-\pth{1-\eta \hlambda_i }^t}^4}{(\hlambda_i)^2}
\le \frac {\eta_t^3}n \sum\limits_{i=1}^r
\min\set{\frac{1}{\eta_t^3 \hlambda_i^2},\eta_t \hlambda_i^2} \nonumber \\
&\stackrel{\circled{3}}{\le} \frac {\eta_t^3}n \sum\limits_{i=1}^r
\min\set{\hlambda_i,\frac{1}{\eta_t}}
=\eta_t^3 \hat R_K^2(\sqrt{{1}/{\eta_t}})\le
\frac {\eta_t}{\sigma_0^2},
\eal
where $\circled{3}$ follows from
\bals
\min\set{\frac{1}{\eta_t^3 \hlambda_i^2},\eta_t \hlambda_i^2}
= \hlambda_i
\min\set{\frac{1}{\eta_t^3 \hlambda_i^3},\eta_t \hlambda_i}
\le  \hlambda_i.
\eals
Combining (\ref{eq:bounded-Linfty-hat-et-sum-1})-(\ref{eq:bounded-Linfty-hat-et-sum-5}) with $u=1$ in
(\ref{eq:bounded-Linfty-hat-et-sum-2}), we have
\bals
&\Prob{1/n \cdot \tr{\bU \bR \bU^{\top} \bw \bw^{\top} } -
\Expect{}{1/n \cdot \tr{\bU \bR \bU^{\top} \bw \bw^{\top} }} \ge 1} \nonumber \\
&\le \exp\pth{-c \min\set{n/\eta_t,n \sigma_0^2/\eta_t}}
\le \exp\pth{-nc'/\eta_t} \le
\exp\pth{-c'n\hat\eps_n^2},
\eals
where $c'  = c\min\set{1,\sigma_0^2}$, and the last inequality is due
to the fact that $1/\eta_t \ge \hat\eps_n^2$ since
$t \le T \le \hat T$.
It follows that with probability at least $1-
\exp\pth{- \Theta(n\hat\eps_n^2)}$,
$\norm{\sum_{t'=0}^{t-1} \hat e_1(\cdot,t')}{\cH_{K}}^2 \le 2$.

We now find the upper bound for $\norm{\sum_{t'=0}^{t-1} \hat e_2(\cdot,t')}{\cH_K}$. We have
\bals
\norm{\hat e_2(\cdot,t')}{\cH_K}^2
&\le \frac{\eta^2}{n^2} \bbe_2^{\top}(t')\bK\bbe_2(t')
\le \eta^2 \hlambda_1 \tau^2,
\eals
so that
\bal\label{eq:bounded-Linfty-hat-et-sum-6}
&\norm{\sum_{t'=0}^{t-1} \hat e_2(\cdot,t')}{\cH_K}
\le \sum_{t'=0}^{t-1} \norm{\hat e_2(\cdot,t')}{\cH_K}
\le  T \eta \sqrt{\hlambda_1} \tau \le 1,
\eal
if $\tau \lsim 1/(\eta T) $ since $\hlambda_1 \in (0, \Theta(1))$ due to the fact that $\hlambda_1 \le \sup_{\bx \in \cX}K(\bx,\bx) = \Theta(1)$.

Finally, it follows from (\ref{eq:bounded-Linfty-vt-sum-seg2}),
(\ref{eq:bounded-Linfty-hat-et-sum-2}), and (\ref{eq:bounded-Linfty-hat-et-sum-6})
that
\bals
\norm{h_t}{\cH_K}  &\le \norm{\sum_{t'=0}^{t-1} \hat v(\cdot,t')}{\cH_K}
+\norm{\sum_{t'=0}^{t-1} \hat e_1(\cdot,t')}{\cH_K} + \norm{\sum_{t'=0}^{t-1} \hat e_2(\cdot,t')}{\cH_K} \le \gamma_0 + \sqrt{2} + 1  = B_h.
\eals
\end{proof}

\begin{theorem}
\label{theorem:Km-close-to-K-supnorm}
Suppose $\hell = \Theta(1)$.
For any fixed $\bx' \in \cX$ and every $\delta \in (0,1)$, with probability at least
$1-\delta$ over the random initialization $\bQ = \set{\bbq_r}_{r=1}^m$, we have
\bal\label{eq:Km-close-to-K-supnorm}
\sup_{\bx \in \cX}\abth{\hK(\bx,\bx') - K(\bx,\bx')} \lsim d^{\hell} \sqrt{\frac{\log 2/{\delta}}{m}}.
\eal
As a result, with probability at least
$1-\delta$ over $\bQ$,
\bal
\sup_{\bx \in \cX,i \in [n]}\abth{\hK(\bx,\bbx_i) - K(\bx,\bbx_i)} \lsim d^{\hell}
 \sqrt{\frac{\log (2n/{\delta})}{m}}, \label{eq:Km-close-to-K-S} \\
\ltwonorm{\hbK_n - \bK_n} \lsim d^{\hell} \sqrt{\frac{\log (2n/{\delta})}{m}}
\label{eq:Km-close-to-K-spectralnorm}.
\eal
\end{theorem}
\begin{proof}
First, it follows from (\ref{eq:concentration-RKHS-Hilbert-space-sigma-bound}) in the proof of Lemma~\ref{lemma:concentration-RKHS-Hilbert-space-sigma}
that for all $\bx,\bx' \in \cX$,
\bals
\abth{\sigma_{\btau}(\bx,\bx')} \le \sup_{\bq \in \cX} \norm{\sigma_{\btau}(\cdot,\bq)}{\cH_{\sigma}}^2 =
\sum\limits_{\ell=0}^{\hell} \mu^{\frac 12}_{\sigma,\ell} N(d,\ell)
= \sum\limits_{\ell=0}^{\hell} N^{\frac 12}(d,\ell) = \Theta(d^{\hell/2})  \defeq p_0,
\eals
which follows from the fact that $N^{\frac 12}(d,\ell) \asymp d^{\frac{\hell}{2}}$ for every $\ell \in [0\relcolon \hell]$ with $\hell = \Theta(1)$.
The following arguments hold for every given $\bx' \in \cX$.
We have
\bals
\Expect{\bbw}{ \sigma_{\btau}(\cdot,\bbw) \sigma_{\btau}(\bbw,\bx')} = K(\cdot,\bx').
\eals
It  then follows from (\ref{eq:concentration-RKHS-Hilbert-space-sigma})
of Lemma~\ref{lemma:concentration-RKHS-Hilbert-space-sigma}
that
for every $t > 0$,
\bal\label{eq:Km-close-to-K-supnorm-seg2}
\Prob{ \norm{\frac{1}{m}\sum\limits_{r=1}^m \sigma_{\btau}(\cdot,\bbq_r) \sigma_{\btau}(\bbq_r,\bx')
-K(\cdot,\bx') }{\cH_{\sigma}} < t}
\ge 1- 2\exp\pth{-\frac{mt^2}{\Theta(d^{3\hell/2})}}.
\eal
Noting that
$1/m \cdot \sum\limits_{r=1}^m \sigma_{\btau}(\cdot,\bbq_r) \sigma_{\btau}(\bbq_r,\bx')=\hK(\cdot,\bx')$,
it then follows from
(\ref{eq:Km-close-to-K-supnorm-seg2}) that
\bal\label{eq:Km-close-to-K-supnorm-seg3}
\Prob{ \norm{\hK(\cdot,\bx')
-K(\cdot,\bx') }{\cH_{\sigma}} < t}
\ge 1- 2\exp\pth{-\frac{mt^2}{\Theta(d^{3\hell/2})}}.
\eal
(\ref{eq:Km-close-to-K-supnorm}) then follows from
(\ref{eq:Km-close-to-K-supnorm-seg3})
 and the fact that
\bals
\sup_{\bx \in \cX}\abth{\hK(\bx,\bx') - K(\bx,\bx')} \le \norm{\hK(\cdot,\bx')
-K(\cdot,\bx') }{\cH_{\sigma}} \cdot \sup_{\bx \in \cX}\norm{\sigma_{\btau}(\cdot,\bx)}{\cH_{\sigma}},
\eals
and (\ref{eq:Km-close-to-K-S}) and (\ref{eq:Km-close-to-K-spectralnorm}) follow from
(\ref{eq:Km-close-to-K-supnorm}) by the union bound.
\end{proof}

\begin{lemma}
\label{lemma:concentration-RKHS-Hilbert-space-sigma}
Suppose $\hell = \Theta(1)$, and $p$ is a function defined on $\cX$ and
$\sup_{\bx \in \cX} \abth{p(\bx)} \le p_0$ for a positive number $p_0$. Then for every $r > 0$,
\bal\label{eq:concentration-RKHS-Hilbert-space-sigma}
\Prob{\norm{\frac{1}{m} \sum_{r=1}^m \sigma_{\btau}(\cdot,\bbq_r)p(\bbq_r)-
\Expect{\bbw}{\sigma_{\btau}(\cdot,\bbw)p(\bbw)}}{\cH_{\sigma}}> r}
\le 2\exp\pth{-\frac{mr^2}{\Theta(d^{\hell/2})p_0^2}}.
\eal
\end{lemma}
\begin{proof}
Let $\cB = \cH_K \subseteq L^2(\unitsphere{d-1}, \mu)$, then
$\cB \in D(1,1)$~\cite{Pinelis1992}.
We then construct the martingale $\set{f_k}_{k \in [m]}$.
First, for every $\bq \in \cX$, we have
\bal\label{eq:concentration-RKHS-Hilbert-space-sigma-bound}
\norm{\sigma_{\btau}(\cdot,\bq)}{\cH_{\sigma}}^2 =
\sigma_{\btau}(\bq,\bq) = \sum\limits_{\ell=0}^{\hell} \mu^{-\frac 12}_{\sigma,\ell}
P^{(d)}_{\ell}(1) = \sum\limits_{\ell=0}^{\hell}  N^{\frac 12}(d,\ell)
= \Theta(d^{\hell/2}).
\eal
We define $p_1 \defeq 2p_0 \norm{\sigma_{\btau}(\cdot,\bq)}{\cH_{\sigma}} = \Theta(d^{\hell/4})p_0$
for every $\bq \in \cX$. For each $k \in [m]$, we also define
\bals
f_k \defeq \Expect{}{\frac{1}{{p_1 \sqrt {m}}} \sum\limits_{r=1}^m \pth{
\sigma_{\btau}(\cdot,\bbq_r)p(\bbq_r)
-\Expect{\bbw}{\sigma_{\btau}(\cdot,\bbw)p(\bbw)}}\longmid
\cF_k}, \forall k \in [m],
\eals
where $\set{\cF_k}_{k=0}^m$ is an increasing sequence of $\sigma$-algebras, $\cF_k$ is the $\sigma$-algebra generated by $\set{\bbq_r}_{r=1}^k$, and $\cF_0$ is the trivial $\sigma$-algebra so that
$f_0 = 0$. We note that
\bals
f_m &= \frac{1}{{p_1\sqrt {m}}} \sum_{r=1}^m \pth{\sigma_{\btau}(\cdot,\bbq_r)p(\bbq_r)
-\Expect{\bbw}{K(\cdot,\bbw)p(\bbw)}}, \nonumber \\
d_k &= f_k - f_{k-1}
= \frac{1}{{p_1\sqrt {m}}} \pth{\sigma_{\btau}(\cdot,\bbq_k)p(\bbq_k)- \Expect{\bbw}{\sigma_{\btau}(\cdot,\bbw)p(\bbw)}}, \forall k \in [m],
\eals
and
$f^* = \max_{k \in [m]} \norm{f_k}{}$. For every $k \in [m]$,
we have
\bal\label{eq:concentration-RKHS-Hilbert-space-sigma-seg1}
\norm{d_k}{\cH_K}
&= \norm{\frac{1}{{p_1\sqrt {m}}} \pth{\sigma_{\btau}(\cdot,\bbq_k)p(\bbq_k)
- \Expect{\bbw}{\sigma_{\btau}(\cdot,\bbw)p(\bbw)}}}{\cH_{\sigma}} \nonumber \\
&\stackrel{\circled{1}}{\le} \frac{1}{p_1\sqrt {m}}
\pth{p_0 \norm{\sigma_{\btau}(\cdot,\bbq_k)}{\cH_{\sigma}} + p_0
\Expect{\bbw}{\norm{\sigma_{\btau}(\cdot,\bbw)}{\cH_{\sigma}}}}
\stackrel{\circled{2}}{\le} \frac 1{\sqrt m},
\eal
where $\circled{1}$ follows from the triangle inequality and
the Jensen's inequality, and $\circled{2}$ follows from
(\ref{eq:concentration-RKHS-Hilbert-space-sigma-bound}).

It follows from (\ref{eq:concentration-RKHS-Hilbert-space-sigma-seg1})
that $\sum_{k=1}^{\infty} \norm{d_k}{}^2 \le 1$.
Applying
Lemma~\ref{lemma:concentration-Hilbert-space} with
the martingale $\set{f_k}_{k=0}^m$ and
$\cB = \cH_{\sigma} \subseteq L^2(\unitsphere{d-1}, \mu)$ with $B = 1$,
we have
$\Prob{f^* =\max_{k \in [m]} \norm{f_k}{}> r} \le
2\exp\pth{-\frac{r^2}{2}}$. As a result, for every $r > 0$,
\bals
\Prob{\norm{\frac{1}{{p_1\sqrt {m}}} \sum_{r=1}^m \pth{\sigma_{\btau}(\cdot,\bbq_r)p(\bbq_r)
-\Expect{\bbw}{\sigma_{\btau}(\cdot,\bbw)p(\bbw)}}}{\cH_{\sigma}}> r} \le
2\exp\pth{-\frac{r^2}{2}},
\eals
and it follows that
\bals
\Prob{\norm{\frac{1}{m} \sum_{r=1}^m \sigma_{\btau}(\cdot,\bbq_r)p(\bbq_r)-
\Expect{\bbw}{\sigma_{\btau}(\cdot,\bbw)p(\bbw)}}{\cH_{\sigma}}> r}
\le 2\exp\pth{-\frac{mr^2}{\Theta(d^{\hell/2})p_0^2}},
\eals
which completes the proof of
(\ref{eq:concentration-RKHS-Hilbert-space-sigma}).

\end{proof}

In order to prove Lemma~\ref{lemma:concentration-RKHS-Hilbert-space-sigma}, we need to the following
concentration inequality for independent random variables taking values in a Hilbert space $\cB$ of functions
defined on a measurable space $(S,\Sigma_S,\mu_S)$. Let $\set{f_k}_{n=0}^{\infty}$ be a martingale
over a separable Banach space $\pth{\cB,\norm{\cdot}{}}$ with respect to an increasing sequence of $\sigma$-algebras $\set{\cF_k}_{n=0}^{\infty}$ and
$f_0 = 0$. Define $d_k \defeq f_k-f_{k-1}$ for $k \ge 1$, $d_0 = 0$,
and $f^* \defeq \sup_{n \ge 0} \norm{f_k}{}$. The following lemma is about the martingale based concentration inequality for Banach space-valued random process~\cite{Pinelis1992}.
\begin{lemma}
[{\cite[Theorem 2]{Pinelis1992}}]
\label{lemma:concentration-Hilbert-space}
Suppose that $\sum_{k=1}^{\infty}
\textup{esssup} \norm{d_k}{}^2 \le 1$ where
$\textup{esssup}(f) = \inf_{a \in \RR} \set{\mu(f^{-1}(a,+\infty)) = 0}$
for a function denotes the essential supremum of a function, and
$\cB \in D(A_1,A_2)$ or $\cB \subseteq L^p(S,\Sigma,\mu)$ with
$p \ge 2$. Then for every $r > 0$,
\bal\label{eq:concentration-Hilbert-space}
\Prob{f^* > r} \le 2\exp\pth{-\frac{r^2}{2B}}
\eal
with $B = p-1$ for $\cB \subseteq L^p(S,\Sigma_S,\mu_S)$.
\end{lemma}

\begin{lemma}
\label{lemma:spectrum-K}
The integral operator $T_K \colon L^2(\cX,\mu) \to L^2(\cX,\mu), \pth{T_K f}(\bx) \defeq \int_{\cX} K(\bx,\bx') f(\bx') \diff \mu(\bx')$ is
a positive, self-adjoint, and compact operator on $L^2(\cX,\mu)$. $\set{Y_{\ell j}}_{j \in [N(d,\ell)]}$ are the eigenfunction of $T_K$ with
$\mu_{\ell} = {\mu_{\sigma,\ell}}$ being the corresponding eigenvalue for every $\ell \in [0:\hell]$. Furthermore,
\bal\label{eq:K-mercer-decomp}
K(\bx,\bx') = \sum\limits_{\ell=0}^{\hell} \sum\limits_{j=1}^{N(d,\ell)}\mu_{\ell} Y_{\ell,j}(\bx)Y_{\ell,j}(\bx'), \quad \bx,\bx' \in \cX,
\eal
and $\sup_{\bx,\bx' \in \cX} \abth{K(\bx,\bx')} = \hell+1 = \Theta(1)$.
\end{lemma}
\begin{proof}
It follows from the definition of the activation function $\sigma$ in (\ref{eq:sigma-activation-def})
and the definition of $K$ in (\ref{eq:K-def}) that
\bal\label{eq:spectrum-K-seg1}
K (\bx,\bx') &=  \int_{\cX} \sigma_{\btau}(\bx,\bw) \sigma_{\btau}(\bw,\bx')  \diff \mu(\bw) \nonumber \\
&= \int_{\cX}
\pth{\sum\limits_{\ell=0}^{\hell} \sum\limits_{j=1}^{N(d,\ell)} \mu^{\frac 12}_{\sigma,\ell}Y_{\ell,j}(\bx)Y_{\ell,j}(\bw) }
\cdot
\pth{\sum\limits_{\ell=0}^{\hell} \sum\limits_{j=1}^{N(d,\ell)} \mu^{\frac 12}_{\sigma,\ell}Y_{\ell,j}(\bw)Y_{\ell,j}(\bx')} \diff \mu(\bw) \nonumber \\
&= \sum\limits_{\ell=0}^{\hell} \sum\limits_{j=1}^{N(d,\ell)} \mu_{\ell} Y_{\ell,j}(\bx)Y_{\ell,j}(\bx'),
\eal
where the last inequality follows from the orthogonality of the
orthogonal set $\set{Y_{\ell j}}_{\ell \in [0:\hell], j \in [N(d,\ell)]}$.

It follows from (\ref{eq:spectrum-K-seg1}) that $K$ is PD kernel of finite rank over the compact set $\cX$, so that $T_K$ is a positive, self-adjoint, and compact operator on $L^2(\cX,\mu)$. Furthermore,
for every $\ell \in [0:\hell]$ and every $j \in [N(d,\ell)]$, $T_K Y_{\ell,j} = \mu_{\ell} Y_{\ell,j}$, showing that $\mu_{\ell}$ is the eigenvalue for every function
in $\set{Y_{\ell j}}_{\ell \in [0:\hell], j \in [N(d,\ell)]}$.

\vspace{.1in}
Finally, considering the RKHS associated with the PD kernel $K$, we have
\bals
\sup_{\bx,\bx' \in \cX} \abth{K(\bx,\bx')} &= \sup_{\bx,\bx' \in \cX} \abth{ \iprod{K(\cdot,\bx)}{K(\cdot,\bx')}_{\cH_K} }
\le  \sup_{\bx \in \cX} K(\bx,\bx) \nonumber \\
&=  \sum\limits_{\ell=0}^{\hell} \mu_{\ell}
N(d,\ell) P^{(d)}_{\ell}(1) = \hell+1 = \Theta(1),
\eals
which is due to the fact that $P^{(d)}_k(1) = 1$ for all $k \ge 0$ discussed in Section~\ref{sec:harmonic-analysis-detail} of this appendix.
\end{proof}

\begin{lemma}[In the proof of
{\cite[Lemma 8]{RaskuttiWY14-early-stopping-kernel-regression}}]
\label{lemma:bounded-Ut-f-in-RKHS}
Let $r$ be the rank of the gram matrix $\bK$ for the kernel $K$ over the training features $\bS$. Then
for any $f \in \cH_{K}(\gamma_0)$, we have
\bal\label{eq:b.ounded-Ut-f-in-RKHS}
\frac 1n \sum_{i=1}^r \frac{\bth{{\bU}^{\top}f(\bS)}_i^2}{\hlambda_i} \le \gamma_0^2.
\eal

\end{lemma}

\begin{lemma}
\label{lemma:auxiliary-lemma-1}
For any positive real number $a \in (0,1)$ and natural number $t$,
we have
\bal\label{eq:auxiliary-lemma-1}
(1-a)^t \le e^{-ta} \le \frac{1}{eta}.
\eal
\end{lemma}
\begin{proof}
The result follows from the facts that
$\log(1-a) \le a$ for $a \in (0,1)$ and $\sup_{u \in \RR}
ue^{-u} \le 1/e$.
\end{proof}

\begin{lemma}
[{\cite[Lemma B.7]{yang2024gradientdescentfindsoverparameterized}}]
\label{lemma:hat-eps-eps-relation}
With probability at least $1-4\exp(-\Theta(n\eps_n^2))$,
\bal
\eps_n^2 \lsim\hat \eps_n^2, \quad \hat \eps_n^2 \lsim \eps_n^2.
\label{eq:bound-eps-n-hat-eps-n}
\eal
\end{lemma}

\begin{proofof}{Lemma~\ref{lemma:LRC-population-NN}}
We first decompose the
Rademacher complexity of the function class
$\{f \in \cF(B,w) \colon \Expect{P}{f^2} \le r\}$ into two terms as follows:
\bal\label{eq:lemma-LRC-population-NN-decomp}
&\cfrakR \pth{\set{f \colon f \in \cF(B,w) , \Expect{P}{f^2} \le r}} \nonumber \\
&\le \underbrace{\frac 1n \Expect{}
{\sup_{f \in \cF(B,w) \colon \Expect{P}{f^2} \le r }
{ \sum\limits_{i=1}^n {\sigma_i}{h(\bbx_i)}}
}}_{\defeq \cR_1} +
\underbrace{\frac 1n \Expect{}
{\sup_{f \in \cF(B,w) \colon \Expect{P}{f^2} \le r } {{ \sum\limits_{i=1}^n {\sigma_i}{e(\bbx_i)}}
} }}_{\defeq \cR_2}.
\eal
We now analyze the upper bounds for $\cR_1, \cR_2$ on the RHS
of (\ref{eq:lemma-LRC-population-NN-decomp}).

\textbf{Derivation for the upper bound for $\cR_1$.}

According to Definition~\ref{eq:func-class-B-w} and Theorem~\ref{theorem:bounded-NN-class}, for any $f \in  \cF(B_,w)$,
we have $f = h + e$ with $h \in \cH_K(B)$,
$e \in L^{\infty}$, $\supnorm{e} \le w$.

When $\Expect{P}{f^2} \le r$, it follows from the triangle inequality
that $\norm{h}{L^2} \le \norm{f}{L^2} + \norm{e}{L^2} \le {\sqrt r} + w \defeq r_h$.
We now consider $h \in \cH_{K}(B)$ with $\norm{h}{L^2} \le r_h$ in  the remaining of this proof. We have
\bal\label{eq:lemma-LRC-population-NN-seg1}
\sum\limits_{i=1}^n {\sigma_i}{f(\bbx_i)} &=
\sum\limits_{i=1}^n {\sigma_i}\pth{h(\bbx_i) + e(\bbx_i)}
\nonumber \\
&=
\iprod{h}
{\sum\limits_{i=1}^n {\sigma_i}{K(\cdot,\bbx_i)}}_{\cH_K} +
\sum\limits_{i=1}^n {\sigma_i}e(\bbx_i).
\eal
Because $\set{v_q}_{q \ge 1}$ is an orthonormal basis of $\cH_K$, for any $0 \le Q \le n$, we further express the first term
on the RHS of (\ref{eq:lemma-LRC-population-NN-seg1}) as
\bal\label{eq:lemma-LRC-population-NN-seg2}
\iprod{h}
{\sum\limits_{i=1}^n {\sigma_i}{K(\cdot,\bbx_i)}}_{\cH_K}
&=\iprod{\sum\limits_{q=1}^Q \sqrt{\lambda_q}
\iprod{h}{v_q}_{\cH_K} v_q }
{\sum\limits_{q=1}^Q \frac{1}{\sqrt{\lambda_q}}
\iprod{\sum\limits_{i=1}^n{\sigma_i}{K(\cdot,\bbx_i)}}{v_q}_{\cH_K}v_q}_{\cH_K}
\nonumber \\
&\phantom{=}+\iprod{h}
{\sum\limits_{q > Q} \iprod{\sum\limits_{i=1}^n {\sigma_i}{K(\cdot,\bbx_i)}}{v_q}_{\cH_K}v_q}_{\cH_K}.
\eal
Due to the fact that $h \in \cH_K$,
$h = \sum\limits_{q =1}^{\infty}  \bbeta^{(h)}_q v_q
=\sum\limits_{q =1}^{\infty}  \sqrt{\lambda_q} \bbeta^{(h)}_q e_q $
with $v_q = \sqrt{\lambda_q} e_q$. Therefore,
$\norm{h}{L^2}^2 = \sum\limits_{q=1}^{\infty} \lambda_q {\bbeta^{(h)}_q}^2$, and
\bal\label{eq:lemma-LRC-population-NN-seg3}
\norm{\sum\limits_{q=1}^Q \sqrt{\lambda_q}\iprod{h}{v_q}_{\cH_K}v_q}{\cH_K}
&= \norm{\sum\limits_{q=1}^Q \sqrt{\lambda_q} \bbeta^{(h)}_q v_q}{{\cH_K}}
= \sqrt{\sum\limits_{q=1}^Q \lambda_q  {\bbeta^{(h)}_q}^2}
\le
\norm{h}{L^2} \le r_h.
\eal
According to Mercer's Theorem, because the kernel $K$ is continuous, symmetric and positive definite, it has the decomposition
\bals
K(\cdot,\bbx_i) = \sum\limits_{j=1}^{\infty} \lambda_j e_j(\cdot)
e_j(\bbx_i),
\eals
so that we have
 \bal\label{eq:lemma-LRC-population-NN-seg4}
\iprod{\sum\limits_{i=1}^n{\sigma_i}{K(\cdot,\bbx_i)}}{v_q}_{\cH_K}
&=\iprod{\sum\limits_{i=1}^n{\sigma_i} \sum\limits_{j=1}^{\infty} \lambda_j e_j e_j(\bbx_i) }{v_q}_{\cH_K} =\iprod{\sum\limits_{i=1}^n{\sigma_i} \sum\limits_{j=1}^{\infty}
\sqrt{\lambda_j}e_j(\bbx_i) \cdot v_j  }{v_q}_{\cH_K} \nonumber \\
&=\sum\limits_{i=1}^n{\sigma_i} \sqrt{\lambda_q}e_q(\bbx_i).
\eal
Combining (\ref{eq:lemma-LRC-population-NN-seg2}),
(\ref{eq:lemma-LRC-population-NN-seg3}), and
(\ref{eq:lemma-LRC-population-NN-seg4}), we have
\bal\label{eq:lemma-LRC-population-NN-seg5}
\iprod{h}
{\sum\limits_{i=1}^n {\sigma_i}{K(\cdot,\bbx_i)}}
&\stackrel{\circled{1}}{\le}
\norm{\sum\limits_{q=1}^Q \sqrt{\lambda_q}\iprod{h}{v_q}_{\cH_K}v_q}{\cH_K}
\cdot
\norm{\sum\limits_{q=1}^Q \frac{1}{\sqrt{\lambda_q}}
\iprod{\sum\limits_{i=1}^n{\sigma_i}{K(\cdot,\bbx_i)}}{v_q}_{\cH_K}v_q}{\cH_K}
\nonumber \\
&\phantom{\le}+ \norm{h}{\cH_K} \cdot
\norm{\sum\limits_{q = Q+1}^{\infty} \iprod{\sum\limits_{i=1}^n{\sigma_i}{K(\cdot,\bbx_i)}}{v_q}_{\cH_K}v_q}{\cH_K}
\nonumber \\
&\le \norm{h}{L^2}
\norm{\sum\limits_{q=1}^Q \sum\limits_{i=1}^n{\sigma_i} e_q(\bbx_i)v_q}{\cH_K}
+ B
\norm{\sum\limits_{q = Q+1}^{\infty} \sum\limits_{i=1}^n{\sigma_i} \sqrt{\lambda_q}e_q(\bbx_i) v_q}{\cH_K} \nonumber \\
&\le r_h \sqrt{\sum\limits_{q=1}^Q \pth{\sum\limits_{i=1}^n{\sigma_i} e_q(\bbx_i)}^2}
+ B
\sqrt{\sum\limits_{q = Q+1}^{\infty} \pth{\sum\limits_{i=1}^n {\sigma_i} \sqrt{\lambda_q}e_q(\bbx_i)}^2},
\eal
where $\circled{1}$ is due to Cauchy-Schwarz inequality.
Moreover, by Jensen's inequality we have
\bal\label{eq:lemma-LRC-population-NN-seg6}
\Expect{}{\sqrt{\sum\limits_{q=1}^Q \pth{\sum\limits_{i=1}^n{\sigma_i} e_q(\bbx_i)}^2}}
&\le \sqrt{ \Expect{}{\sum\limits_{q=1}^Q \pth{\sum\limits_{i=1}^n{\sigma_i} e_q(\bbx_i)}^2 } }
\le \sqrt{
\Expect{}{\sum\limits_{q=1}^Q \sum\limits_{i=1}^n e_q^2(\bbx_i) }}
=\sqrt{nQ}.
\eal
and similarly,
\bal\label{eq:lemma-LRC-population-NN-seg7}
\Expect{}{\sqrt{\sum\limits_{q = Q+1}^{\infty} \pth{\sum\limits_{i=1}^n{\sigma_i\sqrt{\lambda_q}} e_q(\bbx_i)}^2}}
&\le \sqrt{
\Expect{}{\sum\limits_{q = Q+1}^{\infty} \lambda_q \sum\limits_{i=1}^n e_q^2(\bbx_i) }}
=\sqrt{n\sum\limits_{q = Q+1}^{\infty}\lambda_q}.
\eal
Since (\ref{eq:lemma-LRC-population-NN-seg5})-(\ref{eq:lemma-LRC-population-NN-seg7}) hold for all $Q \ge 0$,
it follows that
\bal\label{eq:lemma-LRC-population-NN-seg8}
\Expect{}{\sup_{h \in \cH_K(B), \norm{h}{L^2} \le r_h } {\frac{1}{n} \sum\limits_{i=1}^n {\sigma_i}{h(\bbx_i)}} }
\le \min_{Q \colon Q \ge 0} \pth{ r_h \sqrt{nQ} +
B
\sqrt{n\sum\limits_{q = Q+1}^{\infty}\lambda_q}}.
\eal
It follows from (\ref{eq:lemma-LRC-population-NN-decomp}),
(\ref{eq:lemma-LRC-population-NN-seg1}), and
(\ref{eq:lemma-LRC-population-NN-seg8}) that
\bal\label{eq:lemma-LRC-population-NN-R1}
\cR_1 &\le \frac 1n \Expect{}{\sup_{h \in \cH_K(B), \norm{h}{L^2} \le r_h } { \sum\limits_{i=1}^n {\sigma_i}{h(\bbx_i)}} }
\le \min_{Q \colon Q \ge 0} \pth{r_h \sqrt{\frac{Q}{n}} +
B
\pth{\frac{\sum\limits_{q = Q+1}^{\infty}\lambda_q}{n}}^{1/2}}.
\eal

\textbf{Derivation for the upper bound for $\cR_2$.}

Because  $\abth{1/n \sum_{i=1}^n \sigma_i e(\bbx_i) }\le w$
when $\supnorm{e} \le w$, we have
\bal\label{eq:lemma-LRC-population-NN-R2}
\cR_2 \le \frac 1n \Expect{}{\sup_{e \in L^{\infty} \colon \supnorm{e} \le w } {{ \sum\limits_{i=1}^n {\sigma_i}{e(\bbx_i)}} } }
\le w.
\eal
It follows from (\ref{eq:lemma-LRC-population-NN-R1})
and (\ref{eq:lemma-LRC-population-NN-R2}) that
\bals
&\cfrakR \pth{\set{f \colon f \in \cF(B,w), \Expect{P}{f^2} \le r}}
\le \min_{Q \colon Q \ge 0} \pth{r_h \sqrt{\frac{Q}{n}} +
B
\pth{\frac{\sum\limits_{q = Q+1}^{\infty}\lambda_q}{n}}^{1/2}}
+ w.
\eals

Plugging $r_h$ in the RHS of the above inequality
 completes the proof.
\end{proofof}

\section{Proofs for Channel Selection}
\label{sec:channel-selection}

\begin{proofof}{Theorem~\ref{theorem:channel-selection}}
We denote $\tau_{\ell}(1)$ as $\tau_{\ell}$ for all $\ell \in [0\relcolon L]$ in this proof.
Let $\bbeta \in \RR^{r_0}$ be the vector with the elements $\set{ \beta_{\ell,j} \colon \ell \in [0\relcolon \ell_0],j \in [N(d,\ell)]}$.

We note that $f^*(\bS) = \bY(\bS,r_0) \bbeta$ and
$\by = f^*(\bS) + \bw$, so that
$\tau_{\ell} = \tau_{*,\ell} + \tau_{\bw,\ell} + \hat \tau_{\bw,\ell}$, and
\bals
\tau_{*,\ell} &\defeq \frac{1}{n^2 m} \bbeta^{\top} \bY^{\top}(\bS,r_0)\bY(\bS,\ell)  \bY^{\top}(\bQ,\ell)\bY(\bQ,m_L)  \bY^{\top}(\bS,m_L) \bY(\bS,r_0)\bbeta, \\
\tau_{\bw,\ell} &\defeq \frac{1}{n^2 m}
\bw^{\top} \bY(\bS,\ell)  \bY^{\top}(\bQ,\ell)\bY(\bQ,m_L)  \bY^{\top}(\bS,m_L) \bw, \\
\hat \tau_{\bw,\ell}  &\defeq \frac{2}{n^2 m}
\bbeta^{\top} \bY^{\top}(\bS,r_0) \bY(\bS,\ell)  \bY^{\top}(\bQ,\ell)\bY(\bQ,m_L)  \bY^{\top}(\bS,m_L) \bw,
\eals
where $\bbeta \in \RR^{r_0}$, and the elements of $\bbeta$ form the enumeration of $\set{\beta_{\ell j}}_{0 \le \ell \le \ell_0, j \in [N(d,\ell)]}$.
We let
\bals
\bY^{\top}(\bS,r_0) \bY(\bS,\ell)/n &= \bE_{r_0,\ell} + \bDelta_{r_0,\ell},
\bE_{r_0,\ell} \defeq \Expect{}{\bY^{\top}(\bS,r_0) \bY(\bS,\ell)},
 \\
\bY^{\top}(\bS,m_L) \bY(\bS,r_0)/n &= \bE_{m_L,r_0} + \bDelta_{m_L,r_0},
\bE_{m_L,r_0} \defeq \Expect{}{\bY^{\top}(\bS,m_L) \bY(\bS,r_0)},
 \\
\bY^{\top}(\bQ,\ell)\bY(\bQ,m_L)/m &= \bE_{\ell,m_L} + \bDelta_{\ell,m_L},
\bE_{\ell,m_L} \defeq \bY^{\top}(\bQ,\ell)\bY(\bQ,m_L).
\eals
Here $\bE_{r_0,\ell}, \bDelta_{r_0,\ell}, \in \RR^{r_0 \times N(d.\ell)}$,
$\bE_{m_L,r_0}, \bDelta_{m_L,r_0} \in \RR^{m_L \times r_0}$, and $\bE_{\ell,m_L}, \bDelta_{\ell,m_L} \in \RR^{ N(d.\ell) \times m_L}$. We let $\bA_{[s\relcolon t]}$ to denote the submatrix of $\bA$ formed by rows of $\bA$ with row indices in $[s \relcolon t]$, and $\bA^{[s\relcolon t]}$ to denote the submatrix of $\bA$ formed by columns of $\bA$ with columns indices in $[s \relcolon t]$. Then  if $0 \le \ell \le \ell_0$,
\bals
\bth{\bE_{r_0,\ell}}_{[m_{\ell-1}+1\relcolon m_{\ell}]} = \bI_{N(d,\ell)}, \quad \bth{\bE_{r_0,\ell}}_j = \bzero {\textup{ for all }} j \notin [m_{\ell-1}+1\relcolon m_{\ell}],
\eals
and $\bE_{r_0,\ell} = \bzero$ if $\ell > \ell_0$.
Similarly,
\bals
\bth{\bE_{m_L,r_0}}_{[1\relcolon r_0]} = \bI_{r_0}, \quad
\bth{\bE_{m_L,r_0}}_{[r_0+1\relcolon m_L]} = \bzero,
\eals
and
\bals
\bth{E_{\ell,m_L}}^{[m_{\ell-1}+1\relcolon m_{\ell}]} = \bI_{N(d,\ell)}, \quad \bth{E_{\ell,m_L}}^{(j)} = \bzero {\textup{ for all }} j \notin [m_{\ell-1}+1\relcolon m_{\ell}].
\eals
With $\min\set{m,n} > 4m_L\log({4m_L}/{\delta})$,
it follows from Lemma~\ref{lemma:bounded-sigma-channel-selection}
that,
with probability at least $1-\delta$ for every $\delta \in (0,1)$,
\bal\label{eq:channel-selection-small-error-norm}
\max\set{\ltwonorm{\bDelta_{r_0,\ell}}, \ltwonorm{ \bDelta_{m_L,r_0}}}
\le \sqrt{ \log\pth{\frac{4m_L}{\delta}} \frac{4m_L}{n}  } \le 1, \quad
\ltwonorm{ \bDelta_{\ell,m_L}} \le
\sqrt{ \log\pth{\frac{4m_L}{\delta}} \frac{4m_L}{m} } \le 1,
\eal
which are due to the fact that $\max\set{\ltwonorm{\bDelta_{r_0,\ell}}, \ltwonorm{ \bDelta_{m_L,r_0}}} \le \ltwonorm{\Delta_{\bS, m_L}}$ and
$\ltwonorm{ \bDelta_{\ell,m_L}} \le \ltwonorm{\Delta_{\bQ, m_L}} $, where $\Delta_{\bS, m_L}, \Delta_{\bQ, m_L}$ are defined in
(\ref{eq:delta-S-mL-def})-(\ref{eq:delta-Q-mL-def}).
We have
\bal\label{eq:channel-selection-seg1}
\tau_{*,\ell} &=\frac{1}{n^2 m}\bbeta^{\top} \bY^{\top}(\bS,r_0)\bY(\bS,\ell)
\underbrace{\bY^{\top}(\bQ,\ell)\bY(\bQ,m_L)  \bY^{\top}(\bS,m_L) \bY(\bS,r_0)}_{\defeq \bD_1} \bbeta
\nonumber \\
&=\frac{1}{n m}\bbeta^{\top} E_{r_0,\ell}  \bD_1 \bbeta
+\underbrace{\frac{1}{n m}\bbeta^{\top} \bDelta_{r_0,\ell}\bD_1 \bbeta}_{\defeq E_1}.
\eal
It follows from (\ref{eq:channel-selection-small-error-norm}) that
\bal\label{eq:channel-selection-error-YQl-ml-YSml-r0}
\ltwonorm{\bY^{\top}(\bQ,\ell)\bY(\bQ,m_L)}
\le 2m, \quad \ltwonorm{ \bY^{\top}(\bS,m_L) \bY(\bS,r_0)}
\le 2n.
\eal
It follows from (\ref{eq:channel-selection-error-YQl-ml-YSml-r0}) that
\bal\label{eq:channel-selection-bD1-bound}
\ltwonorm{\bD_1}  \le  4mn.
\eal
It then follows from (\ref{eq:channel-selection-small-error-norm})
 and (\ref{eq:channel-selection-bD1-bound}) that
\bal\label{eq:channel-selection-E1-bound}
\abth{E_1} \le \frac{1}{n m} \cdot \ltwonorm{\bbeta}^2 \ltwonorm{\bDelta_{r_0,\ell}}
\ltwonorm{\bD_1} \le 4\gamma_0^2
\sqrt{ \log\pth{\frac{4m_L}{\delta}} \frac{4m_L}{n}}.
\eal
We have
\bal\label{eq:channel-selection-seg2}
\frac{1}{n m }\bbeta^{\top} E_{r_0,\ell}  \bD_1 \bbeta
&= \frac{1}{n}\bbeta^{\top} E_{r_0,\ell} E_{\ell,m_L}  \bY^{\top}(\bS,m_L) \bY(\bS,r_0)\bbeta
\nonumber \\
&\phantom{=}+  \underbrace{\frac{1}{n}\bbeta^{\top} E_{r_0,\ell} \bDelta_{\ell,m_L}
  \bY^{\top}(\bS,m_L) \bY(\bS,r_0)\bbeta}_{\defeq E_2},
\eal
and
\bal\label{eq:channel-selection-E2-bound}
\abth{E_2} \le \frac{1}{n  } \ltwonorm{\bbeta}^2 \ltwonorm{\bDelta_{\ell,m_L}}
\ltwonorm{\bY^{\top}(\bS,m_L) \bY(\bS,r_0)}
\le 2\gamma_0^2\sqrt{ \log\pth{\frac{4m_L}{\delta}} \frac{4m_L}{m}}.
\eal
We further have
\bal\label{eq:channel-selection-seg3}
\frac{1}{n}\bbeta^{\top} E_{r_0,\ell} E_{\ell,m_L}  \bY^{\top}(\bS,m_L) \bY(\bS,r_0)\bbeta
&=\bbeta^{\top} E_{r_0,\ell} E_{\ell,m_L}   E_{m_L,r_0}\bbeta
\nonumber \\
&\phantom{=}+\underbrace{\bbeta^{\top} E_{r_0,\ell} E_{\ell,m_L}
\bDelta_{m_L,r_0}\bbeta}_{\defeq E_3},
\eal
and
\bal\label{eq:channel-selection-E3-bound}
\abth{E_3} \le \gamma_0^2\sqrt{ \log\pth{\frac{4m_L}{\delta}} \frac{4m_L}{n}}.
\eal
We note that
\bal\label{eq:channel-selection-within-channel-weight-1}
\bbeta^{\top} E_{r_0,\ell} E_{\ell,m_L}  E_{m_L,r_0}\bbeta
=
\begin{cases}
0 &\ell_0 < \ell \le L \\
\sum_{j \in N(d,\ell)} \beta^2_{\ell,j} & \ell \in [0\relcolon \ell_0].
\end{cases}
\eal
It follows that when $\ell \in [0\relcolon \ell_0]$,
\bal\label{eq:channel-selection-within-channel-weight-2}
\bbeta^{\top} E_{r_0,\ell} E_{\ell,m_L}  E_{m_L,r_0}\bbeta
\ge N(d,\ell) \min_{\ell \in [0\relcolon \ell_0],j \in [N(d,\ell)]]}
{\beta^2_{\ell,j}}  \ge \beta_0^2.
\eal
For $\eps_0 \le \beta_0^2/3$, with
\bals
m \ge \frac{256m_L\gamma_0^4}{\eps_0^2}
\log\pth{\frac{4m_L}{\delta}},
n \ge \max\set{\frac{400m_L\gamma_0^4}{\eps_0^2}
\log\pth{\frac{4m_L}{\delta}},\frac{32m_L(\sigma_0^2+1)}{\eps_0}, \frac{8192 \gamma_0^2 m_L(\sigma_0^2+1)}{\eps_0^2}},
\eals
by Lemma~\ref{lemma:quadratic-subgaussian-subspace} and
Lemma~\ref{lemma:quadratic-subgaussian-subspace-hat-tau} we have
\bals
E_1+E_3 &\le 5\gamma_0^2
\sqrt{ \log\pth{\frac{4m_L}{\delta}} \frac{4m_L}{n}} \le \eps_0/2,
\quad E_2 \le 2\gamma_0^2
\sqrt{ \log\pth{\frac{4m_L}{\delta}} \frac{4m_L}{m}}  \le \eps_0/4,
\\
\abth{\tau_{\bw,\ell}} &\le \frac{4m_L(\sigma_0^2+1)}{n} \le \eps_0/8,
\quad \abth{\hat \tau_{\bw,\ell}} \le 8 {\sqrt 2} \gamma_0 \sqrt{\frac{m_L(\sigma_0^2+1)}{n}}  \le \eps_0/8,
\eals
which hold for all $\ell \in [0\relcolon L]$.
Combining the above results,  we have
\bal\label{eq:channel-selection-within-channel-weight-tau-l-star}
\begin{cases}
\tau_{*,\ell} \ge  \beta_0^2 -E_1-E_2-E_3\ge \frac{9\eps_0}{4}, & \ell \in [0\relcolon \ell_0], \\
\abth{\tau_{*,\ell}} \le  E_1+E_2+E_3\le \frac{3\eps_0}{4},& \ell_0 < \ell \le L.
\end{cases}
\eal
As a result, when
$\ell \in [0\relcolon \ell_0]$, we have
\bal\label{eq:channel-selection-within-channel-weight-tau-in-target}
\tau_{\ell} &= \tau_{*,\ell} + \tau_{\bw,\ell} + \hat \tau_{\bw,\ell}
\ge\frac{9\eps_0}{4}-\frac{\eps_0}{8} - \frac{\eps_0}{8}
\ge 2\eps_0.
\eal
When $\ell_0 < \ell \le L$, we have
\bal\label{eq:channel-selection-within-channel-weight-tau-out-target}
\abth{\tau_{\ell}} &\le \abth{\tau_{*,\ell}} +
\abth{\tau_{\bw,\ell}} +\abth{\hat \tau_{\bw,\ell}}\le\frac{3\eps_0}{4}+\frac{\eps_0}{8} + \frac{\eps_0}{8} \le \eps_0,
\eal
which completes the proof with the union bound.

\end{proofof}

\begin{lemma}\label{lemma:quadratic-subgaussian-subspace}
For every $\delta \in (0,1)$, suppose $\min\set{n,m} > 4m_L\log({4m_L}/{\delta})$.
Then with probability at least
 $1-\exp\pth{-\Theta(m_L)}-\delta$, for every $\ell \in [0\relcolon L]$,
\bal\label{eq:quadratic-subgaussian-subspace}
\abth{\tau_{\bw,\ell}} \le \frac{4m_L(\sigma_0^2+1)}{n} .
\eal
\end{lemma}
\begin{proof}
We first define $\bM = \bY(\bS,\ell)  \bY^{\top}(\bQ,\ell)
\bY(\bQ,m_L)  \bY^{\top}(\bS,m_L)/\pth{n^2 m} \in \RR^{n\times n}$, then $\tau_{\bw,\ell} =  \bw^{\top} \bM \bw$.
With $n > 4m_L\log({4m_L}/{\delta})$,
it follows from (\ref{eq:bounded-sigma-channel-selection})
in Lemma~\ref{lemma:bounded-sigma-channel-selection}
that both $\bY(\bS,\ell)$ and $\bY(\bS,m_L)$ are of full column rank.
We let the singular value decomposition of $ \bY(\bS,m_L)$ and
$\bY(\bS,\ell)$ be
\bal\label{eq:quadratic-subgaussian-subspace-svd}
\bY(\bS,m_L) = \bU^{(L)} \bSigma^{(L)} {\bV^{(L)}}^{\top},
\bY(\bS,\ell) = \bU^{(\ell)} \bSigma^{(\ell)} {\bV^{(\ell)}}^{\top},
\eal
where $\bU^{(L)} \in \RR^{n \times m_L}, \bV^{(L)} \in \RR^{m_L \times m_L}, \bU^{(\ell)} \in \RR^{n \times N(d,\ell)}, \bV^{(\ell)} \in \RR^{N(d,\ell) \times N(d,\ell)}$ are orthogonal matrices, $ \bSigma^{(L)} \in
\RR^{m_L \times m_L},  \bSigma^{(\ell)}\in \RR^{N(d,\ell) \times N(d,\ell)}$ are diagonal matrices.  We can then express $\bM$ as
\bal\label{eq:quadratic-subgaussian-subspace-M}
\bM = \frac{1}{n}\bU^{(\ell)} \underbrace{(\bSigma^{(\ell)}/{\sqrt n}) {\bV^{(\ell)}}^{\top}  (\bY^{\top}(\bQ,\ell)
\bY(\bQ,m_L)/m)  \bV^{(L)} (\bSigma^{(L)}/{\sqrt n}) }_{\defeq \bD}{\bU^{(L)} }^{\top}.
\eal
The operator norm of $\bD$ in (\ref{eq:quadratic-subgaussian-subspace-M}) can be bounded by
\bal\label{eq:quadratic-subgaussian-subspace-D-bound}
\ltwonorm{\bD} \le 4,
\eal
which holds with probability at least $1-\delta$.
It follows from (\ref{eq:bounded-sigma-channel-selection})
in Lemma~\ref{lemma:bounded-sigma-channel-selection}
 again that
$\ltwonorm{\bSigma^{(\ell)}/{\sqrt n}} \le {\sqrt 2}$,
$\ltwonorm{\bSigma^{(L)}/{\sqrt n}} \le {\sqrt 2}$,
and $\ltwonorm{(\bY^{\top}(\bQ,\ell)
\bY(\bQ,m_L)/m)} \le 2$, so that
(\ref{eq:quadratic-subgaussian-subspace-D-bound}) holds.
Moreover, because the column space of $\bY(\bS,\ell)$ is a subspace
of the column space of $\bY(\bS,m_L)$, we have
$\ltwonorm{{\bU^{(\ell)} }^{\top} \bw} \le \ltwonorm{ {\bU^{(L)} }^{\top} \bw}$. It then follows from this fact and
(\ref{eq:quadratic-subgaussian-subspace-M})-(\ref{eq:quadratic-subgaussian-subspace-D-bound}) that
\bal\label{eq:quadratic-subgaussian-subspace-seg-1}
\tau_{\bw,\ell} = \bw^{\top} \bM \bw
\le\frac{4}{n} \ltwonorm{ {\bU^{(L)} }^{\top} \bw}^2.
\eal
It follows from the concentration inequality about quadratic forms of sub-Gaussian random variables in \cite{quadratic-tail-bound-Wright1973} that
$\Prob{\ltwonorm{{\bU^{(L)} }^{\top} \bw}^2 -
\Expect{}{\ltwonorm{{\bU^{(L)} }^{\top} \bw}^2} > m_L}
\le \exp\pth{-\Theta(m_L)}$. Then
 with probability at least
 $1-\exp\pth{-\Theta(m_L)}$, we have
\bal\label{eq:quadratic-subgaussian-subspace-seg-2}
\ltwonorm{{\bU^{(L)} }^{\top} \bw}^2 \le
\Expect{}{\ltwonorm{{\bU^{(L)} }^{\top} \bw}^2} + m_L\le
\sigma_0^2 \tr{\bU^{(L)} {\bU^{(L)} }^{\top}} +m_L=m_L(\sigma_0^2+1).
\eal
(\ref{eq:quadratic-subgaussian-subspace}) then follows from
(\ref{eq:quadratic-subgaussian-subspace-seg-1}) and
(\ref{eq:quadratic-subgaussian-subspace-seg-2}).
\end{proof}

\begin{lemma}\label{lemma:quadratic-subgaussian-subspace-hat-tau}
For every $\delta \in (0,1)$, suppose $\min\set{n,m}  > 4m_L\log({4m_L}/{\delta})$.
Then with probability at least
 $1-\exp\pth{-\Theta(m_L)}-\delta$, for every $\ell \in [0\relcolon L]$,
\bal\label{eq:quadratic-subgaussian-subspace-hat-tau}
\abth{\hat \tau_{\bw,\ell}} \le 8 {\sqrt 2} \gamma_0 \sqrt{\frac{m_L(\sigma_0^2+1)}{n}} .
\eal
\end{lemma}
\begin{proof}
Recall that the singular value decomposition of $ \bY(\bS,m_L)$
is $\bY(\bS,m_L) = \bU^{(L)} \bSigma^{(L)} {\bV^{(L)}}^{\top}$ as that  in (\ref{eq:quadratic-subgaussian-subspace-svd}). Then we have
\bals
\hat \tau_{\bw,\ell} &= \frac{2}{\sqrt n}\bbeta^{\top} \underbrace{(\bY^{\top}(\bS,r_0)\bY(\bS,\ell)/n)   (\bY^{\top}(\bQ,\ell)
\bY(\bQ,m_L)/m)  \bV^{(L)} (\bSigma^{(L)}/{\sqrt n}) }_{\defeq \bD_{\bu}}{\bU^{(L)} }^{\top} \bw
\nonumber \\
&=\frac{2}{\sqrt n} \bbeta^{\top}\bD_{\bu}\bw.
\eals
It follows from Lemma~\ref{lemma:bounded-sigma-channel-selection}
that with probability at least $1-\delta$,
\bal\label{eq:quadratic-subgaussian-subspace-hat-tau-Du}
\ltwonorm{\bD_{\bu}} \le 4{\sqrt 2}.
\eal
Moreover, it follows from (\ref{eq:quadratic-subgaussian-subspace-seg-2})   in the proof of Lemma~\ref{lemma:quadratic-subgaussian-subspace} that with probability at least $1-\exp\pth{-\Theta(m_L)}$,
\bal\label{eq:quadratic-subgaussian-subspace-hat-tau-bUw}
\ltwonorm{{\bU^{(L)} }^{\top} \bw} \le \sqrt{m_L(\sigma_0^2+1)}.
\eal
(\ref{eq:quadratic-subgaussian-subspace-hat-tau}) then follows from (\ref{eq:quadratic-subgaussian-subspace-hat-tau-Du}),
 (\ref{eq:quadratic-subgaussian-subspace-hat-tau-bUw}), and the fact that $\ltwonorm{\bbeta} \le \gamma_0$.

\end{proof}

\begin{lemma}\label{lemma:bounded-sigma-channel-selection}
With $\tau_{\ell}=1$ for all $\ell\in [0\relcolon L]$ in the activation function \bals
\sigma_{\btau}(\bx,\bx') = \sum\limits_{\ell=0}^{L} \sum\limits_{j=1}^{N(d,\ell)} \tau_{\ell} \mu_{\sigma,\ell} Y_{\ell,j}(\bx)Y_{\ell,j}(\bx'),
\eals
we have $\sup_{\bx,\bx' \in \cX} \abth{\sigma_{\btau}(\bx,\bx')} \le L+1$. Moreover, define
\bal
\bY^{\top}(\bS,m_L)\bY(\bS,m_L)/n - \bI_{m_L} \defeq \Delta_{\bS, m_L},
\label{eq:delta-S-mL-def} \\
\bY^{\top}(\bQ,m_L) \bY(\bQ,m_L)/m)-\bI_{m_L} \defeq \Delta_{\bQ, m_L}
\label{eq:delta-Q-mL-def}.
\eal
When $n,m \ge 4m_L\log({4m_L}/{\delta})$, for every $\delta \in (0,1)$, with probability at least $1-\delta$,
\bal\label{eq:bounded-sigma-channel-selection}
\max\set{\Delta_{\bS, m_L}, \Delta_{\bQ, m_L} } \le 1.
\eal
\end{lemma}
\begin{proof}
First, we note that with $\tau_{\ell}=1$ for every $\ell\in [0\relcolon L]$,
\bals
\sigma_{\btau}(\bx,\bx')=  \sum\limits_{\ell=0}^{L} P^{(d)}_{\ell}(\iprod{\bx}{\bx'})
\le L+1,
\eals
which follows from the fact that $\sup_{t \in [-1,1], k \ge 0}\abth{P^{(d)}_k(t)} \le 1$ in Section~\ref{sec:harmonic-analysis-detail}
of the appendix.
Furthermore, it follows from Lemma~\ref{lemma:YYtop-approximate-identity} that with probability at least $1-\delta$ for every $\delta \in (0,1)$,
\bals
&\max\bigg\{\ltwonorm{\bY^{\top}(\bS,m_L)\bY(\bS,m_L)/n - \bI_{m_L}},
\ltwonorm{(\bY^{\top}(\bQ,m_L) \bY(\bQ,m_L)/m)-\bI_{m_L}} \bigg\}
\nonumber \\
&\le \max\set{\sqrt{ \log\pth{\frac{4m_L}{\delta}} \frac{4m_L}{n} },\sqrt{ \log\pth{\frac{4m_L}{\delta}} \frac{4m_L}{m} } } \le 1,
\eals
which proves (\ref{eq:bounded-sigma-channel-selection}). It is noted that we use $\bS$ and $\bQ$ to replace the sample $\set{\bbw_r}$ in
Lemma~\ref{lemma:YYtop-approximate-identity} to
 obtain (\ref{eq:bounded-sigma-channel-selection}).
\end{proof}

\begin{lemma}
\label{lemma:YYtop-approximate-identity}
Recall that $\set{Y_j}_{j=0}^{m_L-1} = \set{Y_{\ell j}}_{0 \le \ell \le L, j \in [N(d,\ell)]}$ as the enumeration of all the spherical harmonics of up to degree $L$. Suppose $A,B$ are two nonempty subsets of $[0:m_L-1]$ containing consecutive integers starting with $0$ with $\abth{A} = r_1$, $\abth{B} = r_2$,
and $Y_A = \set{Y_j \colon j \in A }$ and $Y_B = \set{Y_j \colon j \in B }$. For any vector $\bw \in \cX$, we define $Y_A(\bw) \in \RR^{r_1}$ as a vector whose elements are $\set{Y_j(\bw) \colon j \in A}$, and $Y_B(\bw)$ is defined similarly.
Let $\set{\bbw_r}_{r \in [m]} \overset{\text{iid}}{\sim} \Unif{\cX}$. We define $\bA^{(r)} \in \RR^{r_1}$ with $\bA^{(r)} = Y_A(\bbw_r)$ for all $r \in [m]$, and
$\bA = \bth{\bA^{(1)},\ldots \bA^{(m)}} \in \RR^{{r_1} \times m}$. Similarly, we define $\bB^{(r)} \in \RR^{r_2}$ with $\bB^{(r)} = Y_B(\bbw_r)$ for all $r \in [m]$, and
$\bB = \bth{\bB^{(1)},\ldots \bB^{(m)}} \in \RR^{{r_2} \times m}$.
Suppose that $\ltwonorm{Y_A(\bw)}^2$  and $\ltwonorm{Y_B(\bw)}^2$ are not varying with $\bw$, and $\ltwonorm{Y_A(\bw)}^2\in [1,m_L]$,
$\ltwonorm{Y_B(\bw)}^2 \in [1,m_L]$. Then for every $\delta \in  (0,1)$, when
$m \ge 4m_L\log({2m_L}/{\delta})$,
\bal\label{eq:YYtop-approximate-identity}
\Prob{\ltwonorm{\frac {\bA \bB^{\top}}{m} -
\Expect{}{\frac {\bA \bB^{\top}}{m}}}  \ge
\sqrt{ \log\pth{\frac{2m_L}{\delta}} \frac{4m_L}{m}  }  } \le \delta.
\eal
\end{lemma}
\begin{remark}
When $Y_A$ contains spherical harmonics of several degrees, for example, there exists $\ell_1,\ell_2 \in [0\relcolon L]$ and $\ell_1 \le \ell_2$ such that
$Y_A =  \set{Y_{\ell j}}_{\ell_1 \le \ell \le \ell_2, j \in [N(d,\ell)]}$,
then it can be verified that $\ltwonorm{Y_A(\bw)}^2 = \sum_{\ell=\ell_1}^{\ell_2} N(d,\ell)$ which does not vary with $\bw \in \cX$.  The same argument applies to $Y_B$. Throughout this paper we would apply
Lemma~\ref{lemma:YYtop-approximate-identity} for such cases.
\end{remark}
\begin{proof}
First, we have
\bals
\frac {\bA \bB^{\top}}m = \frac{1}{m} \sum_{r=1}^m \bA^{(r)} {\bB^{(r)}}^{\top},
\quad \Expect{}{\bA \bB^{\top} / m} \defeq \bE \in \RR^{r_1 \times r_2}.
\eals
Let $A = \set{0,1,\ldots,r_1-1}$ and $B = \set{0,1,\ldots,r_2-1}$, then it follows from the orthogonality of $\set{Y_j}_{j=0}^{m_L-1}$ that $\bE_{st} = \indict{s=t} \cdot \indict{s\le \min\set{r_1,r_2}}$ for all $s \in [r_1]$ and $j \in [r_2]$. It follows that the off-diagonal elements of $\bE \bE^{\top}$ and $\bE^{\top} \bE$ are  $0$, and the diagonal elements of $\bE \bE^{\top}$ and $\bE^{\top} \bE$ are either $0$ or $1$.
We now apply the matrix Bernstein inequality in Theorem~\ref{theorem:matrix-bernstein-inequality}. We define  $\bX^{(r)} \defeq \bA^{(r)} {\bB^{(r)}}^{\top}- \bE \in \RR^{r_1 \times r_2}$. Then we have $\Expect{}{\bX^{(r)}} = 0$, and
\bal\label{eq:YYtop-approximate-identity-seg1}
\ltwonorm{\bX^{(r)}} \le \ltwonorm{\bA^{(r)}} \ltwonorm{ \bB^{(r)}} + 1
\le {m_L} + 1,
\eal
where we use the fact that $\max\set{\ltwonorm{Y_A(\bw)}^2,\ltwonorm{Y_B(\bw)}^2}
\le m_L$. Let $V = \ltwonorm{\sum\limits_{r=1}^m \Expect{}{\bX^{(r)}{\bX^{(r)}}^{\top}}}$,  then we have
\bal\label{eq:YYtop-approximate-identity-seg2}
&V \le \sum_{r=1}^m \ltwonorm{ \Expect{}{ \pth{\bA^{(r)} {\bB^{(r)}}^{\top}- \bE} \pth{\bA^{(r)} {\bB^{(r)}}^{\top}- \bE}^{\top} } } \nonumber \\
&= \sum_{r=1}^m \ltwonorm{ \Expect{}{ \bA^{(r)} {\bB^{(r)}}^{\top} \bB^{(r)} {\bA^{(r)}}^{\top} - \bA^{(r)} {\bB^{(r)}}^{\top} {\bE}^{\top} - \bE\bB^{(r)} {\bA^{(r)}}^{\top}  + \bE \bE^{\top} } } \nonumber \\
&\stackrel{\circled{1}}{=} \sum_{r=1}^m  \ltwonorm{\bI_{r_1} \ltwonorm{\bB^{(r)}}^2-\bE \bE^{\top} } \stackrel{\circled{2}}{\le} m (m_L-1),
\eal
where $\circled{1}$ follows from $\Expect{}{ \bA^{(r)} {\bA^{(r)}}^{\top}  } = \bI_{r_1}$ due to the orthogonality of the set $Y_A$.
$\circled{2}$ follows from the fact that $\ltwonorm{\bB^{(r)}}^2$ is a constant and $1 \le \ltwonorm{\bB^{(r)}}^2 \le m_L$. It can be verified in a way
similar to (\ref{eq:YYtop-approximate-identity-seg2}) that
$\ltwonorm{\sum\limits_{r=1}^m \Expect{}{ {\bX^{(r)}}^{\top}\bX^{(r)}} }
\le m (m_L-1)$.

As a result, it follows from the matrix Bernstein inequality in
Theorem~\ref{theorem:matrix-bernstein-inequality},
(\ref{eq:YYtop-approximate-identity-seg1}), and (\ref{eq:YYtop-approximate-identity-seg2}) that, for every $t  \in (0,1]$,
\bals
\Prob{\ltwonorm{\frac {\bA \bB^{\top}}{m} -
\Expect{}{\frac {\bA \bB^{\top}}{m}}}   \ge t } &\le 2m_L \exp\pth{ - \frac{m^2 t^2}{2 m (m_L- 1) + 2 (m_L + 1) m t / 3} } \nonumber \\
&\le 2 m_L \exp \pth{ - \frac{m t^2}{4m_L}  },
\eals
which proves (\ref{eq:YYtop-approximate-identity}).
\end{proof}

\begin{theorem}
[Matrix Bernstein Inequality, {\cite[Theorem 6.1.1]{Tropp2015-matrix-concentration}}]
\label{theorem:matrix-bernstein-inequality}
Let $\set{\bX^{(r)}}_{i=1}^n$ be independent, centered, self-adjoint random matrices in $\RR^{d_1 \times d_2}$ such that
$\Expect{}{\bX^{(r)}} = 0$, $\ltwonorm{\bX^{(r)}} \leq L$ for all $i \in [n]$. Let the total variance be
\bals
\sigma^2 := \max\set{  \ltwonorm{\sum_{i=1}^{n} \Expect{}{ \bX^{(r)} {\bX^{(r)}}^{\top} }},
\ltwonorm{\sum_{i=1}^{n} \Expect{}{ {\bX^{(r)}}^{\top}\bX^{(r)} }}  }.
\eals
Then, for all $t \geq 0$,
\bal\label{eq:matrix-bernstein-inequality}
\Prob{ \ltwonorm{\sum_{i=1}^{n} \bX^{(r)}} \ge t } \le (d_1+d_2)\exp\pth{\frac{- t^2 / 2}{\sigma^2 + L t / 3}}.
\eal
\end{theorem}

\section{Existing Empirical and Theoretical Works about Channel Attention and General Attention Mechanism}
\label{sec:related-works-channel-attention}
Channel attention mechanisms~\cite{DANet, ECA, XCiT} have emerged as an effective method to enhance feature representations learned by DNNs by adaptively reweighting channel responses.  DANet~\cite{DANet} incorporates a channel attention branch alongside spatial attention to capture inter-channel relationships, enabling feature refinement for the image segmentation task. Following that, ECA-Net~\cite{ECA} introduces a parameter-efficient channel attention module based on the 1D convolution. XCiT~\cite{XCiT} interprets channel attention as a cross-covariance operation across feature dimensions, and demonstrates its effectiveness for image classification by replacing the self-attention module in the vision transformer. More recently, \cite{ChenS0LS25} establishes a theoretical framework for covariance-based channel interactions, which is also referred to as covariance pooling, demonstrating that matrix function normalizations, such as logarithm, power, or square-root, applied to Symmetric Positive Definite (SPD) covariance matrices implicitly induce Riemannian classifiers, thereby offering a principled explanation of how second-order channel statistics improve discriminability and enhance the stability of DNNs for image classification.

Building on the same theoretical perspective, \cite{SongS021} analyzes why approximate matrix square root computations via Newton–Schulz iteration consistently outperform exact singular value decomposition (SVD) in covariance pooling, attributing the superiority of the approximate method to improved numerical stability and gradient smoothness. Furthermore, \cite{WangZGXZLZH23} investigates covariance pooling from an optimization perspective, showing that it smooths the loss landscape, yields flatter local minima, and acts as a feature-based preconditioner on gradients, thereby explaining its ability to accelerate convergence, improve robustness, and enhance generalization of deep architectures.

Kernelizable attention has been investigated in~\cite{ChoromanskiLDSG21-attention-random-feature-approx,Peng0Y0SK21-rf-attention,ZhengYWK23-efficient-attention-variance} for efficient approximation of attention matrices, and~\cite{HronBSN20-infinite-attention} analyzes multi-head attention architectures in the Gaussian process limit with infinitely many heads.
Although a few works, such as~\cite{kim2024transformers-optimal-in-context}, study the optimality of attention-based neural networks for in-context learning (ICL) tasks, the theoretical benefits of attention mechanisms, particularly channel attention, for standard nonparametric regression tasks remain largely unexplored.

However, to the bet of our knowledge, most existing works in attention mechanisms, including channel attention, do not give sharp rates for nonparametric regression with target function being low-degree spherical polynomials. Our work is among the first to reveal the theoretical benefit of channel attention with a novel and provable learnable channel selection algorithm for learning low-degree spherical polynomials with a minimax optimal rate.

\end{appendices}

\bibliographystyle{IEEEtran}
\bibliography{ref}

\end{document}